\documentclass[a4paper,11pt]{article}

\usepackage[utf8]{inputenc} 
\usepackage[T1]{fontenc}    

\usepackage{graphicx}
\usepackage{subfig}
\usepackage{booktabs} 
\usepackage{nicefrac}       
\usepackage[margin=1in]{geometry}
\usepackage{url}
\usepackage{smile}
\usepackage{algorithm}
\usepackage{algorithmic}
\usepackage{todonotes}
\usepackage{epstopdf}
\usepackage{wrapfig}
\usepackage{enumitem}
\usepackage{mathtools}

\usepackage[colorlinks, linkcolor=blue, anchorcolor=blue, citecolor=blue]{hyperref}       

\usepackage{kpfonts}
\DeclareMathAlphabet{\mathsf}{OT1}{cmss}{m}{n}

\SetMathAlphabet{\mathsf}{bold}{OT1}{cmss}{bx}{n}

\usepackage{natbib}

\usepackage{pifont}
\newcommand{\cmark}{\ding{51}}%
\newcommand{\xmark}{\ding{55}}%

\usepackage{makecell}

\newcommand{\Bnorm}{B_{p,q}^s}
\newcommand{\cBnorm}{\mathcal{B}_{p,q}^s}
\newcommand{\bnorm}{b_{p,q}^s}

\newcommand{\mone}{\mathds{1}}

\newcommand{\mtoned}{\widetilde{\mathds{1}}_{\Delta}}

\newcommand{\ttimes}{\widetilde{\times}}

\newcommand{\hf}{\widehat{f}}

\newcommand{\tf}{\widetilde{f}}
\newcommand{\tM}{\widetilde{M}}

\newcommand{\thh}{\widetilde{h}}

\newcommand{\trho}{\widetilde{\rho}}

\newcommand{\tdi}{\widetilde{d}_i^2}
\newcommand{\supp}{\mathrm{supp}}

\newcommand{\ReLU}{\mathrm{ReLU}}
\newcommand{\Conv}{\mathrm{Conv}}
\newcommand{\id}{\mathrm{id}}

\newtheorem*{theorem*}{Theorem}

\begin{document}


\title{Besov Function Approximation and Binary Classification on  
          Low-Dimensional Manifolds Using Convolutional Residual Networks}
\author{Hao Liu $\quad$ Minshuo Chen $\quad$ Tuo Zhao $\quad$ Wenjing Liao \thanks{Hao Liu is affiliated with the Department of Mathematics at Hong Kong Baptist University; Wenjing Liao is affiliated with the School of Mathematics at Georgia Tech; Minshuo Chen and Tuo Zhao are affiliated with the ISYE department at Georgia Tech ; Email: \text{haoliu@hkbu.edu.hk, $\{$mchen393, wliao60, tzhao80$\}$@gatech.edu }.}}

\date{}

\maketitle

\begin{abstract}

Most of existing statistical theories on deep neural networks have sample complexities cursed by the data dimension and therefore cannot well explain the empirical success of deep learning on high-dimensional data. To bridge this gap, we propose to exploit low-dimensional geometric structures of the real world data sets. We establish theoretical guarantees of convolutional residual networks (ConvResNet) in terms of function approximation and statistical estimation for binary classification. Specifically, given the data lying on a $d$-dimensional manifold isometrically embedded in $\RR^D$, we prove that if the network architecture is properly chosen, ConvResNets can (1)  approximate {\it Besov functions} on manifolds with arbitrary accuracy, and (2) learn a classifier by minimizing the empirical logistic risk, which gives an {\it excess risk} in the order of $n^{-\frac{s}{2s+2(s\vee d)}}$, where $s$ is a smoothness parameter. This implies that the sample complexity depends on the intrinsic dimension $d$, instead of the data dimension $D$. Our results demonstrate that ConvResNets are adaptive to low-dimensional structures of data sets.
\end{abstract}


\section{Introduction}

Deep learning has achieved significant success in various practical applications with high-dimensional data set, such as computer vision \citep{krizhevsky2012imagenet}, natural language processing \citep{graves2013speech,young2018recent,wu2016google}, health care \citep{miotto2018deep,jiang2017artificial} and bioinformatics \citep{Alipanahi2015PredictingTS,zhou2015predicting}. 


The success of deep learning clearly demonstrates the great power of neural networks in representing complex data. In the past decades, the representation power of neural networks has been extensively studied. The most commonly studied architecture is the feedforward neural network (FNN), as it has a simple composition form.
The representation theory of FNNs has been developed
with smooth activation functions (e.g., sigmoid) in \citet{cybenko1989approximation,barron1993universal,mccaffrey1994convergence, hamers2006nonasymptotic,kohler2005adaptive,kohler2011analysis} or nonsmooth activations (e.g., ReLU) in \citet{lu2017expressive,yarotsky2017error,lee2017ability,suzuki2018adaptivity}. These works show that if the network architecture is properly chosen, FNNs can approximate uniformly smooth functions (e.g., H\"older or Sobolev) with arbitrary accuracy. 


\begin{figure}[htb!]
\vspace{-0.1in}
	\centering
	\subfloat[Convolution.]{\includegraphics[width=0.3\textwidth]{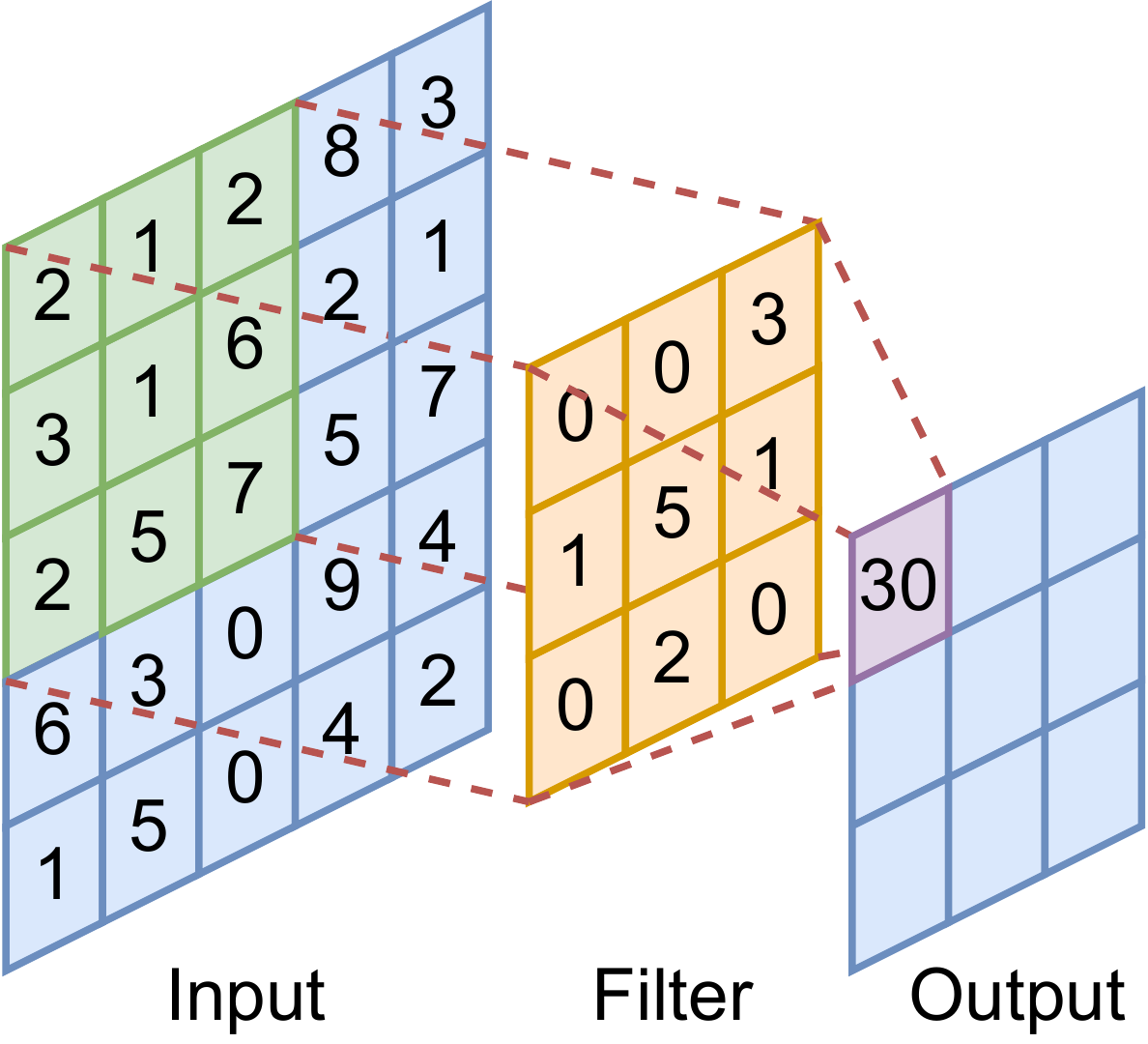}}\hspace{0.4in}
	\subfloat[Skip-layer connection.]{\includegraphics[width=0.3\textwidth]{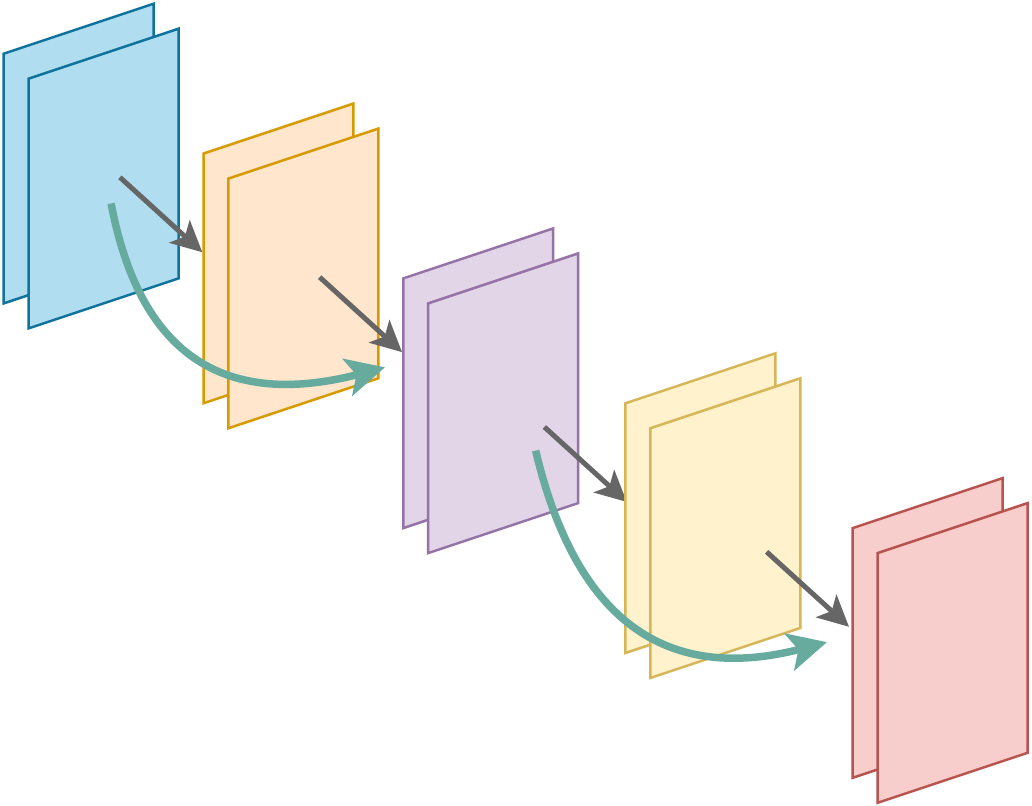}}
	\caption{Illustration of (a) convolution and (b) skip-layer connection.}
	\label{fig.cnnDemo}
\end{figure}

In real-world applications, convolutional neural networks (CNNs) are more popular than FNNs
 \citep{lecun1989backpropagation,krizhevsky2012imagenet,sermanet2013overfeat,he2016deep,chen2017deeplab,long2015fully,simonyan2014very,girshick2015fast}. In a CNN, each layer consists of several filters (channels) which are convolved with the input, as demonstrated in Figure \ref{fig.cnnDemo}(a). Due to such complexity in the CNN architecture, there are limited works on the representation theory of CNNs \citep{zhou2020universality,zhou2020theory,fang2020theory,petersen2020equivalence}. The constructed CNNs in these works become extremely wide (in terms of the size of each layer's output) as the approximation error goes to 0. In most real-life applications, the network width does not exceed 2048 \citep{zagoruyko2016wide,zhang2020resnest}.


 Convolutional residual networks (ConvResNet) is a special CNN architecture with skip-layer connections, as shown in Figure \ref{fig.cnnDemo}(b). Specifically, in addition to CNNs, ConvResNets have identity connections between inconsecutive layers. In many applications, ConvResNets outperform CNNs in terms of generalization performance and computational efficiency, and alleviate the vanishing gradient issue. Using this architecture, \citet{he2016deep} won the 1st place on the ImageNet classification task with a 3.57\% top 5 error in 2015.

\begin{table*}[htb!]\footnotesize
	\centering
\caption{\label{tab.results.compare}Comparison of our approximation theory and existing theoretical results.}
	\begin{tabular}{>{\centering\arraybackslash}p{4cm} cccc|c}
	\toprule\hline
		& \makecell{Network type} & \makecell{Function class} & \makecell{Low dim.\\ structure} & \makecell{Fixed\\width} & \makecell{Training }\\
		\hline
		\citet{yarotsky2017error} & FNN & Sobolev & \xmark & \xmark &   \multirow{4}{*}{\makecell{difficult to train \\due to the \\cardinality constraint}}\\
		 \cline{1-5}
		\citet{suzuki2018adaptivity} & FNN & Besov & \xmark & \xmark &  \\
		\cline{1-5}
	 \citet{chen1908nonparametric} & FNN & H\"{o}lder & \cmark & \xmark &   \\
	 \cline{1-5}
	 \citet{petersen2020equivalence}& CNN & FNN & \xmark & \xmark &  \\
	 \hline
	 \citet{zhou2020universality} & CNN & Sobolev & \xmark & \xmark & \multirow{3}{*}{\makecell{can be trained \\ without the \\cardinality constraint}}\\
	 \cline{1-5}
	 \citet{oono2019approximation} & ConvResNet & H\"{o}lder & \xmark&  \cmark &\\
	 \cline{1-5}
	 Ours& ConvResNet & Besov &\cmark & \cmark & \\
	 \hline\bottomrule
	\end{tabular}
\end{table*}

Recently, \citet{oono2019approximation} develops the only representation and statistical estimation theory of ConvResNets. \citet{oono2019approximation} proves that if the network architecture is properly set, ConvResNets with a fixed filter size and a fixed number of channels can universally approximate H\"{o}lder functions with arbitrary accuracy.
However, the sample complexity in \citet{oono2019approximation} grows exponentially with respect to the data dimension and therefore cannot well explain the empirical success of ConvResNets for high dimensional data. In order to estimate a $C^s$ function in $\RR^D$ with accuracy $\varepsilon$, the sample size required by \citet{oono2019approximation} scales as $\varepsilon^{-\frac{2s+D}{s}}$, which is far beyond the sample size used in practical applications. For example, the ImageNet data set consists of 1.2 million labeled images of size $224\times224\times3$. According to this theory, to achieve a 0.1 error, the sample size is required to be in the order of $10^{224\times224\times3}$ which greatly exceeds 1.2 million. Due to the curse of dimensionality, there is a huge gap between theory and practice.

We bridge such a gap by taking low-dimensional geometric structures of data sets into consideration. It is commonly believed that real world data sets exhibit low-dimensional structures due to rich local
regularities, global symmetries, or repetitive patterns \citep{hinton2006reducing,osher2017low,tenenbaum2000global}. For example, the ImageNet data set contains many images of the same object with certain transformations, such as rotation, translation, projection and skeletonization. As a result, the degree of freedom of the ImageNet data set is significantly smaller than the data dimension \citep{gong2019intrinsic}.

The function space considered in \citet{oono2019approximation} is the H\"{o}lder space in which functions are required to be differentiable everywhere up to certain order. In practice, the target function may not have high order derivatives. Function spaces with less restrictive conditions are more desirable. In this paper, we consider the Besov space $B^s_{p,q}$, which is more general than the H\"{o}lder space. In particular, the H\"{o}lder and Sobolev spaces are special cases of the Besov space:
$$
W^{s+\alpha,\infty}= \cH^{s,\alpha}\subseteq B^{s+\alpha}_{\infty,\infty} \subseteq B^{s+\alpha}_{p,q}
$$
for any $0<p,q\le \infty, s\in \NN$ and $\alpha\in(0,1]$.  For practical applications, it has been demonstrated in image processing that Besov norms can capture important features, such as edges \citep{jaffard2001wavelets}. Due to the generality of the Besov space, it is shown in \citet{suzuki2019deep} that kernel ridge estimators have a sub-optimal rate when estimating Besov functions.

In this paper,  we establish theoretical guarantees of ConvResNets for the approximation of Besov functions on a low-dimensional manifold, and a statistical theory on binary classification. 
Let $\cM$ be a $d$-dimensional compact Riemannian manifold isometrically embedded in $\RR^D$. Denote  the Besov space on $\cM$ as $\Bnorm(\cM)$ for $0<p,q\leq \infty$ and $0<s<\infty$. Our function approximation theory is established for $\Bnorm(\cM)$. For binary classification, we are given $n$ i.i.d. samples $\{(\xb_i,y_i)\}_{i=1}^n$ where $\xb_i \in \cM$ and $y_i\in\{-1,1\}$ is the label.
The label $y$ follows the Bernoulli-type distribution
$$	\PP(y=1|\xb)=\eta(\xb),\ \PP(y=-1|\xb)=1-\eta(\xb)
$$
for some  $\eta: \cM\rightarrow [0,1]$.
%
Our results (Theorem \ref{thm.approximation} and \ref{thm.classification}) are summarized as follows:

\begin{theorem*}[informal]
	Assume $s\geq d/p+1$.
	\begin{enumerate}[noitemsep,topsep=0pt,leftmargin=*]
		\item Given $\varepsilon\in(0,1)$, we construct a ConvResNet architecture such that, for any $f^*\in \Bnorm(\cM)$, if the weight parameters of this ConvResNet  are properly chosen, it gives rises to $\bar{f}$ satisfying
		$$
		\|\bar{f}-f^*\|_{L^{\infty}}\leq \varepsilon.
		$$
		\item Assume $\eta \in \Bnorm(\cM) $. Let $f_{\phi}^*$ be the minimizer of the population logistic risk. If the ConvResNet architecture is properly chosen, minimizing the empirical logistic risk gives rise to $\hf_{\phi,n}$ with the following excess risk bound
		\begin{align*}
			\EE(\cE_{\phi}(\hf_{\phi,n},f_{\phi}^*))\leq Cn^{-\frac{s}{2s+2(s\vee d)}}\log^4 n,
		\end{align*}
		where $\cE_{\phi}(\hf_{\phi,n},f_{\phi}^*)$ denotes the excess logistic risk of $\hf_{\phi,n}$ against $f_{\phi}^*$ and $C$ is a constant independent of $n$.
		
	\end{enumerate}
\end{theorem*}
We remark that the first part of the theorem above requires the network size to depend on the intrinsic dimension $d$ and only weakly depend on $D$. The second part is built upon the first part and shows a fast convergence rate of the excess risk in terms of $n$ where the exponent depends on $d$ instead of $D$.
 Our results demonstrate that ConvResNets are adaptive to low-dimensional structures of data sets.


\textbf{Related work.}
Approximation theories of FNNs with the ReLU activation have been established for Sobolev \citep{yarotsky2017error}, H\"{o}lder \citep{schmidt2017nonparametric} and Besov \citep{suzuki2018adaptivity} spaces. The networks in these works have certain cardinality constraint, i.e., the number of nonzero parameters is bounded by certain constant, which requires a lot of efforts for training.


Approximation theories of CNNs are developed in \citet{zhou2020universality,petersen2020equivalence,oono2019approximation}. Among these works, \citet{zhou2020universality} shows that CNNs can approximate Sobolev functions in $W^{s,2}$ for $s\geq D/2+2$ with an arbitrary accuracy $\varepsilon\in(0,1)$. 
The network in \citet{zhou2020universality} has width increasing linearly with respect to depth and has depth growing in the order of $\varepsilon^{-2}$ as $\varepsilon$ decreases to 0.
It is shown in \citet{petersen2020equivalence,zhou2020theory} that any approximation error achieved by FNNs can be achieved by CNNs. Combining \citet{zhou2020theory} and \citet{yarotsky2017error}, we can show that CNNs can approximate $W^{s,\infty}$ functions in $\RR^D$ with arbitrary accuracy $\varepsilon$. Such CNNs have the number of channels in the order of $\varepsilon^{-D/s}$ and a cardinality constraint. The only theory on ConvResNet can be found in \citet{oono2019approximation}, where an approximation theory for H\"{o}lder functions is proved for  ConvResNets with fixed width.

Statistical theories for binary classification by FNNs are established with the hinge loss \citep{ohn2019smooth,hu2020sharp} and the logistic loss \citep{kim2018fast}. 
Among these works, \citet{hu2020sharp} uses a parametric model given by a teacher-student network. The nonparametric results in \citet{ohn2019smooth,kim2018fast} are cursed by the data dimension, and therefore require  a large number of samples for high-dimensional data.

Binary classification by CNNs has been studied in \citet{kohler2020rate,kohler2020statistical,nitanda2018functional,huang2018learning}. Image binary classification is studied in \citet{kohler2020rate,kohler2020statistical} in which the probability function is assumed to be in a  hierarchical max-pooling model class. ResNet type classifiers are considered in \citet{nitanda2018functional,huang2018learning} while the generalization error is not given explicitly.

Low-dimensional structures of data sets are explored for neural networks in \citet{chui2018deep,shaham2018provable,chen1908nonparametric, chen2019efficient,schmidt2019deep,nakada2019adaptive,cloninger2020relu, chen2020doubly,montanelli2020error}. These works  show that, if data are near a low-dimensional manifold, the performance of FNNs depends on the intrinsic dimension of the manifold and only weakly depends on the data dimension. Our work focuses on ConvResNets for practical applications.

The networks in many aforementioned works have a cardinality constraint. From the computational perspective, training such networks requires substantial efforts \citep{han2016dsd, han2015learning,blalock2020state}. In comparison, the ConvResNet in \citet{oono2019approximation} and this paper does not require any cardinality constraint.
Additionally, our constructed network has a fixed filter size and a fixed number of channels, which is desirable for practical applications.

As a summary, we compare our approximation theory and existing results in Table \ref{tab.results.compare}.

The rest of this paper is organized as follows: In Section \ref{sec:pre}, we briefly introduce manifolds, Besov functions on manifolds and convolution. Our main results are presented in Section \ref{sec.mainresults}. We give a proof sketch in Section \ref{thm.approximation.proof} and conclude this paper in Section \ref{sec.conclusion}.



\section{Preliminaries}\label{sec:pre}

\textbf{Notations}: We use bold lower-case letters to denote vectors, upper-case letters to denote matrices, calligraphic letters to denote tensors, sets and manifolds. 
 For any $x>0$, we use $\lceil x\rceil$ to denote the smallest integer that is no less than $x$ and use $\lfloor x \rfloor$ to denote the largest integer that is no larger than $x$. For any $a,b\in \RR$, we denote $a\vee b=\max(a,b)$. For a function $f: \RR^d\rightarrow \RR$ and a set $\Omega\subset \RR^d$, we denote the restriction of $f$ to $\Omega$ by $f|_{\Omega}$. We use $\|f\|_{L^{p}}$ to denote the $L^{p}$ norm of $f$. We denote the Euclidean ball centered at $\cbb$ with radius $\omega$ by $B_{\omega}(\cbb)$.  


\subsection{Low-dimensional manifolds}

We first introduce some concepts on manifolds. We refer the readers to \citet{tu2010introduction,lee2006riemannian} for  details.
Throughout this paper, we let $\cM$ be a $d$-dimensional Riemannian manifold $\cM$ isometrically embedded in $\RR^D$ with $d \le D$.
We first introduce charts, an atlas and the partition of unity.
\begin{definition}[Chart]
A chart on $\cM$ is a pair $(U, \phi)$ where $U \subset \cM$ is open and $\phi : U \to \RR^d,$ is a homeomorphism (i.e., bijective, $\phi$ and $\phi^{-1}$ are both continuous).
\end{definition}
In a chart $(U,\phi)$, $U$ is called a coordinate neighborhood and $\phi$ is a coordinate system on $U$.
A collection of charts which covers $\cM$ is called an atlas of $\cM$.
\begin{definition}[$C^k$ Atlas]
A $C^k$ atlas for $\cM$ is a collection  of charts $\{(U_\alpha, \phi_\alpha)\}_{\alpha \in \cA}$ which satisfies $\bigcup_{\alpha \in \cA} U_\alpha = \cM$, and
 are pairwise $C^k$ compatible:
\begin{align*}
	&\phi_\alpha \circ \phi_\beta^{-1} : \phi_\beta(U_\alpha \cap U_\beta) \to \phi_\alpha(U_\alpha \cap U_\beta) \quad \textrm{and} \quad \phi_\beta \circ \phi_\alpha^{-1} : \phi_\alpha(U_\alpha \cap U_\beta) \to \phi_\beta(U_\alpha \cap U_\beta)
\end{align*} are both $C^k$ for any $\alpha,\beta\in \cA$. An atlas is called finite if it contains finitely many charts.
\end{definition}

\begin{definition}[Smooth Manifold] A smooth manifold is a manifold $\cM$ together with a $C^\infty$ atlas.
\end{definition}
The Euclidean space, the torus and the unit sphere are  examples of smooth manifolds. $C^s$ functions on a smooth manifold $\cM$ are defined as follows:
\begin{definition}[$C^s$ functions on $\cM$]
  Let $\cM$ be a smooth manifold and $f:\cM \rightarrow \RR$ be a function on $\cM$. We say $f$ is a $C^s$ function on $\cM$, if for every chart $(U,\phi)$ on $\cM$, the function $f\circ \phi^{-1}: \phi(U)\rightarrow \RR$ is a $C^s$ function. \end{definition}

We next define the $C^\infty$ partition of unity which is an important tool for the study of functions on manifolds.
\begin{definition}[Partition of Unity]
A $C^\infty$ partition of unity on a manifold $\cM$ is a collection of $C^\infty$ functions $\{\rho_{\alpha}\}_{\alpha\in\cA}$ with $\rho_\alpha: \cM \to [0,1]$ such that for any $\xb\in\cM$,
\begin{enumerate}
\item there is a neighbourhood of $\xb$ where only a finite number of the functions in $\{\rho_{\alpha}\}_{\alpha\in\cA}$ are nonzero, and
\item $\displaystyle\sum_{\alpha\in\cA} \rho_\alpha(\xb) = 1$.
\end{enumerate}
\end{definition}
An open cover of a manifold $\cM$ is called locally finite if every $\xb\in\cM$ has a neighbourhood which intersects with a finite number of sets in the cover. The following proposition shows that a $C^\infty$ partition of unity for a smooth manifold always exists \citep[Chapter 2, Theorem 15] {spivak1970comprehensive}. 
\begin{proposition}[Existence of a $C^\infty$ partition of unity]\label{thm:parunity}
Let $\{U_\alpha\}_{\alpha \in \cA}$ be a locally finite cover of a smooth manifold $\cM$. There is a $C^\infty$ partition of unity $\{\rho_\alpha\}_{\alpha=1}^\infty$ such that $\supp(\rho_\alpha) \subset U_\alpha$. 
\end{proposition}

Let $\{(U_\alpha,\phi_\alpha)\}_{\alpha\in \cA}$ be a $C^\infty$ atlas of $\cM$. Proposition \ref{thm:parunity} guarantees the existence of a partition of unity $\{\rho_\alpha\}_{\alpha\in\cA}$ such that $\rho_{\alpha}$ is supported on $U_{\alpha}$.

The reach of $\cM$ introduced by Federer \citep{federer1959curvature} is an important quantity defined below.
Let $d(\xb,\cM)=\inf_{\yb\in \cM}\|\xb-\yb\|_2$ be the distance from $\xb$ to $\cM$.
\begin{definition}[Reach \citep{federer1959curvature,niyogi2008finding}]
Define the set
\begin{align*}
	G=\{& \xb\in \RR^D: \exists\mbox{ distinct } \pb,\qb\in\cM \mbox{ such that } d(\xb,\cM)=\|\xb-\pb\|_2=\|\xb-\qb\|_2\}.	
\end{align*}

The closure of $G$ is
called the medial axis of $\cM$.
The reach of $\cM$ is defined as
$$
\tau=\inf_{\xb \in \cM} \ \inf_{\yb\in G}\|\xb-\yb\|_2.
$$
\end{definition}

\begin{figure}[htb!]
	\centering
	\includegraphics[width=0.6\textwidth]{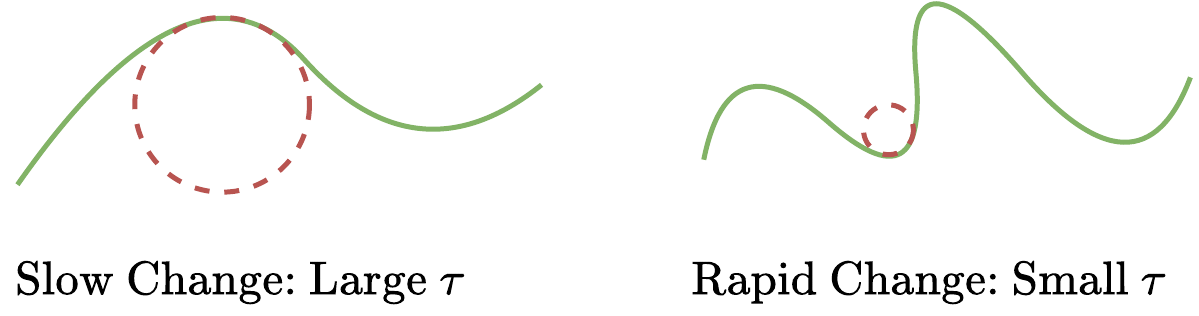}\vspace{-0.1in}
	\caption{Illustration of manifolds with large and small reach.
	}
	\label{fig.reach}
\end{figure}

We illustrate large and small reach in Figure \ref{fig.reach}. 
\subsection{Besov functions on a smooth manifold}

We next define Besov function spaces on  $\cM$,
which generalizes more elementary function spaces such as the Sobolev and H\"older spaces. 
To define Besov functions, we first introduce the modulus of smoothness.

\begin{definition}[Modulus of Smoothness \citep{devore1993constructive,suzuki2018adaptivity}]
  Let $\Omega \subset \RR^D$.
   For a function $f: \RR^D\rightarrow\RR$ be in $ L^p(\Omega)$ for $p>0$, the $r$-th modulus of smoothness of $f$ is defined by
  \begin{align*}
  &w_{r,p}(f,t)=\sup_{\|\hb\|_2\le t} \|\Delta_{\hb}^r(f)\|_{L^p}, \mbox{ where}&\\
  &\Delta_{\hb}^r(f)(\xb)=\begin{cases}
    \sum_{j=0}^r \binom{r}{j}
    (-1)^{r-j}f(\xb+j\hb) &\mbox{ if }\  \xb\in \Omega, \xb+r\hb\in\Omega,\\
    0 &\mbox{ otherwise}.
  \end{cases}
\end{align*}
\end{definition}
\begin{definition}[Besov Space $B_{p,q}^s(\Omega)$]\label{def.besovNorm}
For $0<p,q\leq \infty, s>0,r=\lfloor s\rfloor+1$, define the seminorm $|\cdot |_{B_{p,q}^s}$ as
  $$
  |f|_{\Bnorm(\Omega)}:=\begin{cases}
    \left( \displaystyle\int_0^{\infty} (t^{-s}w_{r,p}(f,t))^q\frac{dt}{t}\right)^{\frac{1}{q}} & \mbox{ if } q<\infty,\\
    \sup_{t>0} t^{-s}w_{r,p}(f,t) & \mbox{ if }q=\infty.
  \end{cases}
  $$
  The norm of the Besov space $B_{p,q}^s(\Omega)$ is defined as $\|f\|_{\Bnorm(\Omega)}:= \|f\|_{L^p(\Omega)}+|f|_{\Bnorm(\Omega)}$. The Besov space is $\Bnorm(\Omega)=\{ f\in L^p(\Omega) | \|f\|_{\Bnorm}<\infty\}$.
\end{definition}

We next define $\Bnorm$ functions on $\cM$ \citep{geller2011band,triebel1983theory,triebel1994theory}.
\begin{definition}[$\Bnorm$ Functions on $\cM$]\label{def.besovM}
Let $\cM$ be a compact smooth manifold of dimension $d$. Let $\{(U_i, \phi_{i})\}_{i=1}^{C_{\cM}}$ be a finite atlas on $\cM$ and $\{\rho_i\}_{i=1}^{C_{\cM}}$ be a partition of unity on $\cM$ such that $\supp(\rho_i)\subset U_i$. A function $f: \cM \to \RR$ is in $\Bnorm(\cM)$ if
\begin{align}
  \|f\|_{\Bnorm(\cM)}:=\sum_{i=1}^{C_{\cM}} \|(f\rho_i)\circ\phi_{i}^{-1}\|_{\Bnorm(\RR^d)}<\infty.
\end{align}
\end{definition}

Since $\rho_i$ is supported on $U_{i}$, the function $(f\rho_i)\circ \phi_{i}^{-1}$ is supported on $\phi(U_i)$. We can extend $(f\rho_i)\circ \phi_{i}^{-1}$ from $\phi(U_i)$ to $\mathbb{R}^d$ by setting the function to be $0$ on $\mathbb{R}^d\setminus \phi(U_i)$. The extended function lies in the Besov space $B^s_{p,q}(\RR^d)$ \citep[Chapter 7]{triebel1994theory}.


\subsection{Convolution and residual block}

In this paper, we consider one-sided stride-one convolution in our network. Let $\cW=\{\cW_{j,k,l}\}\in \RR^{C'\times K\times C}$ be a filter where $C'$ is the output channel size, $K$ is the filter size and $C$ is the input channel size. For $z\in \RR^{D\times C}$, the convolution of $\cW$ with $z$ gives $y\in \RR^{D\times C'}$ such that
\begin{align}
  y=\cW *z,\quad  y_{i,j}=\sum_{k=1}^K \sum_{l=1}^{C} \cW_{j,k,l} z_{i+k-1,l},
\end{align}
where $1\leq i\leq D, 1\leq j \leq C' $ and we set $z_{i+k-1,l}=0$ for $i+k-1>D$, as demonstrated in Figure \ref{fig.conv}(a).

The building blocks of ConvResNets are residual blocks. For an input $\xb$, each residual block computes
$$
\xb+F(\xb)
$$
where $F$ is a subnetwork consisting of convolutional layers (see more details in Section \ref{sec.ConvResNet}). A residual block is demonstrated in Figure \ref{fig.conv}(b).

\begin{figure}[htb!]
	\centering\vspace{-0.2in}
	\subfloat[Convolution.]{\includegraphics[width=0.48\textwidth]{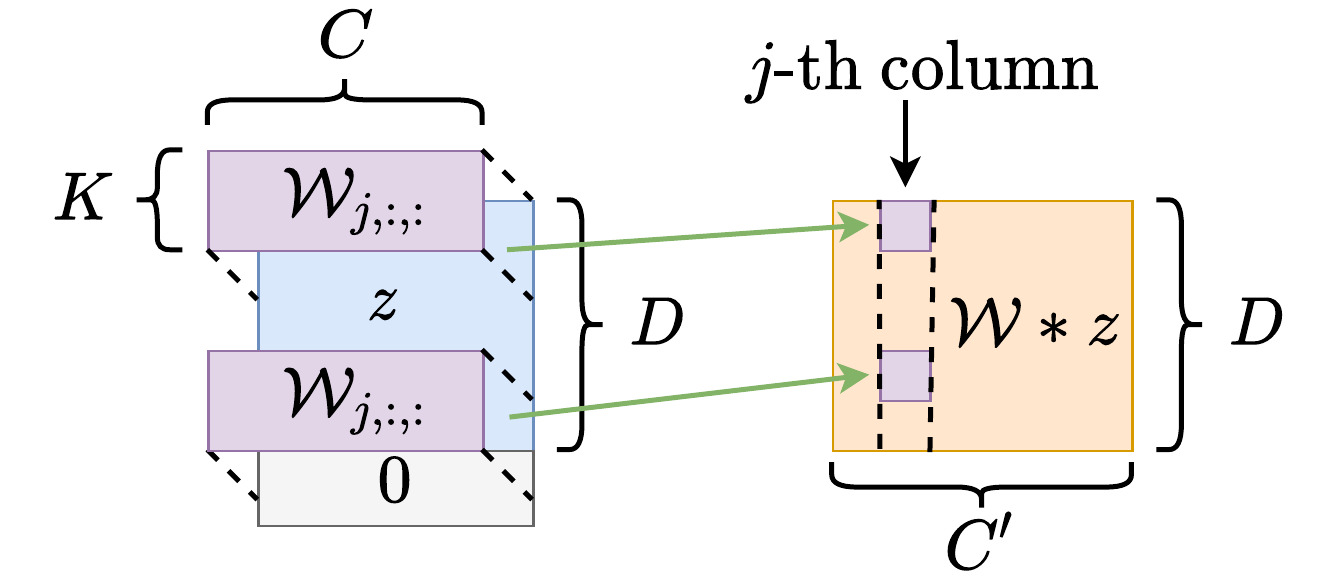}}
	\hspace{-0.3cm}
	\subfloat[A residual block.]{\includegraphics[width=0.26\textwidth]{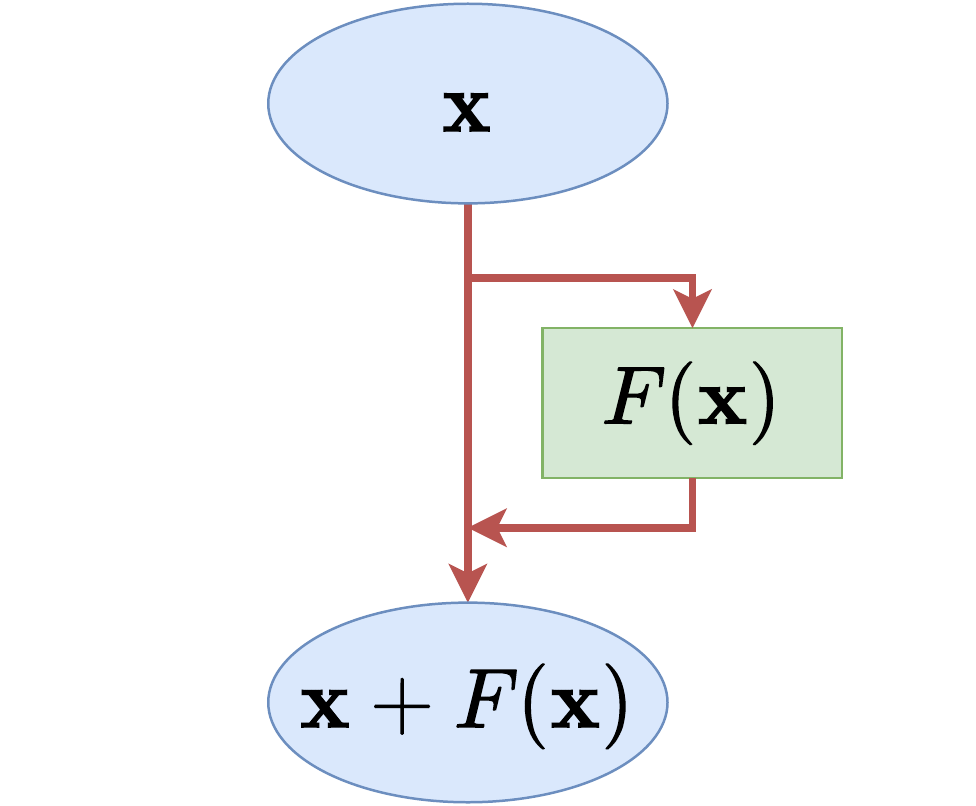}}\vspace{-0.1in}
	\caption{(a) Demonstration of $\cW *z$, where the input is $z\in  \RR^{D\times C}$, and the output is $\cW* z \in \RR^{D\times C'}$. Here $\cW=\{\cW_{j,k,l}\}\in \RR^{C'\times K\times C}$ is a filter where $C'$ is the output channel size, $K$ is the filter size and $C$ is the input channel size.
		$\cW_{j,:,:}$ is a $D\times C$ matrix for the $j$-th output channel. (b) Demonstration of a residual block.
	}\vspace{-0.1in}
	\label{fig.conv}
\end{figure}

\section{Theory}\label{sec.mainresults}
In this section, we first introduce the ConvResNet architecture, and then present our main results. 

\subsection{Convolutional residual neural network}\label{sec.ConvResNet}

We study the ConvResNet with the rectified linear unit ($\ReLU$) activation function: $\ReLU(z)=\max(z,0)$.
The ConvResNet we consider consists of a padding layer and several residual blocks followed by a fully connected feedforward layer.

We first define the padding layer.
Given an input $A\in\RR^{D\times C_1}$, the network first applies a padding operator $P:\RR^{D\times C_1}\rightarrow \RR^{D\times C_2}$ for some integer $C_2\geq C_1$ such that
$$
Z=P(A)=\begin{bmatrix}
	A & \zero &\cdots & \zero
\end{bmatrix} \in \RR^{D \times C_2}.
$$
Then the matrix $Z$ is passed through $M$ residual blocks.

In the $m$-th block, let
$\cW_m=\{ \cW_m^{(1)},...,\cW_m^{(L_m)}\}$ and $\cB_m=\{B_m^{(1)},...,B_m^{(L_m)}\}$ be a collection of filters and biases.
The $m$-th residual block maps a matrix from $\RR^{D\times C}$ to $\RR^{D\times C}$ by
$$
\Conv_{\cW_m,\cB_m}+\id,
$$
where $\id$ is the identity operator and
\begin{align} 
	&\Conv_{\cW_m,\cB_m}(Z)=\ReLU\Big( \cW_m^{(L_m)}*\cdots \cdots*\ReLU\left(\cW_m^{(1)}*Z+B_m^{(1)}\right)\cdots +B_m^{(L_m)}\Big),
	\label{eq.conv}
\end{align}
with $\ReLU$ applied entrywise. Denote
\begin{align}
	Q(\xb)=&\left(\Conv_{\cW_M,\cB_M}+\id\right)\circ\cdots \circ \left(\Conv_{\cW_1,\cB_1}+\id\right)\circ P(\xb).
	\label{eq.cnnBlock}
\end{align}
For networks only consisting of residual blocks, we define the network class as
\begin{align}
	\cC^{\Conv}(M,L,J,K,\kappa)=&\big\{ Q| Q(\xb) \mbox{ is in the form of (\ref{eq.cnnBlock}) with $M$ residual blocks. Each block has } \nonumber\\
	&\hspace{0.5cm}\mbox{filter size bounded by } K, \mbox{ number of channels bounded by } J,\nonumber\\
	&\hspace{0.5cm} \max_m L_m\leq L,\ \max_{m,l}\|\cW_m^{(l)}\|_{\infty} \vee \|B_m^{(l)}\|_{\infty} \leq \kappa\big\},
\end{align}
where $\norm{\cdot}_\infty$ denotes $\ell^\infty$ norm of a vector, and for a tensor $\cW$, $\norm{\cW}_{\infty} = \max_{j,k,l} |\cW_{j,k,l}|$. 

Based on the network $Q$ in (\ref{eq.cnnBlock}), a ConvResNet has an additional fully connected layer and can be expressed as
\begin{align}
	f(\xb)=WQ(\xb)+b
	\label{eq.cnn}
\end{align}
where $W$ and $b$ are the weight matrix and the bias in the fully connected layer. 
The class of ConvResNets is defined as
\begin{align}
	\cC(M,L,J,K,\kappa_1,\kappa_2,R)=	&\big\{ f| f(\xb)=WQ(\xb)+b \mbox{ with } Q\in \cC^{\Conv}(M,L,J,K,\kappa_1), \nonumber\\
	&\hspace{0.5cm}\|W\|_{\infty} \vee |b| \leq \kappa_2, \|f\|_{L^{\infty}}\leq R \big\}.
\end{align}
Sometimes we do not have restriction on the output, we omit the parameter $R$ and denote the network class by $\cC(M,L,J,K,\kappa_1,\kappa_2)$.
\subsection{Approximation theory}

Our approximation theory is based on the following assumptions of $\cM$ and the object function $f^*:\cM\rightarrow \RR$.
\begin{assumption}\label{assum.M}
  $\cM$ is a $d$-dimensional compact smooth Riemannian manifold isometrically embedded in $\RR^D$. There is a constant $B$ such that for any $\xb\in\cM$, $\|\xb\|_{\infty}\leq B$.
\end{assumption}
\begin{assumption}\label{assum.reach}
  The reach of $\cM$ is $\tau>0$.
\end{assumption}

\begin{assumption}\label{assum.f}
  Let $0<p,q\leq \infty$, $ d/p+1\leq s<\infty$. Assume $f^*\in \Bnorm(\cM)$ and $\|f^*\|_{\Bnorm(\cM)}\leq c_0$ for a constant $c_0>0$. Additionally, we assume $\|f^*\|_{L^{\infty}}\leq R$ for a constant $R >0$.
\end{assumption}
Assumption \ref{assum.f} implies that $f^*$ is Lipschitz continuous \citep[Section 2.7.1 Remark 2 and Section 3.3.1]{triebel1983theory}.

Our first result is the following universal approximation error of ConvResNets for Besov functions on $\cM$.
\begin{figure}[htb!]
	\centering
	\includegraphics[width=0.7\textwidth,trim={2.4cm 0 1.2cm 0},clip]{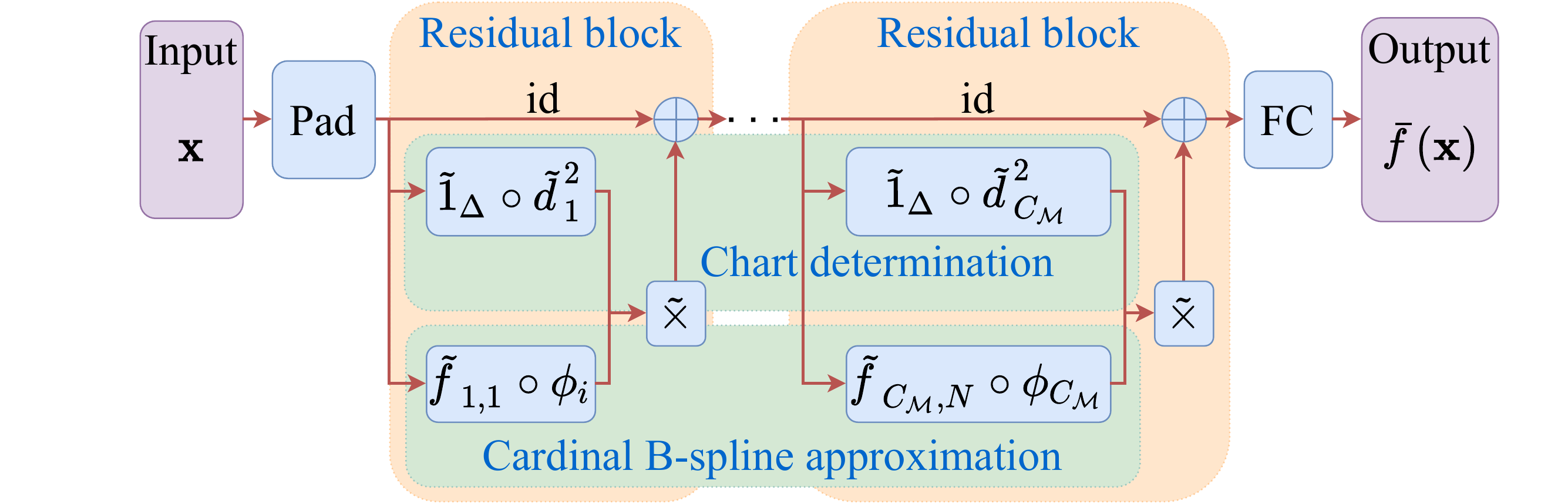}
	\caption{The ConvResNet in Theorem \ref{thm.approximation} contains a padding layer, $M$ residual blocks, and a fully connected (FC) layer. 
	}\label{fig.res-CNN}
\end{figure}
\begin{theorem}\label{thm.approximation}
 Assume Assumption \ref{assum.M}-\ref{assum.f}. 
 For any $\varepsilon\in(0,1)$ and positive integer $K\in[2,D]$, there is a ConvResNet architecture $\cC(M,L,J,K,\kappa_1,\kappa_2)$ such that, for any $f^*\in \Bnorm(\cM)$, if the weight parameters of this ConvResNet are properly chosen, the network yields a function $\bar{f}\in\cC(M,L,J,K,\kappa_1,\kappa_2)$ satisfying 
\begin{align}
	\|\bar{f}-f^*\|_{L^\infty}\leq \varepsilon.
	\label{eq.thm1}
\end{align}
  Such a network architecture has
  \begin{align}
  &M=O\left(\varepsilon^{-d/s}\right),\ L=O(\log (1/\varepsilon)+D+\log D),\  J=O(D),\ \kappa_1=O(1),\ \log \kappa_2=O(\log^2(1/\varepsilon)).
  \end{align}
  The constant hidden in $O(\cdot)$ depend on $d$,  $s$,$\frac{2d}{sp-d}$, $p$, $q$,$c_0,\tau$ and the surface area of $\cM$. 
\end{theorem}

The architecture of the ConvResNet in Theorem \ref{thm.approximation} is illustrated in Figure \ref{fig.res-CNN}. It has the following properties:
\begin{itemize}
	\item The network has a fixed filter size and a fixed number of channels. 
	\item There is no cardinality constraint.
	\item The network size depends on the intrinsic dimension $d$, and only weakly depends on $D$.
\end{itemize}

Theorem \ref{thm.approximation} can be compared with \citet{suzuki2018adaptivity} on the approximation theory for Besov functions in $\RR^D$ by FNNs as follows: (1) To universally approximate Besov functions in $\RR^D$ with $\varepsilon$ error,
the FNN constructed in \citet{suzuki2018adaptivity} requires $O\left(\log(1/\varepsilon)\right)$ depth,   $ O\left(\varepsilon^{-D/s}\right)$ width and $ O\left(\varepsilon^{-D/s}\log(1/\varepsilon)\right)$ nonzero parameters. By exploiting the manifold model, our network size depends on the intrinsic dimension $d$ and weakly depends on $D$. (2) The ConvResNet in Theorem \ref{thm.approximation} does not require any cardinality constraint, while such a constraint is needed in \citet{suzuki2018adaptivity}. 

\subsection{Statistical theory}

We next consider binary classification on $\cM$. For any $\xb\in\cM$, denote its label by $y\in\{-1,1\}$. The label $y$ follows the following Bernoulli-type distribution
\begin{align}
	\PP(y=1|\xb)=\eta(\xb),\ \PP(y=-1|\xb)=1-\eta(\xb)
	\label{eq.y.eta}
\end{align}
for some $\eta:\cM\rightarrow [0,1]$.


We assume the following data model:
\begin{assumption}\label{assum.class.sample}
	We are given i.i.d. sample $\{(\xb_i,y_i)\}_{i=1}^n$, where $\xb_i \in \cM$, and the $y_i$'s are sampled according to (\ref{eq.y.eta}). 
\end{assumption}



In binary classification, a classifier $f$ predicts the label of $\xb$ as $\sign(f(\xb))$. 
To learn the optimal classifier, we consider the logistic loss 
$\phi(z)=\log(1+\exp(-z))$. The logistic risk $\cE_{\phi}(f)$ of a classifier $f$ is defined as
\begin{align}
  \cE_{\phi}(f)=\EE(\phi(yf(\xb))).
  \label{eq.class.risk}
\end{align}
The minimizer of $\cE_{\phi}(f)$ is denoted by $f_{\phi}^*$, which satisfies 
\begin{align}
	f_{\phi}^*(\xb)=\log \frac{\eta(\xb)}{1-\eta(\xb)}.
	\label{eq.class.min}
\end{align}
For any classifier $f$, we define its logistic excess risk as
\begin{align}
  \cE_{\phi}(f,f_{\phi}^*)=\cE_{\phi}(f)-\cE_{\phi}(f_{\phi}^*).
  \label{eq.class.excess}
\end{align}

In this paper, we consider ConvResNets with the following architecture:
\begin{align}
	\cC^{(n)}=\big\{f|f=\bar{g}_2\circ\bar{h}\circ \bar{g}_1\circ\bar{\eta} \mbox{ where} \ &\bar{\eta}\in \cC^{\Conv}\left(M_1,L_1,J_1,K,\kappa_1\right),\ \bar{g}_1\in \cC^{\Conv}\left(1,4,8,1,\kappa_2\right),\nonumber\\
	& \bar{h}\in \cC^{\Conv}\left(M_2,L_2,J_2,1,\kappa_1\right),\ \bar{g}_2\in \cC\left(1,3,8,1,\kappa_3,1,R\right) \big\},
\end{align}
where $M_1,M_2,L,J,K,\kappa_1,\kappa_2,\kappa_3$ are some parameters to be determined. 

The empirical classifier is learned by minimizing the empirical logistic risk:
\begin{align}
	\hf_{\phi,n}=\argmin_{f\in \cC^{(n)}} \frac{1}{n}\sum_{i=1}^n \phi(y_if(\xb_i)).
\end{align}

We establish an upper bound on the excess risk of $\hf_{\phi,n}$:


\begin{theorem}\label{thm.classification}
Assume Assumption \ref{assum.M}, \ref{assum.reach} and \ref{assum.class.sample}. Assume $0<p,q\leq \infty$, $0<s<\infty$, $s\geq d/p+1$ and $\eta\in \Bnorm(\cM)$ with $\|\eta\|_{\Bnorm}\leq c_0$ for some constant $c_0$. For any $2\leq K\leq D$, we set 
\begin{align*}
	&M_1=O\left(n^{\frac{2d}{s+2(s\vee d)}}\right),\  M_2=O\left(n^{\frac{2s}{s+2(s\vee d)}}\right),\ L_1=O(\log (1/\varepsilon)+D+\log D),\ L_2=O(\log (1/\varepsilon)),\\
	& J_1=O(D), \ J_2=O(1),\ \kappa_1=O(1),\  \log \kappa_2=O(\log^2 n),\ \kappa_3=O(\log n),\ R=O(\log n)
\end{align*}
for $\cC^{(n)}$.
 Then
\begin{align}
  \EE(\cE_{\phi}(\hf_{\phi,n},f_{\phi}^*))\leq Cn^{-\frac{s}{2s+2(s\vee d)}}\log^4 n
  \label{eq.thm2.excessrisk}
\end{align}
for some constant $C$. Here $C$ is linear in $D \log D$ and additionally depends on $d,s,\frac{2d}{sp-d},p,q,c_0,\tau$ and the surface area of $\cM$.  The constant hidden in $O(\cdot)$ depends on $d,s,\frac{2d}{sp-d},p,q,c_0,\tau$ and the surface area of $\cM$. 
\end{theorem}
Theorem \ref{thm.classification} shows that a properly designed ConvResNet gives rise to an empirical classifier, of which the excess risk converges at a fast rate with an exponent depending on the intrinsic dimension $d$, instead of $D$. 

Theorem \ref{thm.classification}
is proved in Appendix \ref{sec.proof.class}.
Each building block of $\cC^{(n)}$ is constructed for the following purpose:
\begin{itemize}
	\item $\bar{g}_1\circ\bar{\eta}$ is designed to approximate a truncated $\eta$ on $\cM$, which is realized by Theorem \ref{thm.approximation}.
	\item $\bar{g}_2\circ\bar{h}$ is designed to approximate a truncated univariate function $\log\frac{z}{1-z}$.
\end{itemize}
\section{Proof of Theorem \ref{thm.approximation}}
\label{thm.approximation.proof}

We provide a proof sketch of Theorem \ref{thm.approximation} in this section. More technical details are deferred to Appendix \ref{sec.thm1.lemma.proof}.

We prove Theorem \ref{thm.approximation} in the following four steps:
  \begin{enumerate}
  	\item Decompose $f^* = \sum_{i} f_i$ as a sum of locally supported functions according to the manifold structure.
  	\item Locally approximate each $f_i$ using cardinal B-splines.
  	\item Implement the cardinal B-splines using CNNs.
  	\item Implement the sum of all CNNs by a ConvResNet for approximating $f^*$.
  \end{enumerate}
\textbf{Step 1: Decomposition of $f^*$.}\\
\textbf{$\bullet$ Construct an atlas on $\cM$.} { Since the manifold $\cM$ is compact, we can cover $\cM$ by a finite collection of open balls $B_\omega(\cbb_i)$ for $i = 1, \dots, C_{\cM}$, where $\cbb_i$ is the center of the ball and $\omega$ is the radius to be chosen later. Accordingly, the manifold is partitioned as $\cM = \bigcup_{i} U_i$ with $U_i = B_\omega(\cbb_i) \bigcap \cM$. We choose $\omega < \tau/2$ such that $U_i$ is diffeomorphic to an open subset of $\RR^d$ \citep[Lemma 5.4]{niyogi2008finding}. The total number of partitions is then bounded by $C_{\cM}\leq \left\lceil \frac{{\rm SA}(\cM)}{\omega^d}T_d\right\rceil,$
where ${\rm SA}(\cM)$ is the surface area of $\cM$ and $T_d$ is the average number of $U_i$'s that contain a given point on $\cM$ \citep[Chapter 2 Equation (1)]{Conway:1987:SLG:39091}.

On each partition, we define a projection-based transformation $\phi_i$ as
\begin{align*}
\phi_i(\xb) = a_iV_i^{\top}(\xb-\cbb_i)+\bbb_i,
\end{align*}
where the scaling factor $a_i\in\RR$ and the shifting vector $\bbb_i\in\RR^d$ ensure $\phi_i(U_i)\subset [0,1]^d$, and the column vectors of $V_i \in \RR^{D \times d}$ form an orthonormal basis of the tangent space $T_{\cbb_i}(\cM)$. The atlas on $\cM$ is the collection $(U_i, \phi_i)$ for $i = 1, \dots, \cM$. See Figure \ref{fig.EuclideanBall} for a graphical illustration of the atlas.
} 
%
  
\begin{figure}[t]
	\centering
	\includegraphics[width=0.75\textwidth,trim={0 0 0 6.5cm},clip]{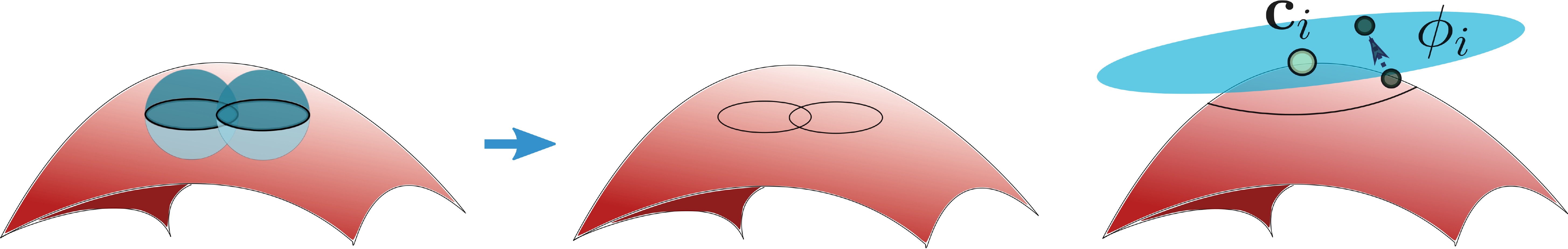}
	\caption{An atlas given by covering $\cM$ using Euclidean balls.}\label{fig.EuclideanBall}
	\vspace{-0.4cm}
\end{figure}

\noindent\textbf{$\bullet$ Decompose $f^*$ according to the atlas.} 
{ We decompose $f^*$ as
\begin{align}
	f^*=\sum_{i=1}^{C_{\cM}} f_i \quad \textrm{with}\quad f_i=f\rho_i,
	\label{eq.f.decompose}
\end{align}
where $\{\rho_i\}_{i=1}^{C_{\cM}}$ is a $C^\infty$ partition of unity with $\textrm{supp}(\phi_i) \subset U_i$. The existence of such a $\{\rho_i\}_{i=1}^{C_{\cM}}$ is guaranteed by Proposition \ref{thm:parunity}. As a result, each $f_i$ is supported on a subset of $U_i$, and therefore, we can rewrite \eqref{eq.f.decompose} as
\begin{align}
	f^*=\sum_{i=1}^{C_{\cM}} (f_i\circ\phi_i^{-1})\circ\phi_i \times \mone_{U_i} \quad \textrm{with}\quad f_i=f\rho_i,
	\label{eq.f.decompose1}
\end{align}
where $\mone_{U_i}$ is the indicator function of $U_i$. Since $\phi_i$ is a bijection between $U_i$ and $\phi_i(U_i)$, $f_i\circ\phi_i^{-1}$ is supported on $\phi_i(U_i)\subset [0,1]^d$. 
We extend $f_i\circ\phi_i^{-1}$ on $[0,1]^d\backslash\phi_i(U_i)$ by 0. The extended function is in $\Bnorm([0,1]^d)$ (see Lemma \ref{lem.besov.extend} in Appendix \ref{lem.besov.extend.proof}).
This allows us to use cardinal B-splines to locally approximate each $f_i \circ \phi_i^{-1}$ as detailed in {\bf Step 2}. 
}


\noindent\textbf{Step 2: Local cardinal B-spline approximation.}
{We approximate $f_i \circ \phi_i^{-1}$ using cardinal B-splines $\tf_i$ as
\begin{align} 
	f_i \circ \phi_i^{-1} \approx \tf_i\equiv\sum_{j=1}^N \tf_{i,j} ~~\textrm{with}~~ \tf_{i,j}=\alpha^{(i)}_{k,\jb}M_{k,\jb,m}^d ,
	\label{eq.f-BSpline}
\end{align}
where $\alpha^{(i)}_{k,\jb}\in \RR$ is a coefficient and $M_{k,\jb,m}^d:[0,1]^d\rightarrow \RR$ denotes a cardinal B-spline with indecies $k,m\in \NN^+,\jb\in \RR^d$. Here $k$ is a scaling factor, $\jb$ is a shifting vector, $m$ is the degree of the B-spline and $d$ is the dimension (see a formal definition in Appendix \ref{sec.B-spline}).

Since $s\geq d/p+1$ (by Assumption \ref{assum.f}), setting $r=+\infty, m=\lceil s \rceil +1$ in Lemma \ref{lem.BSplAppBesov} (see Appendix \ref{sec.lem.BSplAppBesov.proof}) and applying Lemma \ref{lem.besov.extend} gives
\begin{align} 
	\left\|\tf_i-f_i\circ\phi_i^{-1}\right\|_{L^\infty} \leq Cc_0N^{-s/d}
	\label{eq.tfierror}
\end{align}
for some constant $C$ depending on $s,p,q$ and $d$. 
}

Combining (\ref{eq.f.decompose1}) and (\ref{eq.f-BSpline}), we approximate $f^*$ by
\begin{align}
	\tf^*\equiv \sum_{i=1}^{C_{\cM}} \tf_i\circ\phi_i\times \mone_{U_i}=\sum_{i=1}^{C_{\cM}}\sum_{j=1}^N \tf_{i,j}\circ\phi_i\times \mone_{U_i}.
	\label{eq.f.decompose2}
\end{align}
Such an approximation has error
$$
\|\tf^*-f^*\|_{L^\infty} \leq CC_{\cM}c_0 N^{-s/d}.
$$

\noindent\textbf{Step 3: Implement local approximations in Step 2 by CNNs.}
{In {\bf Step 2}, (\ref{eq.f.decompose2}) gives a natural approximation of $f^*$. 	
In the sequel, we aim to implement all ingredients of $\tf_{i,j}\circ\phi_i\times \mone_{U_i}$ using CNNs. In particular, we show that CNNs can implement the cardinal B-spline $\tf_{i,j}$, the linear projection $\phi_i$, the indicator function $\mathds{1}_{U_i}$, and the multiplication operation.

\noindent\textbf{$\bullet$ Implement $\mone_{U_i}$ by CNNs.}
Recall our construction of $U_i$ in {\bf Step 1}. For any $\xb\in\cM$, we have $\mone_{U_i}(\xb)=1$ if $d_i^2(\xb) = \norm{\xb - \cbb_i}_2^2\leq \omega^2$; otherwise $\mone_{U_i}(\xb)= 0$.

To implement $\mathds{1}_{U_i}$, we rewrite it as the composition of a univariate indicator function $\mathds{1}_{[0, \omega^2]}$ and the distance function $d_i^2$:
\begin{align}
\mone_{U_i}(\xb)=\mone_{[0,\omega^2]}\circ d_i^2(\xb) \quad \textrm{for} \quad \xb\in\cM.
\label{eq.indicator.1}
\end{align}
We show that CNNs can efficiently implement both $\mathds{1}_{[0, \omega^2]}$ and $d_i^2$. Specifically, given $\theta \in (0, 1)$ and $\Delta \geq 8DB^2 \theta$, there exist CNNs that yield functions $\tilde{\mathds{1}}_\Delta$ and $\tilde{d}_i^2$ satisfying
\begin{align}
	\|\tdi-d_i^2\|_{L^{\infty}}\leq 4B^2D\theta
	\label{eq.distance.error}
\end{align}
and
\begin{align}
	\mtoned\circ\tdi(\xb)=\begin{cases}1,   \mbox{ if } \xb\in U_i, d_i^2(\xb)\leq \omega^2-\Delta,\\	
	0,   \mbox{ if } \xb\notin U_i,\\
	\mbox{between 0 and 1},   \mbox{ otherwise}.
\end{cases}
\label{eq.indicator.error}
\end{align}
We also characterize the network sizes for realizing $\tilde{\mathds{1}}_\Delta$ and $\tilde{d}_i^2$: The network for $\mtoned$ has $O(\log(\omega^2/\Delta))$ layers, $2$ channels and all weight parameters bounded by $\max(2,|\omega^2-4B^2D\theta|)$; the network for $\tdi$ has $O(\log(1/\theta)+D)$ layers, $6D$ channels and all weight parameters bounded by $4B^2$. More technical details are provided in Lemma \ref{lem.indicator} in Appendix \ref{sec.lem.indicator.proof}.
}

\noindent\textbf{$\bullet$ Implement $\tf_{i,j}\circ\phi_i$ by CNNs.} Since $\phi_i$ is a linear projection, it can be realized by a single-layer perceptron. By Lemma \ref{lem.cnnRealization} (see Appendix \ref{sec.lem.cnnRealization.proof}), this single-layer perceptron can be realized by a CNN, denoted by $\phi_i^{\rm CNN}$. 

For $\tf_{i,j}$, Proposition \ref{prop.ReLUBR} (see Appendix \ref{proof.prop.ReLUBR}) shows that for any $\delta\in(0,1)$ and $2\leq K\leq d$, there exists a CNN $\tf_{i,j}^{\rm CNN}\in \cF^{\rm CNN}(L,J,K,\kappa,\kappa)$ with
\begin{align*}
	\begin{aligned} 
		&L= O\left( \log\frac{1}{\delta}\right), J=O(1), \kappa=O\left( \delta^{-(\log 2)(\frac{2d}{sp-d}+\frac{c_1}{d})}\right)
	\end{aligned}
\end{align*}
such that when setting $N=C_1\delta^{-d/s}$, we have
\begin{align}
	\Big\|\sum_{j=1}^N \tf^{\rm CNN}_{i,j}-f_i\circ \phi_i^{-1}\Big\|_{L^\infty(\phi_i(U_i))}\leq \delta,
	\label{eq.tfi.error}
\end{align}
 where $C_1$ is a constant depending on $s,p,q$ and $d$. The constant hidden in $O(\cdot)$ depends on $d,s,\frac{2d}{sp-d},p,q,c_0$. The CNN class $\cF^{\rm CNN}$ is defined in Appendix \ref{sec.CNNMLP}.

\noindent\textbf{$\bullet$ Implement the multiplication $\times$ by a CNN.}
According to Lemma \ref{lem.multiplication} (see Appendix \ref{sec.lem.multiplication}) and Lemma \ref{lem.cnnRealization}, for any $\eta\in(0,1)$, the multiplication operation $\times$ can be approximated by a CNN $\ttimes$ with $L^{\infty}$ error $\eta$: 
\begin{align}
	\|a\times b-\ttimes(a,b)\|_{L^{\infty}}\leq \eta.
	\label{eq.multiplication.error}
\end{align}
 Such a CNN has $O\left(\log 1/\eta\right)$ layers, $6$ channels. All parameters are bounded by $\max(2c_0^2,1)$.

\noindent\textbf{Step 4: Implement  $\tf^*$ by a ConvResNet.}
We assemble all CNN approximations in \textbf{Step 3} together and show that the whole approximation can be realized by a ConvResNet.

\noindent\textbf{$\bullet$ Assemble all ingredients together.} 
Assembling all CNN approximations together gives an approximation of $\tf_{i,j}\circ \phi_i \times \mone_{U_i}$ as
\begin{align} 
\mathring{f}_{i,j}\equiv\ttimes\left(\tf_{i,j}^{\rm CNN}\circ \phi_i^{\rm CNN},  \mtoned\circ\tdi\right).
\label{eq.bfij}
\end{align}
After substituting (\ref{eq.bfij}) into (\ref{eq.f.decompose2}), we approximate the target function $f^*$ by
\begin{align}
	\mathring{f} \equiv\sum_{i=1}^{C_{\cM}} \sum_{j=1}^N \mathring{f}_{i,j}.
	\label{eq.f.approx.mlp}
\end{align}
The approximation error of $\mathring{f}$ is analyzed in Lemma \ref{lem.totalerror} (see Appendix \ref{sec.lem.totalerror}). According to Lemma \ref{lem.totalerror}, the approximation error can be bounded as follows:
\begin{align*}
&\qquad \|\mathring{f}-f^*\|_{L^\infty}\leq \sum_{i=1}^{C_{\cM}} (A_{i,1}+A_{i,2}+A_{i,3}) \quad \mbox{with } \\
	&A_{i,1}=\sum_{j=1}^N\Big\|\ttimes(\tf_{i,j}^{\rm CNN}\circ \phi_i^{\rm CNN},\mtoned\circ \tdi)-(\tf_{i,j}^{\rm CNN}\circ \phi_i^{\rm CNN})\times (\mtoned\circ \tdi)\Big\|_{L^\infty}\leq N\eta,\\
	&A_{i,2}=\Big\|\Big(\sum_{j=1}^N \left(\tf_{i,j}^{\rm CNN}\circ \phi_i^{\rm CNN}\right)\Big)\times(\mtoned\circ \tdi)-f_i\times (\mtoned\circ \tdi)\Big\|_{L^\infty}\leq \delta,\\
	&A_{i,3}=\|f_i\times(\mtoned\circ \tdi)-f_i\times \mone_{U_i}\|_{L^\infty}\leq \frac{c(\pi+1)}{\omega(1-\omega/\tau)}\Delta,
\end{align*}
where $\delta,\eta,\Delta$ and $\theta$ are defined in (\ref{eq.tfi.error}), (\ref{eq.multiplication.error}), (\ref{eq.indicator.error}) and (\ref{eq.distance.error}), respectively.
For any $\varepsilon\in(0,1)$, with properly chosen $\delta,\eta,\Delta$ and $\theta$ as in (\ref{eq.parameterchoose}) in Lemma \ref{lem.totalerror}, one has
\begin{align}
\|\mathring{f}-f^*\|_{L^\infty}\leq \varepsilon.
\label{eq.thm1.ringerror}
\end{align}
With these choices, the network size of each CNN is quantified in Appendix \ref{sec.cnnsize}.

\noindent\textbf{$\bullet$ Realize $\mathring{f}$ by a ConvResNet.} 
Lemma \ref{lem.bfijcnn} (see Appendix \ref{sec.lem.bfijcnn.proof}) shows that for every $\mathring{f}_{i,j}$, there exists $\bar{f}^{\rm CNN}_{i,j}\in \cF^{\rm CNN}(L,J,K,\kappa_1,\kappa_2) $ with
$L=O(\log 1/\varepsilon+D+\log D), J=O(D), \kappa_1=O(1), \log \kappa_2=O\left(\log^2 1/\varepsilon\right)
$
such that 
$\bar{f}^{\rm CNN}_{i,j}(\xb)=\mathring{f}_{i,j}(\xb)$
for any $\xb\in\cM$. As a result, the function $\mathring{f}$ in (\ref{eq.f.approx.mlp}) can be expressed as a sum of CNNs: 
\begin{align}
	\mathring{f}=\bar{f}^{\rm CNN}\equiv \sum_{i=1}^{C_{\cM}} \sum_{j=1}^N \bar{f}^{\rm CNN}_{i,j},
\end{align}
where $N$ is chosen of $O\left(\varepsilon^{-d/s}\right)$ (see Proposition \ref{prop.ReLUBR} and Lemma \ref{lem.totalerror}).
Lemma \ref{lem.cnn.convresnet} (see Appendix \ref{sec.lem.cnn.convresnet.proof}) shows that $\bar{f}^{\rm CNN}$ can be realized by $\bar{f}\in \cC(M,L,J,\kappa_1, \kappa_2)$ with
\begin{align*}
	&M=O\left(\varepsilon^{-d/s}\right), L=O(\log (1/\varepsilon)+D+\log D), J=O(D), \kappa_1=O(1), \log \kappa_2=O\left(\log^2 (1/\varepsilon)\right).
\end{align*}


\section{Conclusion}\label{sec.conclusion}
Our results show that ConvResNets are adaptive to low-dimensional geometric structures of data sets.
Specifically, we establish a universal approximation theory of ConvResNets for Besov functions on a $d$-dimensional manifold $\cM$. Our network size depends on the intrinsic dimension $d$ and only weakly depends on $D$. We also establish  a statistical theory of ConvResNets for binary classification when the given data are located on $\cM$. The classifier is learned by minimizing the empirical logistic loss.
We prove that if the ConvResNet architecture is properly chosen, the excess risk of the learned classifier decays at a fast rate depending on the intrinsic dimension of the manifold.

Our ConvResNet has many practical properties: it has a fixed filter size and a fixed number of channels. Moreover, it does not require any cardinality constraint, which is beneficial to training.

Our analysis can be extended to multinomial logistic regression for multi-class classification. In this case, the network will output a vector where each component represents the likelihood of an input belonging to certain class. By assuming that each likelihood function is in the Besov space, we can apply our  analysis to approximate each function by a ConvResNet.

\bibliography{ref}
\bibliographystyle{ims}

\appendix

\newpage
\onecolumn

\section*{\centering
	Supplementary Materials for Besov Function Approximation and Binary Classification on  \\
	Low-Dimensional Manifolds Using Convolutional Residual Networks
}

\textbf{Notations:} Throughout our proofs, we define the following notations: For two functions $f:\Omega\rightarrow \RR$ and $g: \Omega\rightarrow \RR$ defined on some domain $\Omega$, we denote $f\lesssim g$ if there is a constant $C$ such that $f(\xb)\leq Cg(\xb)$ for all $\xb\in \Omega$. Similarly, we denote $f\gtrsim g$ if there is a constant $C$ such that $f(\xb)\geq Cg(\xb)$ for all $\xb\in \Omega$. We denote $f \asymp g$ if $f\lesssim g$ and $f\gtrsim g$. We use $\NN$ to denote the set of all nonnegative integers. For a real number $a$, we denote $a_+=\max(a,0)$ and $a_-=\min (a,0)$.

The proof of Theorem \ref{thm.approximation} is sketched in Section \ref{thm.approximation.proof}. In this supplementary material, we prove Theorem \ref{thm.classification} in Section \ref{sec.proof.class}. We define convolutional network and multi-layer perceptrons classes in Section \ref{sec.CNNMLP}, based on which the lemmas used in Section \ref{thm.approximation.proof} are proved in Section \ref{sec.thm1.lemma.proof}. The lemmas used in Section \ref{sec.proof.class} are proved in Section \ref{sec.proof.class.lemma}.
\section{Proof of Theorem \ref{thm.classification}}\label{sec.proof.class} 
\subsection{Basic definitions and tools}
We first define the bracketing entropy and covering number which are used in the proof of Theorem \ref{thm.classification}.

\begin{definition}[Bracketing entropy]
A set of function pairs $\{(f_i^L,f_i^U)\}_{i=1}^N$ is called a $\delta$-bracketing of a function class $\cF$ with respect to the norm $\|\cdot\|$ if for any $i$, $\|f_i^U-f_i^L\|\leq \delta$ and for any $f\in \cF$, there exists a pair $(f_i^L,f_i^U)$ such that $f_i^L\leq f \leq f_i^U$. The $\delta$-bracketing number is defined as the cardinality of the minimal $\delta$-bracketing set and is denoted by $\cN_B(\delta,\cF,\|\cdot\|)$. The $\delta$-bracketing enropy, denoted by $\cH_B(\delta,\cF,\|\cdot\|)$, is defined as
\begin{align*}
  \cH_B(\delta,\cF,\|\cdot\|)=\log\cN_B(\delta,\cF,\|\cdot\|).
\end{align*}
\end{definition}
\begin{definition}[Covering number]
  Let $\cF$ be a set with metric $\rho$. A $\delta$-cover of $\cF$ is a set $\{f_1^*,...,f_N^*\}\subset \cF$ such that for any $f\in\cF$, there exists $f_k^*$ for some $k$ such that $\rho(f,f_k^*)\leq \delta$. The $\delta$-covering number of $\cF$ is defined as
  \begin{align*}
    \cN(\delta,\cF,\rho)=\inf \{N: \mbox{ there exists a } \delta-\mbox{cover } \{f_1^*,...,f_N^*\} \mbox{ of } \cF\}.
  \end{align*}
\end{definition}
It has been shown \citep[Lemma 2.1]{geer2000empirical} that for any $\delta>0, p\geq 1$,
\begin{align*}
  \cH_B(\delta,\cF,\|\cdot\|_{L^p})\leq \log\cN(\delta/2,\cF,\|\cdot\|_{L^\infty}).
\end{align*}

The proof of Theorem \ref{thm.classification} relies on the following proposition which is a modified version of \citet[Theorem 5]{kim2018fast}:
\begin{proposition}
\label{thm.sketch2.probability}
Let $\phi$ be a surrogate loss function for binary classification. Let $f_{\phi}^*, \cE_{\phi}(f_n,f_{\phi}^*)$ be defined as in (\ref{eq.class.min}) and (\ref{eq.class.excess}), respectively.  Assume the following regularity conditions:
  \begin{enumerate}[label={({A}\arabic*)}]
    \item $\phi$ is Lipschitz: $|\phi(z_1)-\phi(z_2)|\leq C_1|z_1-z_2|$ for any $z_1,z_2$ and some constant $C_1$.
    \item For a positive sequence $a_n=O(n^{-a_0})$ for some $a_0>0$, there exists a sequence of function classes $\{\cF_n\}_{n\in\NN}$ such that as $n\rightarrow \infty$,
        \begin{align*}
          \cE_{\phi}(f_n,f_{\phi}^*)\leq a_n
        \end{align*}
        for some $f_n\in \cF_n$.
    \item There exists a sequence $\{F_n\}_{n\in \NN}$ with $F_n\gtrsim 1$ such that $\sup_{f\in\cF_n} \|f\|_{L^\infty}\leq F_n$.
    \item There exists a constant $\nu\in(0,1]$ such that for any $f\in \cF_n$ and any $n\in\NN$,
        \begin{align*}
          \EE\left( \phi(yf(\xb))-\phi(yf_{\phi}^*(\xb))\right)^2\leq C_2F_n^{2-\nu}e^{F_n}\left( \cE_{\phi}(f,f_{\phi}^*)\right)^{\nu}
        \end{align*}
        for some constant $C_2>0$ only depending on $\phi$ and $\eta$.
    \item For a positive constant $C_3>0$, there exists a sequence $\{\delta_n\}_{n\in\NN}$ such that
        \begin{align*}
          \cH_B(\delta_n,\cF_n,\|\cdot\|_{L^2})\leq C_3e^{-F_n}n \left(\frac{\delta_n}{F_n}\right)^{2-\nu}
        \end{align*}
        for $\{\cF_n\}_{n\in\NN}$ in (A2), $\{F_n\}_{n\in\NN}$ in (A3) and $\nu$ in (A4).
  \end{enumerate}
  Let $\epsilon^2_n\asymp \max(a_n,\delta_n)$. Then the empirical $\phi$-risk minimizer $\hf_{\phi,n}$ over $\cF_n$ satisfies
  \begin{align}
    \PP\left( \cE_{\phi}(\hf_{\phi,n},f_{\phi}^*)\geq \epsilon_n \right) \leq C_5\exp\left(-C_4e^{-F_n}n\left(\epsilon_n^2/\left(F_n\right)\right)^{2-\nu}\right)
    \label{eq.sketch2.prob}
  \end{align}
  for some constants $C_4,C_5>0$.
\end{proposition}

Proposition \ref{thm.sketch2.probability} is proved in Appendix \ref{proof.probability}. In Proposition \ref{thm.sketch2.probability}, condition (A1) requires the surrogate loss function $\phi$ to be Lipschitz. This condition is satisfied in Theorem \ref{thm.classification} since $\phi$ is the logistic loss. (A2) is a condition on the bias of $\hf_{\phi,n}$. Take $n$ as the number of samples. (A2) requires the bias to decrease in the order of $O(n^{-a_0})$ for some $a_0$. (A3) requires all functions in the class $\cF_n$ to be bounded. (A4) and (A5) are conditions relate to the variance of $\hf_{\phi,n}$. Condition (A4) for logistic loss can be verified using the following lemma:
\begin{lemma}[Lemma 6.1 in \citet{park2009convergence}]
	\label{lem.sketch2.nu}
	Let $\phi$ be the logistic loss. Given a function class $\cF$ which is uniformly bounded by $F$, for any function $f\in\cF$, we have
	\begin{align*}
		\EE\left[\phi(yf)-\phi(yf_{\phi}^*)\right]^2\leq Ce^F\cE_{\phi}(f,f_{\phi}^*)
	\end{align*}
	for some constant $C$.
\end{lemma}
According to Lemma \ref{lem.sketch2.nu}, (A4) is verified with $\nu=1$. Now we are ready to prove Theorem \ref{thm.classification}.

\subsection{Proof of Theorem \ref{thm.classification}}
\begin{proof}[Proof of Theorem \ref{thm.classification}]
The main idea of the proof is to construct a sequence of network architectures, depending on $n$, such that Condition (A1)-(A5) in Proposition \ref{thm.sketch2.probability} are satisfied. The excess risk is then derived from (\ref{eq.sketch2.prob}). In particular, we choose 
\begin{align}
	\cF_n=\cC^{(n)},a_n=n^{-\frac{s}{2s+2(s\vee d)}}\log^2 n, F_n=\frac{s}{2s+2(s\vee d)}\log n,\delta_n=n^{-\frac{s}{2s+2(s\vee d)}}\log^4 n
	\label{eq.thm2.proof.parameter}
\end{align}
where $\cC^{(n)}$ is the network architecture in Theorem \ref{thm.classification}.

We first prove the probability bound of $\cE_{\phi}(\hf_{\phi,n},f_{\phi}^*)$ by checking conditions (A1)-(A5) in Proposition \ref{thm.sketch2.probability}. Note  that $\phi$ is the logistic loss which is Lipschitz continuous with Lipschitz constant 1. Thus (A1) is verified. According to Lemma \ref{lem.sketch2.nu}, (A4) is verified with $\nu=1$. We next verify (A2), (A3) and (A5).

\paragraph{A truncation technique.}
Recall that $f^*=\log \frac{\eta}{1-\eta}$. As $\eta$ goes to 0 (resp. 1), $f^*$ goes to $\infty$ (resp. $-\infty$). Note that (A3) requires the function class $\cF_n$ to be bounded by $F_n$. To study the approximation error of $\cF_n$ with respect to $f^*$, we consider a truncated version of $f^*$ defined as
  \begin{align}
    f_{\phi,n}^*=\begin{cases}
      F_n, & \mbox{ if } f_{\phi}^*>F_n,\\
      f_{\phi}^*, & \mbox{ if }-F_n\leq f_{\phi}^*\leq F_n,\\
      -F_n, & \mbox{ if } f_{\phi}^*<-F_n.
    \end{cases}
    \label{eq.sketch2.f-phi-n}
  \end{align}

\paragraph{Verification of (A2) and (A3).}
The following lemma is a very important lemma on approximating $f_{\phi,n}^*$ by ConvResNets. It also provides the covering number of the network class which will be used to verify (A5).
  \begin{lemma}\label{lem.sketch2.fPhiNcnn}
    Assume Assumption \ref{assum.M} and \ref{assum.reach}. Assume $0<p,q\leq \infty$, $0<s<\infty$, $s\geq d/p+1$. For any $\varepsilon\in (0,1)$ and any $K\leq D$, there exists a ConvResNet architecture
    \begin{align*}
    \cC^{(F_n)}=\big\{f|&f=\bar{g}_2\circ\bar{h}\circ \bar{g}_1\circ\bar{\eta} \mbox{ where } \bar{\eta}\in \cC^{\Conv}\left(M_1,L_1,J_1,K,\kappa_1\right),\ \bar{g}_1\in \cC^{\Conv}\left(1,4,8,1,\kappa_2\right),\\
    &  \bar{h}\in \cC^{\Conv}\left(M_2,L_2,J_2,1,\kappa_1\right),\ \bar{g}_2\in \cC\left(1,3,8,1,\kappa_3,1,R\right)\big\}
  \end{align*}
with 
\begin{align*}
	&M_1=O\left(\varepsilon^{-d/s}\right),\ M_2=O\left(e^{-F_n}\varepsilon^{-1}\right),
	L_1=O(\log (1/\varepsilon)+D+\log D),\ L_2=O(\log (1/\varepsilon)),\\
	&J_1=O(D),\ J_2=O(1), \ \kappa_1=O(1),\ \log \kappa_2=O(\log^2 (1/\varepsilon)),\  \kappa_3=O(\log (F_n/\varepsilon)+F_n),\ R=F_n,
\end{align*}
  such that for any $\eta\in \Bnorm(\cM)$ with $\|\eta\|_{\Bnorm(\cM)}\leq c_0$ for some constant $c_0$, and $f_{\phi,n}^*$ be defined as in (\ref{eq.sketch2.f-phi-n}), there exists $\bar{f}_{\phi,n}\in \cC^{(F_n)}$ with
    \begin{align*}
      \|\bar{f}_{\phi,n}-f_{\phi,n}^*\|_{L^\infty}\leq 4e^{F_n}\varepsilon.
    \end{align*}
  Moreover, the covering number of $\cC^{(F_n)}$ is bounded by
\begin{align*}
  \cN(\delta,\cC^{(F_n)},\|\cdot \|_{L^\infty})=O\left(D^3\varepsilon^{-\left(\frac{d}{s} \vee 1\right))}\log (1/\varepsilon)\left(\log^2 (1/\varepsilon)+\log D+ F_n+\log(1/\delta)\right)\right).
\end{align*}
The constant hidden in $O(\cdot)$ depends on $d,s,\frac{2d}{sp-d},p,q,c_0,\tau$ and the surface area of $\cM$. 
  \end{lemma}
Lemma \ref{lem.sketch2.fPhiNcnn} is proved in Section \ref{proof.lem.fPhiNcnn}.
By Lemma \ref{lem.sketch2.fPhiNcnn}, fix the network architecture $\cC^{(F_n)}$, for $\varepsilon_1\in (0,1)$, there exists a ConvResNet $\bar{f}_{\phi,n}\in \cC^{(F_n)}$
   such that $\|\bar{f}_{\phi,n}-f_{\phi,n}^*\|_{L^\infty}\leq 4e^{F_n}\varepsilon_1.$ In the following, we choose $\varepsilon_1=n^{-\frac{2s}{2s+2(s\vee d)}}\log n$.

Next we check conditions (A2) and (A3) by estimating $\cE_{\phi}(\bar{f}_{\phi,n},f_{\phi}^*)$. Denote
\begin{align*}
  A_n=\{\xb\in\cM: |f_{\phi}^*|\leq F_n\},\ A_n^{\complement}=\{\xb\in\cM: |f_{\phi}^*|> F_n\}.
\end{align*}
We have
\begin{align}
  \cE_{\phi}(\bar{f}_{\phi,n},f_{\phi}^*)&=\int_{\cM} \eta\left(\phi(\bar{f}_{\phi,n})-\phi(f_{\phi}^*)\right) +(1-\eta)\left( \phi(-\bar{f}_{\phi,n})-\phi(-f_{\phi}^*)\right)\mu(d\xb) \nonumber\\
  &=\underbrace{\int_{A_n} \eta\left(\phi(\bar{f}_{\phi,n})-\phi(f_{\phi,n}^*)\right) +(1-\eta)\left( \phi(-\bar{f}_{\phi,n})-\phi(-f_{\phi,n}^*)\right)\mu(d\xb)}_{\rm T_1} \nonumber\\
  &\quad + \underbrace{\int_{A_n^{\complement}} \eta\left(\phi(\bar{f}_{\phi,n})-\phi(f_{\phi}^*)\right) +(1-\eta)\left( \phi(-\bar{f}_{\phi,n})-\phi(-f_{\phi}^*)\right)\mu(d\xb)}_{\rm T_2},
  \label{eq.sketch2.decom}
\end{align}
where we used $f_{\phi}^*=f_{\phi,n}^*$ on $A_n$. In (\ref{eq.sketch2.decom}), ${\rm T_1}$ represents the approximation error of $\bar{f}_{\phi,n}$, and ${\rm T_2}$ is the truncation error. 
Since $\|\bar{f}_{\phi,n}-f_{\phi,n}^*\|_{L^\infty}\leq 4e^{F_n}\varepsilon_1$,
\begin{align}
  {\rm T_1}&\leq \int_{A_n} \eta|\phi(\bar{f}_{\phi,n})-\phi(f_{\phi,n}^*)| +(1-\eta)| \phi(-\bar{f}_{\phi,n})-\phi(-f_{\phi,n}^*)|\mu(d\xb)\nonumber\\
  &\leq \|\phi(\bar{f}_{\phi,n})-\phi(f_{\phi,n}^*)\|_{L^\infty}\leq 4e^{F_n}\varepsilon_1.
  \label{eq.sketch2.I}
\end{align}

A bound of ${\rm T_2}$ is provided by the following lemma (see a proof in Appendix \ref{proof.sketch2.II}):
\begin{lemma}\label{lem.sketch2.II}
   Assume Assumption \ref{assum.M} and \ref{assum.reach}. Assume $0<p,q\leq \infty$, $0<s<\infty$, $s\geq d/p+1$, $\eta\in \Bnorm(\cM)$ with $\|\eta\|_{\Bnorm(\cM)}\leq c_0$ for some constant $c_0$. Let ${\rm T_2}$ be defined as in (\ref{eq.sketch2.decom}). If $4e^{F_n}\varepsilon_1<1$, the following bound holds:
  \begin{align}
    {\rm T_2}\leq 8F_ne^{-F_n}.
    \label{eq.sketch2.II}
  \end{align}
\end{lemma}
According to our choices of $\varepsilon_1$ and $F_n$, $4e^{F_n}\varepsilon_1<1$ is satisfied.
Combining (\ref{eq.sketch2.I}) and (\ref{eq.sketch2.II}) gives
\begin{align*}
  \cE_{\phi}(\bar{f}_{\phi,n},f_{\phi}^*)\leq{\rm T_1}+{\rm T_2}\leq 4e^{F_n}\varepsilon_1+8F_ne^{-F_n}.
\end{align*}
Substituting $\varepsilon_1=n^{-\frac{2s}{2s+2(s\vee d)}}\log n, F_n=\frac{s}{2s+2(s\vee d)}\log n$ gives
\begin{align*}
  \cE_{\phi}(\bar{f}_{\phi,n},f_{\phi}^*)\leq C_6 n^{-\frac{s}{2s+2(s\vee d)}}\log^2 n
\end{align*}
and $\cC^{(F_n)}=\cC^{(n)}$, where $\cC^{(n)}$ is defined in Theorem \ref{thm.classification}. Here $C_6$ is a constant depending on $s$ and $d$.
Thus (A2) and (A3) are satisfied with $a_n=n^{-\frac{s}{2s+2(s\vee d)}}\log^2 n, F_n=\frac{s}{2s+2(s\vee d)}\log n$.

\paragraph{Verification of (A5).}

For (A5), we only need to check that $\log \cN(\delta_n,\cC^{(n)},\|\cdot\|_{L^\infty})\leq C_3 nF_n^{-1}\delta_n$ for some constant $C_3$ . According to Lemma \ref{lem.sketch2.fPhiNcnn} with our choices of $\varepsilon_1$ and $F_n$, we have
\begin{align*}
  \log \cN(\delta,\cC^{(n)},\|\cdot \|_{L^\infty})=O\left(D^3n^{\frac{2(s\vee d)}{2s+2(s\vee d)} }\log (n)\left(\log^2 n+\log n+\log D+\log(1/\delta)\right)\right).
\end{align*}
Substituting our choice $\delta_n=n^{-\frac{s}{2s+2(s\vee d)}}\log^4 n$ gives rise to
\begin{align}
  \log\cN(\delta_n,\cC^{(n)},\|\cdot\|_{L^\infty})=O\left((D^3\log D) n^{\frac{2(s\vee d)}{2s+2(s\vee d)}}\log^2 n\right)\leq C_3 nF_n^{-1}e^{-F_n}\delta_n
\end{align}
for some $C_3$ depending on $d,D^3\log D,s,\frac{d}{sp-d},p,q,c_0,\tau$ and the surface area of $\cM$.
Therefore (A5) is satisfied. 
\paragraph{Estimate the excess risk.}
Since (A1)-(A5) are satisfied, Proposition \ref{thm.sketch2.probability} gives
\begin{align}
    \PP\left( \cE_{\phi}(\hf_{\phi,n},f_{\phi}^*)\geq \epsilon_n \right) \leq C_5 \exp\left(-C_4\frac{2s+2(s\vee d)}{s}\frac{n^{\frac{s+2(s\vee d)}{2s+2(s\vee d)}}\epsilon^2_n}{\log n}\right)
    \label{eq.sketch2.P}
\end{align}
with $\epsilon^2_n \asymp \max(a_n,\delta_n)=C_7n^{-\frac{s}{2s+2(s\vee d)}}\log^4 n$ and $\hf_{\phi,n}\in \cF_n$ being the minimizer of the empirical risk in (\ref{eq.class.risk}). Here $C_7$ is a constant depending on  $d,D,\log D,s,\frac{d}{sp-d},p,q,c_0,\tau$ and the surface area of $\cM$.

 Note that (A5) is also satisfied for any $\delta_n\geq C_7n^{-\frac{s}{2s+2(s\vee d)}}\log^4 n$. Thus
\begin{align}
    \PP\left( \cE_{\phi}(\hf_{\phi,n},f_{\phi}^*)\geq t \right) \leq C_5 \exp\left(-C_4\frac{2s+2(s\vee d)}{s}\frac{n^{\frac{s+2(s\vee d)}{2s+2(s\vee d)}}t}{\log n}\right)
    \label{eq.sketch2.pp}
\end{align}
for any $t\geq C_7n^{-\frac{s}{2s+2(s\vee d)}}\log^4 n$. Integrating (\ref{eq.sketch2.pp}), we estimate the expected excess risk as
\begin{align}
  \EE(\cE_{\phi}(\hf_{\phi,n},f_{\phi}^*))&=\int_{\cM} \cE_{\phi}(\hf_{\phi,n},f_{\phi}^*)\mu(d\xb) \nonumber\\
  &\leq C_7\PP\left( \cE_{\phi}(\hf_{\phi,n},f_{\phi}^*)\leq C_7n^{-\frac{s}{2s+2(s\vee d)}}\log^4 n \right)n^{-\frac{s}{2s+2(s\vee d)}}\log^4 n \nonumber\\
  &\quad + C_5\int_{C_7n^{-\frac{s}{2s+2(s\vee d)}}\log^4 n}^{\infty} \exp\left(-C_4\frac{2s+2(s\vee d)}{s}\frac{n^{\frac{s+2(s\vee d)}{2s+2(s\vee d)}}t}{\log n}\right) dt \nonumber\\
  &\leq C_8 n^{-\frac{s}{2s+2(s\vee d)}}\log^4 n
\end{align}
for some constants $C_7,C_8$ depending on $d,D,\log D,s,\frac{2d}{sp-d},p,q,c_0,\tau$ and the surface area of $\cM$. 
\end{proof} 

\section{Convolutional neural networks and muli-layer perceptrons}\label{sec.CNNMLP}
The proofs of the main results utilize properties convolutional neural networks (CNN) and multi-layer perceptrons (MLP) with the ReLU activation. We consider CNNs in the form of
\begin{align}\label{eq:convfCNN}
	f(\xb)=W\cdot\Conv_{\cW,\cB}(\xb)
\end{align}
where $\Conv_{\cW,\cB}(Z)$ is defined in (\ref{eq.conv}), $W$ is the weight matrix of the fully connected layer, $\cW,\cB$ are sets of filters and biases, respectively.
We define the class of CNNs as
\begin{equation}
	\begin{aligned}
		\cF^{\rm CNN}(L,J,K,\kappa_1,\kappa_2) = \big\{f ~|& f(\xb) \textrm{ in the form \eqref{eq:convfCNN} with $L$ layers.}\\
		&\mbox{Each convolutional layer has filter size bounded by $K$.} \\
		&  \mbox{The number of channels of each layer is bounded  by $J$}.\\
		& \max_{l}\|\cW^{(l)}\|_{\infty} \vee \|B^{(l)}\|_{\infty} \leq \kappa_1,\  \|W\|_{\infty}  \leq \kappa_2\big\}.\label{eqcFCNN}
	\end{aligned}
\end{equation}

For MLP, we consider the following form
\begin{align}\label{eq:reluf}
	f(\xb) = W_L \cdot \textrm{ReLU}(W_{L-1} \cdots \textrm{ReLU}(W_1 \xb + \bbb_1) \cdots + \bbb_{L-1}) + \bbb_L,
\end{align}
where $W_1, \dots, W_L$ and $\bbb_1, \dots, \bbb_L$ are weight matrices and bias vectors of proper sizes, respectively. 
The class of MLP is defined as
\begin{equation}
\begin{aligned}
\cF^{\rm MLP}(L,J,\kappa) = \big\{f ~|& f(\xb) \textrm{ in the form \eqref{eq:reluf} with $L$-layers and width bounded by $J$}. \\
&  \norm{W_i}_{\infty, \infty} \leq \kappa, \norm{\bbb_i}_\infty \leq \kappa ~\textrm{for}~ i = 1, \dots, L \big\}.\label{eqcF}
\end{aligned}
\end{equation}
In some cases it is necessary to enforce the output of the MLP to be bounded. We define such a class as
\begin{equation*}
	\begin{aligned}
		\cF^{\rm MLP}(L,J,\kappa,R) = \left\{f ~| f(\xb)\in \cF^{\rm MLP}(L,J,\kappa) \mbox{ and } \|f\|_{\infty}\leq R\right\}.
	\end{aligned}
\end{equation*}
In some case we do not need the constraint on the output, we denote such MLP class as $\cF^{\rm MLP}(L,J,\kappa)$.



\section{Lemmas and proofs in Section \ref{thm.approximation.proof}}
\label{sec.thm1.lemma.proof}
\subsection{Lemma \ref{lem.besov.extend} and its proof}
\label{lem.besov.extend.proof}
\begin{lemma}\label{lem.besov.extend}
	Define $f_i,\phi_i$ as in (\ref{eq.f.decompose}). We extend $f_i\circ\phi_i^{-1}$ by 0 on $[0,1]^d\backslash \phi_i(U_i)$ and denote the extended function by $f_i\circ\phi_i^{-1}|_{[0,1]^d}$ . Under Assumption \ref{assum.f}, we have $f_i\circ\phi_i^{-1}|_{[0,1]^d}\in \Bnorm([0,1]^d)$ with $$\|f_i\circ\phi_i^{-1}\|_{\Bnorm([0,1]^d)}<Cc_0$$ 
	where $C$ is a constant depending on $s,p,q$ and $d$. 
\end{lemma}
To prove Lemma \ref{lem.besov.extend}, we first give an equivalent definition of Besov functions:
\begin{definition}\label{def.BesovNew}
	Let $\Omega$ be a Lipschitz domain in $\RR^d$. For $0<p,q\leq \infty$ and $s>0$, $\cBnorm(\Omega)$ is the set of functions
	\begin{align*}
		\cBnorm(\Omega)=\{f: \Omega\rightarrow \RR| \exists g\in \Bnorm(\RR^d) \mbox{ with } g|_{\Omega}=f\},
	\end{align*}
where $g|_{\Omega}$ denotes the restriction of $g$ on $\Omega$.
	The norm is defined as $\|f\|_{\cBnorm(\Omega)}=\inf_g \|g\|_{\cBnorm(\RR^d)}$. 
\end{definition}
According to \citet[Theorem 3.18]{dispa2003intrinsic}, for any Lipschitz domain $\Omega\subset \RR^d$, the norm $\|\cdot\|_{\Bnorm(\Omega)}$ in Definition \ref{def.besovNorm} is equivalent to $\|\cdot\|_{\cBnorm(\Omega)}$ in Definition \ref{def.BesovNew}. Thus $\Bnorm(\Omega)=\cBnorm(\Omega)$.
\begin{proof}[Proof of Lemma \ref{lem.besov.extend}]
  Since $f\in \Bnorm(\cM)$, according to Definition \ref{def.besovM}, $f_i\circ\phi_i^{-1} \in \Bnorm(\RR^d)$ in the sense of extending $f_i\circ\phi_i^{-1}$ by zero on $\RR^d\backslash \phi_i(U_i)$, see \citet[Section 3.2.3]{triebel1983theory} for details. From Assumption \ref{assum.f}, $\|f_i\circ\phi_i^{-1}\|_{\Bnorm(\RR^d)}<c_0$.
  We next restrict $f_i\circ\phi_i^{-1}$ on $[0,1]^d$ and denote the restriction by $f_i\circ\phi_i^{-1} | _{[0,1]^d}$. 
  
  Using Definition \ref{def.BesovNew} and Assumption , we have 
  $$
  \|f_i\circ\phi_i^{-1} | _{[0,1]^d}\|_{\cBnorm([0,1]^d))}\leq\|f_i\circ\phi_i^{-1}\|_{\Bnorm(\RR^d)}<c_0, 
  $$ 
  and we next show $f_i\circ\phi_i^{-1} | _{[0,1]^d}\in U(\Bnorm([0,1]^d))$.
  Since $[0,1]^d$ is a Lipschitz domain, \citet[Theorem 3.18]{dispa2003intrinsic} implies $\Bnorm([0,1]^d)=\cBnorm([0,1]^d)$. Therefore,  there exists a constant $C$ depending on $s,p,q$ and $d$ such that
   $$
   \|f_i\circ\phi_i^{-1} | _{[0,1]^d}\|_{\Bnorm([0,1]^d)}\leq \|f_i\circ\phi_i^{-1}\|_{\cBnorm(\RR^d)}\leq Cc_0.
   $$
   \end{proof}

\subsection{Cardinal B-splines}\label{sec.B-spline}
We give a brief introduction of cardinal B-splines.
\begin{definition}[Cardinal B-spline]
	Let $\psi(x)=\mone_{[0,1]}(x)$ be the indicator function of $[0,1]$. The cardinal B-spline of order m is defined by taking $m+1$-times convolution of $\psi$:
	\begin{align*}
		\psi_m(x)=(\underbrace{\psi\ast\psi\ast\cdots \ast\psi}_{m+1\mbox{ times}})(x)
	\end{align*}
	where $f\ast g(x)\equiv\int f(x-t)g(t)dt$.
\end{definition}
Note that $\psi_m$ is a piecewise polynomial with degree $m$ and support $[0,m+1]$. It can be expressed as \citep{mhaskar1992approximation}
\begin{align*}
	\psi_m(x)=\frac{1}{m!}\sum_{j=0}^{m+1}(-1)^j \binom{m+1}{j}(x-j)_+^m.
\end{align*}
For any $k,j\in \NN$, let $M_{k,j,m}(x)=\psi_m(2^kx-j)$, which is the rescaled and shifted cardinal B-spline with resolution $2^{-k}$ and support $2^{-k}[j,j+(m+1)]$. For $\kb=(k_1,\dots,k_d)\in \NN^d$ and $\jb=(j_1,\dots,j_d)\in \NN^d$, we define the $d$ dimensional cardinal B-spline as $M_{\kb,\jb,m}^d(\xb)=\prod_{i=1}^d \psi_m(2^{k_i}x_i -j_i)$. When $k_1=\ldots=k_d=k \in \NN$, we denote $M_{k,\jb,m}^d(\xb)=\prod_{i=1}^d \psi_m(2^{k}x_i -j_i)$.

\subsection{Lemma \ref{lem.BSplAppBesov}}\label{sec.lem.BSplAppBesov.proof}
For any $m\in \NN$, let $J(k)=\{-m,-m+1,\dots,2^k-1,2^k\}^d$ and the quasi-norm of the coefficient $\{\alpha_{k,j}\}$ for $k\in \NN, \jb\in J(k)$ be

\begin{align}
	\|\{\alpha_{k,\jb}\}\|_{\bnorm}=\left( \sum_{k\in\NN} \left[ 2^{k(s-d/p)}\left( \sum_{\jb\in J(k)} |\alpha_{k,\jb}|^p\right)^{1/p} \right]^q \right)^{1/q}.
	\label{eq.bnorm}
\end{align}

The following lemma, resulted from \citet{devore1988interpolation,dung2011optimal}, gives an error bound for the approximation of functions in $\Bnorm([0,1]^d)$ by cardinal B-splines.

\begin{lemma}[Lemma 2 in \citet{suzuki2018adaptivity}; \citet{devore1988interpolation,dung2011optimal}]\label{lem.BSplAppBesov}
	Assume that $0 < p,q,r \le\infty$ and $0<s<\infty$ satisfying $s>d(1/p-1/r)_+$. Let $m \in \NN$ be the order of the Cardinal B-spline basis such that $0<s<\min(m,m-1+1/p)$. For any $f\in \Bnorm([0,1]^d)$, there exists $f_N$ satisfying
	\begin{align*}
		\|f-f_N\|_{L^r([0,1]^d)}\leq C N^{-s/d}\|f\|_{\Bnorm([0,1]^d)}
	\end{align*}
	for some constant $C$ with $N\gg 1$. $f$ is in the form of
	\begin{align}\label{eq.fN}
		f_N(\xb)=\sum_{k=0}^H \sum_{\jb\in J(k)} \alpha_{k,\jb}M_{k,\jb,m}^d(\xb)+ \sum_{k=K+1}^{H^*}\sum_{i=1}^{n_k} \alpha_{k,\jb_i} M_{k,\jb_i,m}^d(\xb),
	\end{align}
	where $\{\jb_i\}_{i=1}^{n_k} \subset J(k), H=\lceil c_1\log(N)/d \rceil, H^*=\lceil \nu^{-1}\log(\lambda N) \rceil +H+1, n_k=\lceil \lambda N2^{-\nu (k-H)} \rceil$ for $k=H+1,\dots,H^*, u =d(1/p-1/r)_+$ and $\nu=(s-u)/(2u)$. The real numbers $c_1>0$ and $\lambda>0$ are two absolute constants chosen to satisfy $\sum_{k=1}^H (2^k+m)^d+\sum_{k=H+1}^{H^*} n_k\leq N$, which are to $N$. Moreover, we can choose the coefficients $\{\alpha_{k,\jb}\}$ such that
	\begin{align*}
		\|\{\alpha_{k,\jb}\}\|_{\bnorm} \leq C_1 \|f\|_{\Bnorm([0,1]^d)}
	\end{align*}
	for some constant $C_1$.
\end{lemma}

\begin{lemma}\label{lem.alphabound}
	Let $\alpha^{(i)}_{k,\jb}$ be defined as in (\ref{eq.f-BSpline}). Under Assumption \ref{assum.f}, for any $i,k,\jb$, we have 
	\begin{align}
		|\alpha_{k,\jb}|\leq Cc_0 N^{(\log 2)(\nu^{-1}+c_1d^{-1})(d/p-s)_+}
	\end{align}
	for some $C$ depending on $(d/p-s)_+\nu^{-1},s$ and $d$, where $\nu,c_1$ are defined in Lemma \ref{lem.BSplAppBesov}. 
\end{lemma}
\begin{proof}[Proof of Lemma \ref{lem.alphabound}]
	According to (\ref{eq.bnorm}) and Lemma \ref{lem.BSplAppBesov},
	\begin{align*}
		2^{k(s-d/p)}|\alpha_{k,\jb}|\leq \|\{\alpha_{k,\jb}\}\|_{\bnorm} \leq C_1 \|f_i\circ\phi_i^{-1}\|_{\Bnorm([0,1]^d)}.
	\end{align*}
Using Lemma \ref{lem.besov.extend} and since $k\leq H^*$ (from Lemma \ref{lem.BSplAppBesov}), we have
\begin{align}
	|\alpha^{(i)}_{k,\jb}|&\leq C_1 2^{k(d/p-s)_+}\|f_i\circ\phi_i^{-1}\|_{\Bnorm([0,1]^d)}\leq C_2 2^{H^*(d/p-s)_+}c_1c_0 
	\label{eq.alphatempbound}
\end{align}
for some $C_2$ depending on $s$ and $d$. From the expression of $H^*$, we can compute
\begin{align}
	2^{H^*}\leq C_3N^{(\log 2)(\nu^{-1}+c_1d^{-1})}
	\label{eq.2bound}
\end{align}
for some $C_3$ depending on $(d/p-s)_+\nu^{-1}$.
Substituting (\ref{eq.2bound}) into (\ref{eq.alphatempbound}) finishes the proof.

\end{proof}
\subsection{Lemma \ref{lem.multiplication}}\label{sec.lem.multiplication}
\begin{lemma}[Proposition 3 in \cite{yarotsky2017error}]
	For any $C>0$ and $0<\eta<1$. If $|x|\leq C,|y|\leq C$, there is an MLP, denoted by $\ttimes(\cdot,\cdot)$, such that
	$$
	|\ttimes(x,y)-xy|<\eta,\ \ttimes(x,0)=\ttimes(y,0)=0.
	$$
	Such a network has $O\left(\log \frac{1}{\eta}\right)$ layers and parameters. The width of each layer is bounded by 6 and all parameters are bounded by $C^2$.
	\label{lem.multiplication}
\end{lemma}

\subsection{Lemma \ref{lem.cnnRealization}}\label{sec.lem.cnnRealization.proof}
The following lemma is a special case of \citet[Theorem 1]{oono2019approximation}. It shows that each MLP can be realized by a CNN:
\begin{lemma}[Theorem 1 in \cite{oono2019approximation}]
	\label{lem.cnnRealization}
	Let $D$ be the dimension of the input. Let $L,J$ be positive integers and $\kappa>0$. For any $2\leq K'\leq D$, any MLP architectures $\cF^{\rm MLP}(L,J,\kappa)$ can be realized by a CNN architecture $\cF^{\rm CNN}(L',J',K',\kappa_1',\kappa_2')$ with
	$$
	L'=L+D, J'=4J, \kappa'_1=\kappa'_2=\kappa.
	$$
	Specifically, any $\bar{f}^{\rm MLP}\in \cF^{\rm MLP}(L,J,\kappa)$ can be realized by a CNN $\bar{f}^{\rm CNN}\in \cF^{\rm CNN}(L',J',K',\kappa_1',\kappa_2')$.
	Furthermore, the weight matrix in the fully connected layer of $\bar{f}^{\rm CNN}$ has nonzero entries only in the first row.
\end{lemma}

\subsection{Lemma \ref{lem.indicator} and its proof}\label{sec.lem.indicator.proof}

\begin{lemma}\label{lem.indicator}
	Let $d_i^2$ and $\mone_{[0,\omega^2]}$ be defined as in (\ref{eq.indicator.1}). For any $\theta\in (0,1)$ and $\Delta\geq 8B^2D\theta$, there exists a CNN $\tdi$ approximating $d_i^2$ such that 
	$$\|\tdi-d_i^2\|_{L^\infty}\leq 4B^2D\theta,$$ 
	and a CNN $\mtoned$ approximating $\mone_{[0,\omega^2]}$ with
	\begin{align*}
		&\mtoned(\xb)=\begin{cases}
			1,&\mbox{ if } a\leq(1-2^{-k})(\omega^2-4B^2D\theta),\\
			0,&\mbox{ if } a\geq \omega^2-4B^2D\theta,\\
			2^k((\omega^2-4B^2D\theta)^{-1}a-1),&\mbox{ otherwise}.
		\end{cases}
	\end{align*}
for $\xb\in\cM$. The CNN for $\tdi$ has $O(\log(1/\theta))$ layers, $6D$ channels and all weights parameters are bounded by $4B^2$. The CNN for $\mtoned$ has $\left\lceil\log(\omega^2/\Delta)\right\rceil$ layers, $2$ channels. All weight parameters are bounded by $\max(2,|\omega^2-4B^2D\theta|)$.

As a result, for any $\xb\in\cM$, $\mtoned\circ\tdi(\xb)$ gives an approximation of $\mone_{U_i}$ satisfying
\begin{align*}
	&\mtoned\circ\tdi(\xb) =
	\begin{cases}
	1, &\mbox{ if } \xb\in U_i \mbox{ and } d_i^2(\xb)\leq \omega^2-\Delta;\\	
	0,&\mbox{ if } \xb\notin U_i; \\
	\mbox{ between 0 and 1}, &\mbox{ otherwise}.
\end{cases}
\end{align*}
\end{lemma}
\begin{proof}
We first show the existence of $\tdi $.
Here $d_i^2(\xb)$ is the sum of $D$ univariate quadratic functions. Each quadratic function can be approximated by an multi-layer perceptron (MLP, see Appendix \ref{sec.CNNMLP} for the definition) according to  Lemma \ref{lem.multiplication}.
Let $\mathring{h}(x)$ be an MLP approximation of $x^2$ for $x\in [0,1]$ with error $\theta$, i.e., $\|\mathring{h}(x)-x^2\|_{\infty}\leq \theta$. We define
\begin{align*}
	\mathring{d}_i^2(\xb)=4B^2\sum_{j=1}^D \mathring{h}\left(\left| \frac{x_j-c_{i,j}}{2B}\right|\right)
\end{align*}
as an approximation of $d_i^2(\xb)$, which gives rise to the approximation error $\|\mathring{d}_i^2-d_i^2\|_{\infty}\leq 4B^2D\theta$. Such a MLP has $O(\log 1/\theta)$ layers, and width $6D$. All weight parameters are bounded by $4B^2$. According to Lemma \ref{lem.cnnRealization}, $\mathring{d}_i^2$ can be realized by a CNN, which is denoted by $\tdi$. Such a CNN has $O(\log 1/\theta)$ layers, $6D$ channels. All weight parameters are bounded by $4B^2$.

To show the existence of $\mtoned$, we use the following function to approximate $\mone_{[0,\omega^2]}$:
\begin{align*}
	\mtoned(a)=\begin{cases}
		1, & \mbox{ if } a\leq \omega^2-\Delta+4B^2D\theta,\\
		0,& \mbox{ if } a \geq \omega^2-4B^2Dv,\\
		-\frac{1}{\Delta-8B^2D\theta}a+\frac{r^2-4B^2D\theta}{\Delta-8B^2D\theta}, & \mbox{ otherwise. }
	\end{cases}	
\end{align*}
 We implement $\mtoned(a)$ based on the basic step function defined as: $g(a)=2\ReLU(a-0.5(\omega^2-4B^2D\theta))-2\ReLU(a-\omega^2+4B^2D\theta)$. Define
\begin{align*}
	g_k(a)&=\underbrace{ g\circ\cdots\circ g}_k(a)\\
	&=\begin{cases}
		0,& \mbox{ if } a\leq(1-2^{-k})(\omega^2-4B^2D\theta),\\
		\omega^2-4B^2D\theta,&  \mbox{ if } a\geq \omega^2-4B^2D\theta,\\
		2^k(a-\omega^2+4B^2D\theta)+\omega^2-4B^2D\theta,&  \mbox{ otherwise}.
	\end{cases}
\end{align*}
We set $\mtoned=1-(\omega^2-4B^2D\theta)^{-1}g_k$ which can be realized by a CNN (according to Lemma \ref{lem.cnnRealization}).
Such a CNN has $k$ layers, $2$ channels. All weight parameters are bounded by $\max(2,|\omega^2-4B^2D\theta|)$. The number of compositions $k$ is chosen to satisfy $(1-2^{-k})(\omega^2-4B^2D\theta)\geq \omega^2-\Delta +4B^2D\theta$ which gives $k=\lceil \log(\omega^2/\Delta)\rceil$.


\end{proof}
\subsection{Lemma \ref{lem.CNNAppBSpl} and its proof}
Lemma \ref{lem.CNNAppBSpl} shows that each cardinal B-spline can be approximated by a CNN with arbitrary accuracy. This lemma is used to prove Proposition \ref{prop.ReLUBR}.
\begin{lemma}\label{lem.CNNAppBSpl}
	Let $k$ be any number in $\NN$ and $\jb$ be any element in $\NN^d$.
	There exists a constant $C$ depending only on $d$ and $m$ such that, for and $\varepsilon\in(0,1)$ and $2\leq K\leq d$, there exists a CNN $\tM_{k,\jb,m}^d\in \cF^{\rm CNN}(L,J,K,\kappa,\kappa)$ with $L=3+2\lceil \log_2\left( \frac{3\vee m}{C\varepsilon }\right)+5 \rceil \lceil \log_2(d\vee m) \rceil+d, J=24dm(m+2)+8d$ and $\kappa=2(m+1)^m\vee 2^k$ such that for any $k \in \NN$ and $\jb \in \NN^d$,
	\begin{align*}
		\|M_{k,\jb,m}^d-\tM_{k,\jb,m}^d\|_{L^{\infty}([0,1]^d)}\leq \varepsilon,
	\end{align*}
	and $\tM_{k,\jb,m}^d(\xb)=0$ for all $\xb\notin 2^{-k}[0,m+1]^d$.
\end{lemma}
The proof of Lemma \ref{lem.CNNAppBSpl} is based on the following lemma:
\begin{lemma}[Lemma 1 in \citet{suzuki2018adaptivity}]\label{lem.ReLUAppBSpl}
	Let $k$ be any number in $\NN$ and $\jb$ be any element in $\NN^d$.
	There exists a constant $C$ depending only on $d$ and $m$ such that, for all $\varepsilon>0$, there exists an MLP $\bar{M}_{k,\jb,m}^d\in \cF^{\rm MLP}(L,J,\kappa,1)$ with $L=3+2\lceil \log_2\left( \frac{3\vee m}{C\varepsilon }\right)+5 \rceil \lceil \log_2(d\vee m) \rceil, J=6dm(m+2)+2d$ and $\kappa=2(m+1)^m\vee 2^k$ such that for any $k \in \NN$ and $\jb \in \NN^d$,
	\begin{align*}
		\|M_{k,\jb,m}^d-\bar{M}_{k,\jb,m}^d\|_{L^{\infty}([0,1]^d)}\leq \varepsilon,
	\end{align*}
	and $\bar{M}_{k,\jb,m}^d(\xb)=0$ for all $\xb\notin 2^{-k}[0,m+1]^d$.
\end{lemma}

\begin{proof}[Proof of Lemma \ref{lem.CNNAppBSpl}]
	According to Lemma \ref{lem.ReLUAppBSpl}, there exists an MLP $\bar{M}_{k,\jb,m}^d\in \cF^{\rm MLP}(L',J',\kappa',1)$ with $L'=3+2\lceil \log_2\left( \frac{3\vee m}{C\varepsilon }\right)+5 \rceil \lceil \log_2(d\vee m) \rceil, J'=6dm(m+2)+2d$ and $\kappa'=2(m+1)^m\vee 2^k$ such that 
	\begin{align*}
		\|M_{k,\jb,m}^d-\bar{M}_{k,\jb,m}^d\|_{L^{\infty}([0,1]^d)}\leq \varepsilon,
	\end{align*}
and $\bar{M}_{k,\jb,m}^d(\xb)=0$ for all $\xb\notin 2^{-k}[0,m+1]^d$.

Lemma \ref{lem.cnnRealization} shows that such an MLP can be realized by a CNN $\tM_{k,\jb,m}^d\in \cF^{\rm CNN}(L,J,K,\kappa,\kappa)$.
\end{proof}

\subsection{Proposition \ref{prop.ReLUBR} and its proof}\label{proof.prop.ReLUBR}
Proposition \ref{prop.ReLUBR} shows that if $N$ and $\varepsilon_1$ are properly chosen, $\sum_{j=1}^N \tf^{\rm CNN}_{i,j}$ can approximate $f_i\circ \phi_i^{-1}$ with arbitrary accuracy.
\begin{proposition}\label{prop.ReLUBR}
	Let $f_i\circ\phi_i^{-1}$ be defined as in (\ref{eq.f.decompose1}). For any $\delta\in (0,1)$, set $N=C_1\delta^{-d/s}$. Suppose Assumption \ref{assum.f}. For any $2\leq K\leq d$, there exists a set of CNNs $\left\{\tf_{i,j}^{\rm CNN}\right\}_{j=1}^N$ such that 
	\begin{align*}
		\left\|\sum_{j=1}^N \tf^{\rm CNN}_{i,j}-f_i\circ \phi_i^{-1}\right\|_{L^\infty}\leq \delta,
	\end{align*}
	where $C_1$ is a constant depending on $s,p,q$ and $d$.
	
	$\tf_{i,j}^{\rm CNN}$ is a CNN approximation of $\tf_{i,j}$ (defined in (\ref{eq.f-BSpline})) and is in $\cF^{\rm CNN}(L,J,K,\kappa,\kappa)$ with
	\begin{align*}
		\begin{aligned} 
			&L= O\left( \log(1/\delta)\right), J=\lceil 24d(s+1)(s+3)+8d\rceil, 
			\kappa=O\left( \delta^{-(\log 2)(\frac{2d}{sp-d}+c_1d^{-1})}\right).
		\end{aligned}
	\end{align*}
	The constant hidden in $O(\cdot)$ depends on $d,s,\frac{2d}{sp-d},p,q,c_0$.
\end{proposition}

\begin{proof}[Proof of Proposition \ref{prop.ReLUBR}]
	Based on the approximation (\ref{eq.f-BSpline}), for each $\tf_{i,j}$, we construct CNN $\tf_{i,j}^{\rm CNN}$ to approximate it. 
	
	Note that $\tf_{i,j}=\alpha^{(i)}_{k,\jb}M_{k,\jb,m}^d $ with some coefficient $\alpha^{(i)}_{k,\jb}$ and index $k,\jb,m$ where $M_{k,\jb,m}^d$ is a $d$-dimensional cardinal B-spline. Lemma \ref{lem.CNNAppBSpl} shows that $M_{k,\jb,m}^d$ can be approximated by a CNN $\tM_{k,\jb,m}^d$ with arbitrary accuracy. Therefor $\tf_{i,j}$ can be approximated by a CNN $\tf_{i,j}^{\rm CNN}$ with arbitrary accuracy. Assume $\|\tM_{k,\jb,m}^d-M_{k,\jb,m}^d\|_{L^{\infty}}\leq \varepsilon_1$ for some $\varepsilon_1\in (0,1)$.
	Then $\tf_{i,j}^{\rm CNN}\in \cF^{\rm CNN}(L,J,K,\kappa,\kappa)$ with
	\begin{align}
		\begin{aligned} 
			&L= O\left( \log(1/\varepsilon_1)\right), J=24dm(m+2)+8d, 
			\kappa=\max\left(|\alpha^{(i)}_{k,\jb}|,2^k\right)=O\left( N^{(\log 2)(\nu^{-1}+c_1d^{-1})(1\vee(d/p-s)_+)}\right),
		\end{aligned}
		\label{eq.tfijlabel0}
	\end{align}
	where the value of $\kappa$ comes from Lemma \ref{lem.alphabound} and (\ref{eq.2bound}).
	
	The rest proof follows that of \citet[Proposition 1]{suzuki2018adaptivity} in which we show that with properly chosen $N$ and $\varepsilon_1$, $\sum_{j=1}^N \tf^{\rm CNN}_{i,j}$ can approximate $f_i\circ \phi_i^{-1}$ with arbitrary accuracy. 
	
	We decompose the error as
	\begin{align}
		\left\|\sum_{j=1}^N \tf^{\rm CNN}_{i,j}-f_i\circ \phi_i^{-1}\right\|_{L^\infty} 
		\leq \left\|\tf_i-f_i\circ\phi_i^{-1}\right\|_{L^\infty}  + \left\|\sum_{j=1}^N \tf^{\rm CNN}_{i,j}- \tf_i\right\|_{L^\infty} .
		\label{eq.ReLUBR.decompose}
	\end{align}
where $\tf_i$ is defined in (\ref{eq.f-BSpline}).
We next derive an error bound for each term.

Let $m$ be the order of the Cardianl B-spline basis. Set $m=\lceil s\rceil+1$. 
According to (\ref{eq.tfierror}) and Lemma \ref{lem.BSplAppBesov},
  \begin{align}\label{eq.ffN}
    \left\|\tf_i-f_i\circ\phi_i^{-1}\right\|_{L^\infty} \leq Cc_0N^{-s/d},\quad
  \|\{\alpha^{(i)}_{k,\jb}\}\|_{\bnorm} \leq C_1 \|f\|_{\Bnorm},
\end{align}
  for some constant $C$ depending on $s,p,q$ and $d$, some universal constant $C_1$ with
  $\{\jb_i\}_{i=1}^{n_k} \subset J(k),\ H=\lceil c_1\log(N)/d \rceil,\ H^*=\lceil \nu^{-1}\log(\lambda N) \rceil +H+1,\ n_k=\lceil \lambda N2^{-\nu (k-H)} \rceil$ for $k=H+1,\dots,H^*,\ u =d(1/p-1/r)_+$, $\nu=(s-u)/(2u)$ and $ \sum_{k=1}^H (2^k+m)^d+\sum_{k=H+1}^{H^*} n_k\leq N$. By setting $N=\left\lceil\left(\frac{\delta}{2Cc_0}\right)^{-d/s}\right\rceil$, we have $\|\tf_i-f_i\circ\phi_i^{-1}\|_{\infty}\leq \delta/2$.


Next we consider the second term in (\ref{eq.ReLUBR.decompose}). For any $\xb\in [0,1]^d$, we have
  \begin{align*}
    &\left|\sum_{j=1}^N \tf^{\rm CNN}_{i,j}(\xb)- \tf_i(\xb)\right|\leq \sum_{(k,\jb)\in \cS_N} |\alpha^{(i)}_{k,\jb}||M_{k,\jb,m}^d(\xb)-\tM_{k,\jb,m}^d(\xb)|\\
    &\leq \varepsilon_1\sum_{(k,\jb)\in \cS_N}  |\alpha^{(i)}_{k,\jb}|\one_{M_{k,\jb,m}^d(\xb)\neq 0} \leq \varepsilon_1(m+1)^d(1+H^*)2^{H^*(d/p-s)_+}\|f\|_{\Bnorm}\\
    &\leq \varepsilon_1(m+1)^d \left(1+\log(\lambda N)\nu^{-1} + c_1\log(N)/d+3 \right)\left(e^3(\lambda N)^{\nu^{-1}} N^{c_1/d}\right)^{(\log2)(d/p-s)_+}\|f\|_{\Bnorm}\\
    &\leq C_2c_0\log(N)N^{(\log2)(\nu^{-1}+c_1d^{-1})(d/p-s)_+}\varepsilon_1\\
    &\leq C_2c_0\log\left(\frac{2}{\varepsilon}\right)\left( \frac{\varepsilon}{2}\right)^{-(\log2)(\nu^{-1}+c_1d^{-1})\left(\frac{d^2}{sp}-d\right)_+}\varepsilon_1
  \end{align*}
  with $C_2$ being some constant depending on $m,d,s,p,q,\nu^{-1}$.
  In the second inequality, $\one_{M_{k,\jb,m}^d(\xb)\neq 0}=1$ if $M_{k,\jb,m}^d(\xb)\neq 0$ and it equals to $0$ otherwise. The third inequality follows from the fact that for each $k$, there are $(m+1)^d$ basis functions which are non-zero at $\xb$ and $2^{k(s-d/p)}\left|\alpha^{(i)}_{k,\jb}\right|\leq \left\|\{\alpha^{(i)}_{k,\jb}\}\right\|_{\bnorm}\leq C_1 \|f\|_{\Bnorm}$. In the fourth inequality we use $H^*=\left\lceil \log(\lambda N)\nu^{-1} \right\rceil +H+1$ and $H=\left\lceil c_1\log(N)/d \right\rceil$, and the last inequality follows from $N=\left\lceil\left(\frac{\varepsilon}{2C}\right)^{-d/s}\right\rceil$. Setting\\ $\varepsilon_1=\frac{1}{C_2c_0\log\left(2/\delta\right)}\left( \frac{\delta}{2}\right)^{\frac{1}{2}+(\log2)(\nu^{-1}+c_1d^{-1}) \left(\frac{d^2}{sp}-d\right)_+}$ proves the error bound.
  
  Under Assumption \ref{assum.f}, $s\geq d/p+1$. Therefore $\left(\frac{d^2}{sp}-d\right)_+=0$ and $\nu=\frac{sp-d}{2d}$. Substituting these expressions into (\ref{eq.tfijlabel0}) gives the network architectures.
\end{proof}

\subsection{Lemma \ref{lem.totalerror}}\label{sec.lem.totalerror}
Lemma \ref{lem.totalerror} estimates the approximation error of $\bar{f}$.
\begin{lemma}\label{lem.totalerror}
	Let $\eta$ be the approximation error of the multiplication operator $\ttimes(\cdot,\cdot)$, $\delta$ be defined as in Proposition \ref{prop.ReLUBR}, $\Delta$ and $\theta$ be defined as in Lemma \ref{lem.indicator}. Assume $N$ is chosen according to Proposition \ref{prop.ReLUBR}. For any $i=1,...,C_{\cM}$, we have $\|\mathring{f}-f^*\|_{L^\infty}\leq \sum_{i=1}^{C_{\cM}} (A_{i,1}+A_{i,2}+A_{i,3})$ with
	\begin{align*}
		&A_{i,1}=\sum_{j=1}^N\left\|\ttimes(\tf_{i,j}^{\rm CNN}\circ \phi_i^{\rm CNN},\mtoned\circ \tdi)-\tf_{i,j}^{\rm CNN}\circ \phi_i^{\rm CNN}\times (\mtoned\circ \tdi)\right\|_{L^\infty}\leq C\delta^{-d/s}\eta,\\
		&A_{i,2}=\left\|\left(\sum_{j=1}^N \left(\tf_{i,j}^{\rm CNN}\circ \phi_i^{\rm CNN}\right)\right)\times(\mtoned\circ \tdi)-f_i\times (\mtoned\circ \tdi)\right\|_{L^\infty}\leq \delta,\\
		&A_{i,3}=\|f_i\times(\mtoned\circ \tdi)-f_i\times \mone_{U_i}\|_{L^\infty}\leq \frac{c(\pi+1)}{\omega(1-\omega/\tau)}\Delta
	\end{align*}
	for some constant $C$  depending on $d,s,p,q$  and some constant $c$. 
	Furthermore, for any $\varepsilon\in(0,1)$, setting
	\begin{align}
		\delta=\frac{\varepsilon}{3C_{\cM}}, \ \eta =\frac{1}{C}\left(\frac{\varepsilon}{3C_{\cM}}\right)^{\frac{d}{s}+1},\Delta=\frac{\omega(1-\omega/\tau)\varepsilon}{3c(\pi+1)C_{\cM}}, \ \theta=\frac{\Delta}{16B^2D}
		\label{eq.parameterchoose}
	\end{align} 
gives rise to 
$$
\|\mathring{f}-f^*\|_{L^\infty}\leq \varepsilon.
$$
The choice in (\ref{eq.parameterchoose}) satisfies the condition $\Delta> 8B^2D\theta$ in Lemma \ref{lem.indicator}.
\end{lemma}

\begin{proof}[Proof of Lemma \ref{lem.totalerror}]
	In the error decomposition, $A_{i,1}$ measures the error from $\ttimes$:
	$$
	A_{i,1}=\sum_{j=1}^N\left\|\ttimes(\tf_{i,j}^{\rm CNN}\circ \phi_i^{\rm CNN},\mtoned\circ \tdi)-\tf_{i,j}^{\rm CNN}\circ \phi_i^{\rm CNN}\times (\mtoned\circ \tdi)\right\|_{L^\infty}\leq N\eta\leq C\delta^{-d/s}\eta,
	$$
	for some constant $C$  depending on $d,s,p,q$.
	
	$A_{i,2}$ measures the error from CNN approximation of Besov functions. According to Proposition \ref{prop.ReLUBR},  $A_{i,2}\leq \delta$.
	
	$A_{i,3}$ measures the error from CNN approximation of the chart determination function. The bound of $A_{i,3}$ can be derived using \citet[Proof of Lemma 4.5]{chen1908nonparametric} since $f_i\circ\phi_i^{-1}$ is a Lipschitz function and its domain is in $[0,1]^d$.
\end{proof}

\subsection{CNN size quantification of $\mathring{f}_{i,j}$}\label{sec.cnnsize}
Let $\mathring{f}_{i,j}$ be defined as in (\ref{eq.f.approx.mlp}). Under the choices of $\delta,\eta,\Delta,\theta$ in Lemma \ref{lem.totalerror}, we quantify the size of each CNN in $\mathring{f}_{i,j}$ as follows:
\begin{itemize}
	\item $\tdi$ has $O(\log(1/\varepsilon)+D+\log D)$ layers, $6D$ channels and all weights parameters are bounded by $4B^2$.
	\item $\mtoned$ has $O(\log(1/\varepsilon))$ layers with $2$ channels. All weights are bounded by $\max(2,\omega^2)$.
	\item$\ttimes$ has $O(\log 1/\varepsilon)$ layers with $6$ channels. All weights are  of $\max(c_0^2,1)$.
	\item $\tf_{i,j}^{\rm CNN}$ has $O(\log 1/\varepsilon)$ layers with $\lceil 24d(s+1)(s+3)+8d\rceil$ channels. All weights are in the order of $O\left(\varepsilon^{-(\log2)\frac{d}{s}(\frac{2d}{sp-d}+c_1d^{-1})}\right)$ where $c_1$ is defined in Lemma \ref{lem.BSplAppBesov}.
	\item $\phi_i^{\rm CNN}$ has $2+D$ layers and $d$ channels. All weights are bounded by $2B$.
\end{itemize}
In the above network architectures, the constant hidden in $O(\cdot)$ depend on $d,s,\frac{2d}{sp-d},p,q,c_0,\tau$ and the surface area of $\cM$. In particular, the constant depends on $D\log D$ linearly.

\subsection{Lemma \ref{lem.cnn.composition} and its proof}\label{sec.lem.cnn.composition.proof}
Lemma \ref{lem.cnn.composition} shows that the composition of two CNNs can be realized by another CNN. Lemma \ref{lem.cnn.composition} is used to prove Lemma \ref{lem.bfijcnn}.
\begin{lemma}\label{lem.cnn.composition}
	Let $\cF_1^{\rm CNN}(L_1,J_1,K_1,\kappa_1,\kappa_1)$ be a CNN architecture from $\RR^D\rightarrow \RR$ and $\cF_2^{\rm CNN}(L_2,J_2,K_2,\kappa_2,\kappa_2)$ be a CNN architecture from $\RR\rightarrow\RR$. Assume the weight matrix in the fully connected layer of $\cF_1^{\rm CNN}(L_1,J_1,K_1,\kappa_1,\kappa_1)$ and $\cF_2^{\rm CNN}(L_2,J_2,K_2,\kappa_2,\kappa_2)$ has nonzero entries only in the first row. Then there exists a CNN architecture $\cF^{\rm CNN}(L,J,K,\kappa,\kappa)$ from $\RR^D\rightarrow \RR$ with
	\begin{align*}
		L=L_1+L_2, \ J=\max(J_1,J_2),\ K=\max(K_1,K_2), \kappa=\max(\kappa_1,\kappa_2)
	\end{align*}
	such that for any $f_1\in \cF^{\rm CNN}(L_1,J_1,K_1,\kappa_1,\kappa_1)$ and $f_2\in \cF^{\rm CNN}(L_2,J_2,K_2,\kappa_2,\kappa_2)$, there exists $f\in \cF^{\rm CNN}(L,J,K,\kappa,\kappa)$ such that $f(\xb)=f_2\circ f_1(\xb)$.
	Furthermore, the weight matrix in the fully connected layer of $\cF^{\rm CNN}(L,J,K,\kappa,\kappa)$ has nonzero entries only in the first row.
\end{lemma}
In Lemma \ref{lem.cnn.composition} and the following lemmas, the subscript of $\cF^{\rm CNN}$ are used to distinguish different network architectures.
\begin{proof}[Proof of Lemma \ref{lem.cnn.composition}]
	Compared to a CNN, directly composing $f_1$ and $f_2$ gives a network with an additional intermediate fully connected layer. In our network construction, we will design two convolutaionl layers to replace and realize this fully connected layer.
	
	Denote $f_1$ and $f_2$ by 
	$$
	f_1(\xb)=W_1\cdot \Conv_{\cW_1,\cB_1}(\xb) \mbox{ and } f_2(\xb)=W_2\cdot \Conv_{\cW_2,\cB_2}(\xb).
	$$
	where $\cW_1=\left\{\cW_1^{(l)}\right\}_{i=1}^{L_1}, \cB_1=\left\{B_1^{(l)}\right\}_{l=1}^{L_1}, \cW_2=\left\{\cW_2^{(l)}\right\}_{i=1}^{L_2}, \cB_2=\left\{B_2^{(l)}\right\}_{l=1}^{L_2},$ are sets of filters and biases and $\Conv_{\cW_1,\cB_1},\Conv_{\cW_2,\cB_2}$ are defined in (\ref{eq.conv}).
	In the rest of this proof, we will choose proper weight parameters in $\cW,\cB$ and $W$ such that $f(\xb)\in \cF^{\rm CNN}(L,J,K,\kappa,\kappa)$ is in the form of 
	$$
	f(\xb)=W\cdot \Conv_{\cW,\cB}(\xb)
	$$
	and satisfies $f(\xb)=f_2\circ f_1(\xb)$.

	For $1\leq l\leq L_1-1$, we set $\cW^{(l)}=\cW^{(l)}_1, B^{(l)}=B^{(l)}_1$. 
	
	For $l=L_1$, to realize the fully connected layer of $f_1$ by a convolutional layer, we set
	\begin{align*}
		\cW^{(L_1)}_{1,:,:}=(W_1)_{1,:},\ \cW^{(L_1)}_{2,:,:}=-(W_1)_{1,:}
	\end{align*}
	and $B^{(L_1)}=\mathbf{0}$. Here $\cW^{(L_1)}\in \RR^{2\times 1\times M}$ is a size-one filter with two output channels, where $M$ is the number of input channels of $W_1$. The output of the $L_1$-th layer of $f$ has the form
	$$
	\begin{bmatrix}
		(f_1(\xb))_+ & (f_1(\xb))_-\\
		\star & \star
	\end{bmatrix}
	$$
	where $\star$ denotes some elements that will not affect the result.
	
	Since the input of $f_2$ is a real number, all filters of $f_2$ has size 1. The weight matrix in the fully connected layer and all biases only have one row.
	For the $(L_1+1)$-th layer, we set
	\begin{align*}
		\cW^{(L_1+1)}_{i,:,:}=\begin{bmatrix}
			(\cW^{(1)}_2)_{i,:,:} & -(\cW^{(1)}_2)_{i,:,:}
		\end{bmatrix},\ B^{L_1+1}=\begin{bmatrix}
			B^{(1))}_2 \\
			\mathbf{0}
		\end{bmatrix}
	\end{align*}
	where $i$ varies from 1 to the number of output channels of $\cW^{(1)}_2$. Here $\cW^{(L_1+1)}$ is a size-one filter whose number of output channels is the same as that of $\cW^{(1)}_2$.
	
	For $L_1+1\leq l\leq L-1$,
	we set
	\begin{align*}
		\cW^{(l)}=\cW^{(l-L_1)}_2, \ B^{l}=\begin{bmatrix}
			B^{(l-L_1))}_2 \\
			\mathbf{0}
		\end{bmatrix},\ \mbox{ for } l=L_1+1,...,L_1+L_2-1.
	\end{align*}
	For $l=L$, we set 
	$$
	W=\begin{bmatrix}
		W_2\\
		\mathbf{0}
	\end{bmatrix}.
	$$
	
	With the above settings, the lemma is proved.
\end{proof}

\subsection{Lemma \ref{lem.cnn.stack} and its proof} \label{sec.lem.cnn.stack.proof}
Lemma \ref{lem.cnn.stack} is used to prove Lemma \ref{lem.bfijcnn}.
\begin{lemma}\label{lem.cnn.stack}
	Let $f_1\in \cF^{\rm CNN}(L_1,J_1,K_1,\kappa_1,\kappa_1)$ be a CNN from $\RR^D\rightarrow \RR$ and $f_2\in \cF^{\rm CNN}(L_2,J_2,K_2,\kappa_2,\kappa_2)$ be a CNN from $\RR^D\rightarrow\RR$. Assume the weight matrix in the fully connected layer of $f_1$ and $f_2$ have nonzero entries only in the first row. Then there exists a set of filters $\cW$ and biases $\cB$ such that 
	$$
	\Conv_{\cW,\cB}(\xb)=\begin{bmatrix}
		(f_1(\xb))_+ & (f_1(\xb))_- & (f_2(\xb))_+ & (f_2(\xb))_-\\
		\star &  \star & \star & \star
	\end{bmatrix} \in \RR^{D\times 4}.
	$$
	Such a network has $\max(L_1,L_2)$ layers, each filter has size at most $\max(K_1,K_2)$ and at most $J_1+J_2$ channels. All parameter are bounded by $\max(\kappa_1,\kappa_2)$.
\end{lemma}

\begin{proof}[Proof of Lemma \ref{lem.cnn.stack}]
	For simplicity, we assume all convolutional layers of $f_1$ have $J_1$ channels and all convolutional layers of $f_2$ have $J_2$ channels. If some filters in $f_1$ (or $f_2$) have channels less than $J_1$ (or $J_2$), we can add additional channels with zero filters and biases.
	Without loss of generality, we assume $L_1>L_2$.
	
	Denote$f_1$ and $f_2$ by 
	$$
	f_1(\xb)=W_1\cdot \Conv_{\cW_1,\cB_1}(\xb) \mbox{ and } f_2(\xb)=W_2\cdot \Conv_{\cW_2,\cB_2}(\xb),
	$$
	where $\cW_1=\left\{\cW_1^{(l)}\right\}_{i=1}^{L_1}, \cB_1=\left\{B_1^{(l)}\right\}_{l=1}^{L_1}, \cW_2=\left\{\cW_2^{(l)}\right\}_{i=1}^{L_2}, \cB_2=\left\{B_2^{(l)}\right\}_{l=1}^{L_2},$ are sets of filters and biases and $\Conv_{\cW_1,\cB_1},\Conv_{\cW_2,\cB_2}$ are defined in (\ref{eq.conv}).
	In the rest of this proof, We will choose proper weight parameters in $\cW,\cB$ such that
	$$
	\Conv_{\cW,\cB}(\xb)=\begin{bmatrix}
		(f_1(\xb))_+ & (f_1(\xb))_- & (f_2(\xb))_+ & (f_2(\xb))_-\\
		\star &  \star & \star & \star
	\end{bmatrix} \in \RR^{D\times 4}.
	$$
	
	For $1\leq l\leq L_2-1$, we set
	\begin{align*}
		&\cW^{(l)}_{i,:,:}=\begin{bmatrix}
			(\cW_1^{(l)})_{i,:,:} & \mathbf{0}
		\end{bmatrix} \mbox{ for } i=1,...,J_1,\\
		&\cW^{(l)}_{i,:,:}=\begin{bmatrix}
			\mathbf{0} &(\cW_2^{(l)})_{i-J_1,:,:} 
		\end{bmatrix} \mbox{ for } i=J_1+1,...,J_1+J_2,\\
		&B^{(l)}=\begin{bmatrix}
			B_1^{(l)} & B_2^{(l)}
		\end{bmatrix}.
	\end{align*}
	Each $\cW^{(l)}$ is a filter with size $\max(K_1,K_2)$ and $J_1+J_2$ output channels. When $K_1\neq K_2$, we pad the smaller filter by zeros. For example when $K_1<K_2$, we set 
	\begin{align*}
		\cW^{(l)}_{i,:,:}=\begin{bmatrix}
			(\cW_1^{(l)})_{i,:,:} & \mathbf{0}\\
			\mathbf{0} &\mathbf{0}
		\end{bmatrix} \mbox{ for } i=1,...,J_1,
	\end{align*}
such that $\cW^{(l)}$ has size $K_2$ filters.

	For the $L_2$-th layer, we set 
	\begin{align*}
		&\cW^{(L_2)}_{i,:,:}=\begin{bmatrix}
			(\cW_1^{(L_2)})_{i,:,:} & \mathbf{0}
		\end{bmatrix} \mbox{ for } i=1,...,J_1,\\
		&\cW^{(L_2)}_{J_1+1,:,:}=\begin{bmatrix}
			\mathbf{0} & (W_2 )_{1,:}\\
			\mathbf{0} & \mathbf{0}
		\end{bmatrix},\ \cW^{(L_2)}_{J_1+2,:,:}=\begin{bmatrix}
			\mathbf{0} & -(W_2 )_{1,:}\\
			\mathbf{0} & \mathbf{0}
		\end{bmatrix}, \\
		&B^{(L_2)}=\begin{bmatrix}
			B_1^{(L_2)} & \mathbf{0}
		\end{bmatrix}.
	\end{align*}
	$\cW^{(L_2)}$ is a filter with size $K_1$ and $J_1+2$ output channels.
	
	For $L_2+1\leq l\leq L_1-1$, we set
	\begin{align*}
		&\cW^{(l)}_{i,:,:}=\begin{bmatrix}
			(\cW_1^{(l)})_{i,:,:} & \mathbf{0}
		\end{bmatrix} \mbox{ for } i=1,...,J_1,\\
		&\cW^{(l)}_{J_1+1,:,:}=\begin{bmatrix}
			\mathbf{0} & 1 & 0\\
			\mathbf{0} & \mathbf{0} & \mathbf{0}
		\end{bmatrix},\ \cW^{(l)}_{J_1+2,:,:}=\begin{bmatrix}
			\mathbf{0} & 0 & 1\\
			\mathbf{0} & \mathbf{0} & \mathbf{0}
		\end{bmatrix}, \\
		&B^{(L_2)}=\begin{bmatrix}
			B_1^{(L_2)} & \mathbf{0}
		\end{bmatrix}.
	\end{align*}
	For the $L_1$-th layer, we set
	$$
	\cW^{(L_1)}_{1,:,:}=\begin{bmatrix}
		(W_1)_{1,:} & 0 & 0\\
		\mathbf{0} & \mathbf{0} & \mathbf{0}
	\end{bmatrix}, \ 
	\cW^{(L_1)}_{2,:,:}=\begin{bmatrix}
		-(W_1)_{1,:} & 0 & 0\\
		\mathbf{0} & \mathbf{0} & \mathbf{0}
	\end{bmatrix},
	\cW^{(L_1)}_{3,:,:}=\begin{bmatrix}
		\mathbf{0} & 1 & 0\\
		\mathbf{0} & \mathbf{0} & \mathbf{0}
	\end{bmatrix},\ 
	\cW^{(L_1)}_{4,:,:}=\begin{bmatrix}
		\mathbf{0} & 0 & 1\\
		\mathbf{0} & \mathbf{0} & \mathbf{0}
	\end{bmatrix},
	$$
	and $B^{(L_1)}=\mathbf{0}$.
\end{proof}

\subsection{Lemma \ref{lem.cnn.composition2} and its proof}\label{sec.lem.cnn.composition2.proof}
Lemma \ref{lem.cnn.composition2} is used to prove Lemma \ref{lem.bfijcnn}.
\begin{lemma}\label{lem.cnn.composition2}
	Let $M,N$ be positive integers. For any $\xb=\begin{bmatrix}
		x_1 & \cdots & x_{M}
	\end{bmatrix}^{\top} \in \RR^{M}$, define 
	$$
	X=\begin{bmatrix}
		(x_1)_+ &(x_1)_- & \cdots & (x_{M})_+ & (x_{M})_-\\
		\star & \star  &\cdots & \star &\star 
	\end{bmatrix}\in \RR^{N\times (2M)}.
	$$
	For any CNN architecture $\cF_1^{\rm CNN}(L,J,K,\kappa,\kappa)$ from $\RR^{M}\rightarrow \RR$, there exists a CNN architecture $\cF^{\rm CNN}(L,MJ,K,\kappa,\kappa)$ from $\RR^{N\times (2M)}\rightarrow \RR$ such that for any $f_1\in \cF_1^{\rm CNN}(L,J,K,\kappa,\kappa)$, there exists $f\in \cF^{\rm CNN}(L,MJ,K,\kappa,\kappa)$ with $f(X)=f_1(\xb)$.
	Furthermore, the fully connected layer of $\cF^{\rm CNN}(L,MJ,K,\kappa,\kappa)$ has nonzero entries only in the first row.
\end{lemma}

\begin{proof}[Proof of Lemma \ref{lem.cnn.composition2}]
	Denote $f_1$ as
	\begin{align*}
		f_1=W_1\cdot \Conv_{\cW_1,\cB_1}(\xb)
	\end{align*}
where $\cW_1=\left\{\cW_1^{(l)}\right\}_{i=1}^{L_1}, \cB_1=\left\{B_1^{(l)}\right\}_{l=1}^{L_1}$ are sets of filters and biases and $\Conv_{\cW_1,\cB_1}$ is defined in (\ref{eq.conv}).
For simplicity, we assume all convolutional layers of $f_1$ have $J$ channels . If some filters in $f$ have less than $J$ channels, we can add additional channels with zero filters and biases.

	We next choose proper weight parameters in $\cW,\cB$ and $W$ such that 
	$$f=W\cdot \Conv_{\cW,\cB}(x)$$ 
	and $f(X)=f_1(\xb)$.


	
	For the first layer, i.e., $l=1$, we design $\cW^{(1)}$ and $B^{(1)}$ since $\xb$ is a vector in $ \RR^{M}$, the filter $\cW_1^{(1)}$ has 1 input channel and $J$ output channel. For $1\leq i\leq J$ and $ 1\leq j\leq M-K+1$, we set
	\begin{align*}
		\cW^{(1)}_{(i-1)M+j,:,:}=\begin{bmatrix}
			\mathbf{0} & (\cW^{(1)}_1)_{i,1,:} & -(\cW^{(1)}_1)_{i,1,:} & \cdots & (\cW^{(1)}_1)_{i,K,:} & -(\cW^{(1)}_1)_{i,K,:} &\tilde{\mathbf{0}}
		\end{bmatrix} 
	\end{align*}
	where $\mathbf{0}$ is of size $1\times (j-1)$, $\tilde{\mathbf{0}}$ is a zero matrix of size $1\times (M-j-K+1)$.
	
	For $1\leq i\leq J$ and $ M-K+2\leq j \leq M$, we set
	\begin{align*}
		\cW^{(1)}_{(i-1)M+j,:,:}=\begin{bmatrix}
			\mathbf{0} & (\cW^{(1)}_1)_{i,1,:} & -(\cW^{(1)}_1)_{i,1,:} & \cdots & (\cW^{(1)}_1)_{i,K-j+1,:} & -(\cW^{(1)}_1)_{i,K-j+1,:}
		\end{bmatrix} 
	\end{align*}
	where $\mathbf{0}$ is of size $1\times (j-1)$. The bias is set as 
	$$
	B^{(1)}=\begin{bmatrix}
		\left((B_1^{(1)})_{:,1}\right)^\top & \cdots & \left((B_1^{(1)})_{:,J}\right)^\top\\
		\mathbf{0} & \cdots & \mathbf{0}
	\end{bmatrix}.
	$$
	
	We next choose weight parameters for $2\leq l\leq L_1-1$.
	For $1\leq i\leq J$ and $ 1\leq j\leq M-K+1$, we set
	\begin{align*}
		\cW^{(l)}_{(i-1)M+j,:,:}=\begin{bmatrix}
			\mathbf{0} & \left((\cW^{(l)}_1)_{i,:,1}\right)^{\top}& \tilde{\mathbf{0}}   & \cdots & \mathbf{0} & \left((\cW^{(l)}_1)_{i,:,J} \right)^{\top} &\tilde{\mathbf{0}}
		\end{bmatrix} 
	\end{align*}
	where $\mathbf{0}$ is of size $1\times (j-1)$, $\tilde{\mathbf{0}}$ is a zero matrix of size $1\times (M-j-K+1)$.
	
	For $1\leq i\leq J$ and $ M-K+2\leq j\leq M$, we set
	\begin{align*}
		\cW^{(l)}_{(i-1)M+j,:,:}=\begin{bmatrix}
			\mathbf{0} & \left((\cW^{(l)}_1)_{i,1:K-j+1,1}\right)^{\top}  & \cdots & \mathbf{0} & \left((\cW^{(l)}_1)_{i,1:K-j+1,J} \right)^{\top}
		\end{bmatrix} 
	\end{align*}
	where $\mathbf{0}$ is of size $1\times (j-1)$. The bias is set as
	$$
	B^{(l)}=\begin{bmatrix}
		\left((B_1^{(l)})_{:,1}\right)^\top & \cdots & \left((B_1^{(l)})_{:,J}\right)^\top\\
		\mathbf{0} & \cdots & \mathbf{0}
	\end{bmatrix}.
	$$
	
	For the fully connected layer, we set
	$$
	W=\begin{bmatrix}
		\left((W_1)_{:,1}\right)^\top & \cdots & \left((W_1)_{:,J}\right)^\top\\
		\mathbf{0} & \cdots & \mathbf{0}
	\end{bmatrix}.
	$$
	With these choices, the lemma is proved.
\end{proof}

\subsection{Lemma \ref{lem.cnn.rescale} and its proof}\label{sec.lem.cnn.rescale.proof}
Lemma \ref{lem.cnn.rescale} shows that for any CNN, if we scale all weight parameters in convolutional layers by some factors and scale the weight parameters in the fully connected layer properly, the output will remain the same. Lemma \ref{lem.cnn.rescale} is used to prove Lemma \ref{lem.bfijcnn}.
\begin{lemma}\label{lem.cnn.rescale}
	Let $\alpha\geq1$. For any $f \in \cF^{\rm CNN}(L,J,K,\kappa_1,\kappa_2)$, there exists $\tf\in \cF^{\rm CNN}(L,J,K,\alpha^{-1}\kappa_1,\alpha^L\kappa_2)$ such that $\tf(\xb)=f(\xb)$. 
\end{lemma}

\begin{proof}[Proof of Lemma \ref{lem.cnn.rescale}]
	This lemma is proved using the linear property of $\ReLU$ and convolution.
	Let $f$ be any CNN in $\cF^{\rm CNN}(L,J,K,\kappa_1,\kappa_2)$. Denote its architecture as
	$$
	f(\xb)=W\cdot \Conv_{\cW,\cB}(\xb).
	$$
	Define $\bar{W}=\alpha^LW$ and $\bar{\cW},\bar{\cB}$ as
	$$
	\bar{\cW}^{(l)}=\alpha^{-1}\cW^{(l)}, \bar{\cB}^{(l)}=\alpha^{-l}\cB^{(l)}
	$$
	for any $l\in(1,L)$. Set
	$$
	\tf(\xb)=\bar{W}\cdot \Conv_{\bar{\cW},\bar{\cB}}(\xb)
	$$
	We have $\tf\in \cF^{\rm CNN}(L,J,K,\alpha^{-1}\kappa_1,\alpha^L\kappa_2)$ and $\tf(\xb)=f(\xb)$ since $\ReLU(c\xb)=c\ReLU(\xb)$ for any $c>0$.
\end{proof}

\subsection{Lemma \ref{lem.bfijcnn} and its proof} \label{sec.lem.bfijcnn.proof}
Lemma \ref{lem.bfijcnn} shows that each $\mathring{f}_{i,j}$ defined in (\ref{eq.bfij}) can be realized by a CNN.
\begin{lemma}\label{lem.bfijcnn}
	Let $\mathring{f}_{i,j}$ be defined as in (\ref{eq.bfij}). Assume each CNN  in $\mathring{f}_{i,j}$ has architecture discussed in Appendix \ref{sec.cnnsize}. Then there exists a CNN $\bar{f}^{\rm CNN}_{i,j}\in \cF^{\rm CNN}(L,J,K,\kappa_1,\kappa_2) $ with
	\begin{align*}
		L=O(\log (1/\varepsilon)+D+\log D), J=\lceil 48d(s+1)(s+3)+28d+6D\rceil, \kappa_1=1, \log \kappa_2=O\left(\log^2 \frac{1}{\varepsilon}\right)
	\end{align*}
	such that $\bar{f}^{\rm CNN}_{i,j}(\xb)=\mathring{f}_{i,j}(\xb)$ for any $\xb\in\cM$. The constants hidden in $O(\cdot)$ depend on $d,D,s,\frac{2d}{sp-d},p,q,c_0,\tau$ and the surface area of $\cM$.
\end{lemma}
\begin{proof}[Proof of Lemma \ref{lem.bfijcnn}]
	According to Lemma \ref{lem.cnn.composition}, there exists a CNN $g_{i,j}$ realizing $\tf_{i,j}^{\rm CNN}\circ \phi_i^{\rm CNN}$ and a CNN $\tilde{g}_{i}$ realizing  $\mtoned\circ\tdi$. Using Lemma \ref{lem.cnn.stack}, one can construct a CNN excluding the fully connected layer, denoted by $\bar{g}_{i,j}$, such that 
	\begin{align}
		\bar{g}_{i,j}(\xb)=\begin{bmatrix}
			(g_{i,j}(\xb))_+ & (g_{i,j}(\xb))_- & (\tilde{g}_{i}(\xb))_+ & (\tilde{g}_{i}(\xb))_-\\
			\star & \star & \star &\star
		\end{bmatrix}.
		\label{eq.g.output}
	\end{align}
	Here $\bar{g}_{i,j}$ has $\lceil 48d(s+1)(s+3)+28\rceil$ channels.
	
	Since the input of $\ttimes$ is $\begin{bmatrix}
		g_{i,j}\\
		\tilde{g}_{i}
	\end{bmatrix},
	$
	Lemma \ref{lem.cnn.composition2} shows that there exists a CNN $\mathring{g}^{\rm CNN}_{i,j}$ which takes (\ref{eq.g.output}) as the input and outputs $\ttimes(g_{i,j},\tilde{g}_{i})$. 
	
	Note that $\bar{g}_{i,j}$ only contains convolutional layers. The composition $\mathring{g}^{\rm CNN}_{i,j}\circ\bar{g}_{i,j}$, denoted by $\breve{f}^{\rm CNN}_{i,j}$, is a CNN and for any $\xb\in\cM$, $\breve{f}^{\rm CNN}_{i,j}(\xb)=\bar{f}_{i,j}(\xb)$. We have $\breve{f}^{\rm CNN}_{i,j}\in \cF^{\rm CNN}(L,J,K,\kappa,\kappa) $ with
	\begin{align*}
		L=O\left(\log \frac{1}{\varepsilon}+D+\log D\right),\ J=\lceil 28d(s+1)(s+3)+18d\rceil+6D,\ \kappa=O\left(\varepsilon^{-(\log2)\frac{d}{s}(\frac{2d}{sp-d}+c_1d^{-1})}\right).
	\end{align*}
	and $K$ can be any integer in $[2,D]$.
	
	We next rescale all parameters  in convolutional layers of $\breve{f}^{\rm CNN}_{i,j}$ to be no larger than 1.
	Using Lemma \ref{lem.cnn.rescale}, we can realize $\breve{f}^{\rm CNN}_{i,j}$ by $\bar{f}^{\rm CNN}_{i,j}\in \cF^{\rm CNN}(L,J,K,\alpha^{-1}\kappa,\alpha^L\kappa) $ for any $\alpha>1$.
	Set $\alpha=C'\varepsilon^{-(\log 2)\frac{d}{s}(\frac{2d}{sp-d}+c_1d^{-1})}(8KD)M^{\frac{1}{L}}$ where $C'$ is a constant such that $\kappa\leq C'\varepsilon^{-(\log2)\frac{d}{s}(\frac{2d}{sp-d}+c_1d^{-1})}$. With this $\alpha$, we have $\bar{f}^{\rm CNN}_{i,j}\in \cF^{\rm CNN}(L,J,K,\kappa_1,\kappa_2) $ with
	\begin{align*}
		&L=O(\log (1/\varepsilon)+D+\log D),\  J=\lceil 28d(s+1)(s+3)+18d\rceil+6D,\\
		& \kappa_1=(8KD)^{-1}M^{-\frac{1}{L}}=O(1), \ \log \kappa_2=O\left(\log^2 1/\varepsilon\right).
	\end{align*}

\end{proof}

\subsection{Lemma \ref{lem.cnn.convresnet} and its proof}\label{sec.lem.cnn.convresnet.proof}
Lemma \ref{lem.cnn.convresnet} shows that the sum of a set of CNNs can be realized by a ConvResNet.
\begin{lemma}\label{lem.cnn.convresnet}
	Let $\cF^{\rm CNN}(L,J,K,\kappa_1,\kappa_2)$ be any CNN architecture from $\RR^D$ to $\RR$. Assume the weight matrix in the fully connected layer of $\cF^{\rm CNN}(L,J,K,\kappa_1,\kappa_2)$ has nonzero entries only in the first row. Let $M$ be a positive integer. There exists a ConvResNet architecture $\cC(M,L,J,\kappa_1, \kappa_2(1\vee \kappa_1^{-1}))$ such that for any $\{f_i(\xb)\}_{m=1}^M\subset\cF^{\rm CNN}(L,J,K,\kappa_1,\kappa_2) $, there exists $\bar{f}\in \cC(M,L,J,\kappa_1, \kappa_2(1\vee \kappa_1^{-1}))$ with
	$$
	\bar{f}(\xb)=\sum_{m=1}^M \tf_m(\xb).
	$$
	
\end{lemma}

\begin{proof}[Proof of Lemma \ref{lem.cnn.convresnet}]
	Denote the architecture of $f_m$ by
	$$
	f_m(\xb)=W_m\cdot \Conv_{\cW_m,\cB_m}(\xb)
	$$
	with $\cW_m=\left\{\cW_m^{(l)}\right\}_{l=1}^L, \cB_m=\left\{B_m^{(l)}\right\}_{l=1}^L.$ In $\bar{f}$, denote the weight matrix and bias in the fully connected layer by $\bar{W},\bar{b}$ and the set of filters and biases in the $m$-th block by $\bar{\cW}_m$ and $\bar{\cB}_m$, respectively. 
	The padding layer $P$ in $\bar{f}$ pads the input from $\RR^D$ to $\RR^{D\times 3}$ by zeros. Here each column denotes a channel. 
	
	We first show that for each $m$, there exists a subnetowrk $\Conv_{\bar{\cW}_m,\bar{\cB}_m}: \RR^{D\times 3}\rightarrow \RR^{D\times 3}$ such that for any $Z\in\RR^{D\times 3} $ in the form of 
	\begin{align}
			Z=\begin{bmatrix}
			\xb & \star & \star
		\end{bmatrix},
	\label{eq.Zpad}
	\end{align}
	we have
	\begin{align}
		\Conv_{\bar{\cW}_m,\bar{\cB}_m}(Z)=\begin{bmatrix}
			0 & \frac{\kappa_1}{\kappa_2}(f_m)_+ & \frac{\kappa_1}{\kappa_2}(f_m)_-\\
			\mathbf{0} & \star & \star
		\end{bmatrix}
		\label{eq.padZ}
	\end{align}
	where $\star$ denotes some entries that we do not care.
	
	For any $m$, the first layer of $\tf_m$ takes input in $\RR^D$. As a result, the filters in $\cW_m^{(1)}$ are in $\RR^D$. We pad these filters by zeros to get filters in $\RR^{D\times 3}$ and construct $\bar{\cW}_m^{(1)}$ as
	$$
	(\bar{\cW}_m^{(1)})_{j,:,:}=\begin{bmatrix}
		(\cW_m^{(1)})_{j,:,:} & \mathbf{0} & \mathbf{0}
	\end{bmatrix}.
	$$
	For any $Z$ in the form of (\ref{eq.Zpad}), we have $\bar{\cW}_m^{(1)}*Z=\cW_m^{(1)}*\xb$.
	For the filters in the following layers and all biases, we simply set
	\begin{align*}
		&\bar{\cW}_m^{(l)}=\cW_m^{(1)} & \mbox{ for } l=2,\dots,L-1,\\
		&\bar{\cB}_m^{(l)}=\cB_m^{(1)} &\mbox{ for } l=1,\dots,L-1.
	\end{align*}
	In $\Conv_{\bar{\cW}_m,\bar{\cB}_m}$, another convolutional layer is constructed to realize the fully connected layer in $f_m$. According to our assumption,  only the first row of $W_m$ has nonzero entries. We set $\bar{\cB}_m^{(L)}=\mathbf{0}$ and $\bar{\cW}_m^{L}$ as size one filters with three output channels in the form of
	$$
	(\bar{\cW}_m^{L})_{1,:,:}=\mathbf{0},\ (\bar{\cW}_m^{L})_{2,:,:}=\frac{\kappa_1}{\kappa_2}(W_m)_{1,:}, \ (\bar{\cW}_m^{L})_{3,:,:}=-\frac{\kappa_1}{\kappa_2}(W_m)_{1,:}.
	$$
	Under such choices, (\ref{eq.padZ}) is proved and all parameters in $\bar{\cW}_m,\bar{\cB}_m$ are bounded by $\kappa_1$.
	
	By composing all residual blocks, one has
	\begin{align*}
		&(\Conv_{\bar{\cW}_M,\bar{\cB}_M}+\id)\circ \cdots \circ (\Conv_{\bar{\cW}_1,\bar{\cB}_1}+\id)\circ P(\xb)
		= \begin{bmatrix}
			& \frac{\kappa_1}{\kappa_2}\sum_{m=1}^M (f_m)_+ &  \frac{\kappa_1}{\kappa_2}\sum_{m=1}^M (f_m)_+\\
			\xb & \star & \star\\
			& \vdots & \vdots
		\end{bmatrix}.
	\end{align*}
	The fully connect layer is set as
	$$
	\bar{W}=\begin{bmatrix}
		0 & \frac{\kappa_2}{\kappa_1} & -\frac{\kappa_2}{\kappa_1}\\
		\mathbf{0} & \mathbf{0} & \mathbf{0}
	\end{bmatrix}, \ \bar{b}=0.
	$$
	The weights in the fully connected layer are bounded by $\kappa_2(1\vee \kappa_1^{-1})$.
	
	Such a ConvResNet gives
	\begin{align*}
		\bar{f}(\xb)=\sum_{m=1}^M (f_m(\xb))_+-\left(\sum_{m=1}^M (f_m(\xb))_-\right) =\sum_{m=1}^M f_m(\xb).
	\end{align*}

\end{proof}

\section{Proof of Lemmas and Propositions in Section \ref{sec.proof.class}}
\label{sec.proof.class.lemma}
\subsection{Proof of Proposition \ref{thm.sketch2.probability}}\label{proof.probability}
The proof of Proposition \ref{thm.sketch2.probability} relies on the following large-deviation inequality:
\begin{lemma}[Theorem 3 in \citet{shen1994convergence}]\label{lem.largedeviation}
  Let $\{\xb_i\}_{i=1}^n$ be i.i.d. samples from some probability distribution. Let $\cF$ be a class of functions whose magnitude is bounded by $F$. Let $v\geq \sup_{f\in\cF} \Var(f(\xb)),v>0$ be a constant. Assume there are constants $M>0, 0<\lambda<1$ such that
  \begin{enumerate}[label={({B}\arabic*)}]
    \item $\cH_B(v^{1/2},\cF,\|\cdot\|_{L^2})\leq \lambda nM^2/(8(4v+MF/3))$,
    \item $M\leq \lambda v/(4F),\  v^{1/2}\leq F$,
    \item if $\lambda M/8\leq v^{1/2}$, then
    $$
    M^{-1}\int_{\lambda M/32}^{v^{1/2}} \cH_B(u,\cF,\|\cdot \|_{L^2})^{1/2}du\leq \frac{n^{1/2}\lambda^{3/2}}{2^{10}}.
    $$
  \end{enumerate}
  Then
  $$
  \PP\left(\sup_{f\in \cF}\frac{1}{n} \sum_{i=1}^n(f(\xb_i)-\EE[f(\xb_i)])\geq M\right)\leq 3\exp\left(-(1-\lambda) \frac{nM^2}{2(4v+MF/3)}\right).
  $$
\end{lemma}
\begin{proof}[Proof of Proposition \ref{thm.sketch2.probability}]
  Let $C_1,C_2,C_3$ be constants defined in Proposition \ref{thm.sketch2.probability}. Set $\epsilon^2_n=\max\{2a_n,2^7\delta_n/C_1\}$.  For each $\cF_n$ defined in (A2), define
  \begin{align}
  E_n(f)=\frac{1}{n}\sum_{i=1}^n \left[\phi(y_if_n(\xb_i))-\phi(y_if(\xb_i))-\EE\left[\phi(y_if_n(\xb_i))-\phi(y_if(\xb_i))\right]\right].
  \label{eq.thm2.en}
  \end{align}
  Note that $\cE_{\phi}(f_n,f_{\phi}^*)\leq a_n$ by (A2). Since $\hf_{\phi,n}$ is the minimizer of $\cE_{\phi}(f)$, we have
  $$
  \PP\left(\cE_{\phi}(\hf_{\phi,n},f_{\phi}^*)\geq \epsilon_n^2\right)\leq \PP\left(\left[\sup_{f\in \cF_n: \cE_{\phi}(f,f_{\phi}^*)\geq \epsilon_n^2}\frac{1}{n} \sum_{i=1}^n \left[\phi(y_if_n(\xb_i))-\phi(y_if(\xb_i))\right]\right]\geq 0\right).
  $$
  We decompose the set $\left\{f\in\cF_n: \cE_{\phi}(f,f_{\phi}^*)\geq \epsilon_n^2\right\}$ into disjoint subsets $\{\cF_{n,i}\}$ for $i=1,...,i_n$ in the form of
  $$
  \cF_{n,i}=\left\{f\in \cF_n: 2^{i-1}\epsilon_n^2\leq \cE_{\phi}(f,f_{\phi}^*)< 2^i\epsilon_n^2\right\}.
  $$
  Note that $\|f\|_{L^\infty}\leq F_n$ and $\cE_{\phi}(f_{\phi}^*)\leq \cE_{\phi}(f)$ for all $f\in \cF_n$. Therefore, for any $f\in \cF_n$, we have
  $$\cE_{\phi}(f,f^*_{\phi})\leq \cE_{\phi}(f)=\EE[\phi(yf(\xb))]\leq C_1\EE[|f(\xb)|]\leq C_1F_n,$$
  which implies $\cF_{n,i}$ is an empty set for $2^{i-1}\epsilon_n^2>C_1F_n$. We set $i_n=\inf\{i\in\NN: 2^{i-1}\epsilon_n^2>C_1F_n\}$. Then we have $\left\{f\in\cF_n: \cE_{\phi}(f,f_{\phi}^*)\geq \epsilon_n^2\right\}=\bigcup_{i=1}^{i_n} \cF_{n,i}$. Since $\cE_{\phi}(f_n,f_{\phi}^*)\leq a_n\leq \epsilon_n^2/2$, we have
  $$
  \inf_{f\in \cF_{n,i}} \EE\left[\phi(yf(\xb))-\phi(yf_n(\xb)) \right]=\inf_{f\in \cF_{n,i}}\left[\cE_{\phi}(f,f_{\phi}^*)- \cE_{\phi}(f_n,f_{\phi}^*)\right]\geq 2^{i-2}\epsilon_n^2.
  $$
  Denote $M_{n,i}=2^{i-2}\epsilon_n^2$. We have
  \begin{align*}
     \PP\left(\cE_{\phi}(\hf_{\phi,n},f_{\phi}^*)\geq \epsilon_n^2\right)&\leq \sum_{i=1}^{i_n}\PP\left(\sup_{f\in\cF_{n,i}} \EE_n(f)\geq \EE\left[\phi(yf(\xb))-\phi(yf_n(\xb)) \right]\right)\\
     &\leq \sum_{i=1}^{i_n}\PP\left(\sup_{f\in\cF_{n,i}} \EE_n(f)\geq M_{n,i} \right),
  \end{align*}
where $\EE_n(f)$ is defined in (\ref{eq.thm2.en}).
  Then we will bound each summand on the right-hand side using Lemma \ref{lem.largedeviation}. First, using (A4), we have
  \begin{align*}
    &\sup_{f\in \cF_{n,i}} \EE\left[ \left(\phi(yf(\xb))-\phi(yf_n(\xb))\right)^2\right]\\
\leq&2\sup_{f\in \cF_{n,i}} \EE\left[ \left(\phi(yf(\xb))-\phi(yf_{\phi}^*(\xb))\right)^2+\left(\phi(yf_n(\xb))-\phi(yf_{\phi}^*(\xb))\right)^2\right]\\
\leq& 2C_2e^{F_n}F_n^{2-v}\left(\sup_{f\in \cF_{n,i}}\cE_{\phi}(f,f_{\phi}^*)^{\nu}+\cE_{\phi}(f_n,f_{\phi}^*)^{\nu}\right)\\
\leq& 2C_2e^{F_n}F_n^{2-\nu}\left(\left(2^i\epsilon_n^2\right)^{\nu}+\left(\epsilon_n^2/2\right)^{\nu}\right)\leq 4^{\nu+1}C_2e^{F_n}F_n^{2-\nu}\left(\left(2^{i-2}\epsilon_n^2\right)^{\nu}\right)\\
=&4^{\nu+1}C_2e^{F_n}F_n^{2-\nu}M_{n,i}^{\nu}.
  \end{align*}
  Define $\cG_{n,i}=\left\{g=\phi(yf_n(\xb))-\phi(yf(\xb)): f\in \cF_{n,i}\right\}$. For any $g\in \cG_{n,i}$, $\Var(g)\leq 4^{\nu+1}C_2e^{F_n}F_n^{2-\nu}M_{n,i}^{\nu}$. Since $f_n,f\in \cF_n$, we have $\|g\|_{L^\infty}\leq C_1|f_n-f|\leq 2C_1F_n$.

  To apply Lemma \ref{lem.largedeviation} on $\cG_{n,i}$, we set $\lambda=1/2,\ F=D_1F_n,\ M=M_{n,i}$ and $v=v_{n,i}=D_2F_n^{2-\nu}M_{n,i}^{\nu}$ with
  $$
  D_1=\frac{1}{8(2C_1)^{1-\nu}}D_2,\quad D_2=\max\left\{4^{1+\nu}C_2e^{F_n},64(2C_1)^{2-\nu}\right\}.
  $$
  Since $D_1\geq 2C_1$ and $D_2\geq 4^{1+\nu}C_2e^{F_n}$, we have $\sup_{g\in \cG_{n,i}}\|g\|_{L^\infty}\leq F$ and $\sup_{g\in \cG_{n,i}}\Var(g)\leq v_{n,i}$.
  We first check the validation of (B2). Since $M_{n,i}\leq 2C_1F_n,D_2\geq64(2C_1)^{2-\nu}$,
  \begin{align*}
    \frac{v}{F^2}=\frac{v_{n,i}}{D_1^2F_n^2}\leq\frac{64(2C_1)^{2-2\nu}D_2F_n^{2-\nu}(2C_1F_n)^{\nu}}{D_2^2F_n^2} =\frac{64(2C_1)^{2-\nu}}{D_2}\leq 1,
  \end{align*}
  \begin{align*}
    M_{n,i}=M_{n,i}^{1-\nu}M_{n,i}^{\nu}\leq (2C_1F_n)^{1-\nu}M_{n,i}^{\nu}= \frac{8(2C_1)^{1-\nu}D_2F_n^{2-\nu}M_{n,i}^{\nu}}{8D_2F_n} =\frac{v_{n,i}}{8D_2F_n}\leq \frac{v_{n,i}}{8F_n}
  \end{align*}
  when $F_n$ is large enough. Thus (B2) is satisfied.

  For (B3), note that for $g_1=\phi(yf_n(\xb))-\phi(yf_1(\xb)),g_2=\phi(yf_n(\xb))-\phi(yf_2(\xb))$ where $f_1,f_2\in \cF_{n,i}$, $|g_1-g_2|\leq C_1|f_1-f_2|$. We have
  $$
  \cH_B(\delta,\cG_{n,i},\|\cdot\|_{L^2})\leq \cH_B(C_1\delta,\cF_{n,i},\|\cdot\|_{L^2})\leq \cH_B(C_1\delta,\cF_n,\|\cdot\|_{L^2}),
  $$
  where the second inequality comes from $\cF_{n,i}\subset\cF_n$. Since $M_{n,i}^{-1}\displaystyle\int_{\lambda M_{n,i}/32}^{v_{n,i}^{1/2}} \left(\cH_B(\tau,\cG_{n,i},\|\cdot\|_{L^2})\right)^{1/2}d\tau$ is a non-increasing function of $i$,
  \begin{align*}
    &M_{n,i}^{-1}\int_{\lambda M_{n,i}/32}^{v_{n,i}^{1/2}} \left(\cH_B(\tau,\cG_{n,i},\|\cdot\|_{L^2})\right)^{1/2}d\tau \\
    \leq& M_{n,1}^{-1}\int_{ M_{n,1}/64}^{v_{n,1}^{1/2}} \left(\cH_B(C_1\tau,\cF_{n},\|\cdot\|_{L^2})\right)^{1/2}d\tau\\
\leq& M_{n,1}^{-1}v_{n,1}^{1/2}\left(\cH_B(C_1\epsilon_n^2/128,\cF_{n},\|\cdot\|_{L^2})\right)^{1/2}\\
\leq& (D_2F_n^{2-\nu})^{1/2}M_{n,1}^{\nu/2-1}\left(C_3e^{-F_n}n\left(2^{-7}C_1\epsilon_n^2F_n^{-1}\right)^{2-\nu}\right)^{1/2}\\
\leq& C_3^{1/2}C_1^{1-\nu/2}D_2^{1/2}F_n^{1-\nu/2}e^{-F_n/2}(\epsilon_n^2/2)^{\nu/2-1} \left(2^{7\nu/2-7}n^{1/2}\epsilon_n^{2-\nu}F_n^{\nu/2-1}\right)\\
=&\left(2^{6\nu-12}e^{-F_n}C_3C_1^{2-\nu}D_2\right)^{1/2}n^{1/2},
  \end{align*}
  where in the third inequality we used (A5). (B3) is satisfied when $F_n$ is large enough.

  To verify (B1), we use (B2) and (B3). From (B2), since $v_{n,i}^{1/2}\leq F$, we have $M_{n,i}\leq \lambda v_{n,i}/(4F)\leq \frac{1}{8}v_{n,i}^{1/2}$ which implies $v_{n,i}^{1/2}\geq 8M_{n,i}>M_{n,i}/16=\lambda M_{n,i}/8$. Thus the condition in (B3) is satisfied. From (B3), we derive
  \begin{align*}
    &\left(\cH_B\left(v_{n,i}^{1/2},\cG_{n,i},\|\cdot\|_{L^2}\right)\right)^{1/2}\leq \frac{M_{n,i}}{v_{n,i}^{1/2}-M_{n,i}/64}M_{n.i}^{-1} \int_{ M_{n,i}/64}^{v_{n,i}^{1/2}} \left(\cH_B(\tau,\cG_{n,i},\|\cdot\|_{L^2})\right)^{1/2}d\tau\\
    &\leq \frac{M_{n,i}}{v_{n,i}^{1/2}-M_{n,i}/64}\cdot2^{-23/2}n^{1/2} \leq \frac{4}{3}\frac{M_{n,i}}{v_{n,i}^{1/2}}\cdot 2^{-23/2}n^{1/2}=\frac{1}{3\cdot 2^{19/2}}\frac{M_{n,i}}{v_{n,i}^{1/2}}n^{1/2},
  \end{align*}
  where in the third inequality we used $M_{n,i}\leq 16v_{n,i}^{1/2}$.
  Again from (B2), $M_{n,i}\leq \lambda v_{n,i}/(4F)=v_{n,i}/(8F)$. Therefore
  \begin{align*}
\frac{\lambda M_{n,i}^2 n}{8(4v_{n,i}+M_{n,i}F/3)}&=\frac{M_{n,i}^2 n}{16(4v_{n,i}+M_{n,i}F/3)}\geq \frac{M_{n,i}^2n}{(64+2/3)v_{n,i}}\\
    &\geq \frac{M_{n,i}^2n}{9\cdot 2^{19}v_{n,i}}\geq \cH_B\left(v_{n,i}^{1/2},\cG_{n,i},\|\cdot\|_{L^2}\right)
  \end{align*}
  and (B1) is verified.

  Apply Lemma \ref{lem.largedeviation} to each $\cH_{n,i}$, we get
  \begin{align*}
    &\PP\left(\cE_{\phi}(\hf_{\phi,n},f_{\phi}^*)\geq \epsilon_n^2\right)\leq \sum_{i=1}^{i_n} 3\exp\left(-\frac{nM_{n,i}^2}{4\left(4v_{n,i}+M_{n,i}F/3\right)}\right)\\
\leq& \sum_{i=1}^{\infty} 3\exp\left(-C_4nM_{n,i}^2/v_{n,i}\right)\leq \sum_{i=1}^{\infty} 3\exp\left(-C_5\left(2^i\right)^{2-\nu}e^{-F_n}n\left(\epsilon_n^2/F_n\right)^{2-\nu}\right)\\
\leq &C_6\exp\left(-C_5e^{-F_n}n\left(\epsilon_n^2/F_n\right)^{2-\nu}\right).
  \end{align*}
\end{proof}
\subsection{Proof of Lemma \ref{lem.sketch2.fPhiNcnn}}\label{proof.lem.fPhiNcnn}
The proof of Lemma \ref{lem.sketch2.fPhiNcnn} consists of two steps. We first show that there exists a composition of networks $\tf_{\phi,n}$ such that  $\|\tf_{\phi,n}-f_{\phi,n}^*\|_{L^\infty}\leq 4e^{F_n}\varepsilon$. Then we show that $\tf_{\phi,n}$ can be realized by a ConvResNet $\bar{f}_{\phi,n}$.

Lemma \ref{lem.fPhiN} shows the existence of $\tf_{\phi,n}$.
\begin{lemma}\label{lem.fPhiN}
    Assume Assumption \ref{assum.M} and \ref{assum.reach}. Assume $0<p,q\leq \infty$, $0<s<\infty$, $s\geq d/p+1$. There exists a network composition architecture  
    \begin{align}
    	\widetilde{\cF}^{(F_n)}=\{g_{F_n}\circ\thh_n\circ g_n\circ\bar{\eta}\}
    	\label{eq.tfphin}
    \end{align}
    where 
    $\bar{\eta}\in \cC(M,L,J,K,\kappa_1,\kappa_2)$ with
    \begin{align*}
    	M=O(\varepsilon^{-d/s}),\ 
    	L=O(\log(1/\varepsilon)+D+\log D),\
    	J=O(D),\
    	\kappa_1=O(1), \ \log \kappa_2=O\left(\log^2(1/\varepsilon)\right),
    \end{align*}
    $\thh_n\in \cF^{\rm MLP}(L_2,J_2,\kappa_2,F_n)$ with
    \begin{align*}
    	&L_2=O(\log(e^{F_n}/\varepsilon)),\
    	J_2=O(e^{F_n}\varepsilon^{-1}\log(e^{F_n}/\varepsilon)),\
    	\kappa_2=O(e^{F_n}),
    \end{align*}
    and 
    \begin{align*}
    	&g_n(z)=\ReLU\left(-\ReLU\left(-z+\frac{e^{F_n}}{1+e^{F_n}}\right)+\frac{e^{F_n}-1}{1+e^{F_n}}\right) +\frac{1}{1+e^{F_n}},\\
    	&g_{F_n}=\ReLU\left(-\ReLU\left(-z+F_n\right)+2F_n\right) -F_n.
    \end{align*}

For any $\eta\in \Bnorm(\cM)$ with $\|\eta\|_{\Bnorm(\cM)}\leq c_0$ for some constant $c_0$, let $f_{\phi,n}^*$ be defined as in (\ref{eq.sketch2.f-phi-n}). For any $n$ and $\varepsilon\in (0,1)$, there exists a composition of networks $\tf_{\phi,n}\in \widetilde{\cF}^{(F_n)}$ such that
    \begin{align*}
      \|\tf_{\phi,n}-f_{\phi,n}^*\|_{L^\infty}\leq 4e^{F_n}\varepsilon.
    \end{align*}

  \end{lemma}

\begin{proof}[Proof of Lemma \ref{lem.fPhiN}]
	According to Theorem \ref{thm.approximation}, for any $\varepsilon_1\in(0,1)$ and $K\in[2,D]$, there is a ConvResNet architecture $\cC(M,L,J,K,\kappa_1,\kappa_2)$  with
	\begin{align*}
		M=O(\varepsilon_1^{-d/s}),\ 
		L=O(\log(1/\varepsilon_1)+D+\log D),\
		J=O(D),\
		\kappa_1=O(1), \ \log \kappa_2=O\left(\log^2(1/\varepsilon)\right),
	\end{align*}
	such that there exists $\bar{\eta}\in \cC(M,L,J,K,\kappa_1,\kappa_2)$ with $\|\bar{\eta}-\eta\|_{L^\infty}\leq \varepsilon_1$
	
	Since $f_{\phi}^*=\log \frac{\eta}{1-\eta}$, we have $f_{\phi,n}^*=\log \frac{\eta_n}{1-\eta_n}$ with
	\begin{align*}
		\eta_n(\xb)=
		\begin{cases}
			\frac{1}{1+e^{F_n}}, &\mbox{ if } \eta(\xb)<\frac{1}{1+e^{F_n}},\\
			\eta(\xb), &\mbox{ if } \frac{1}{1+e^{F_n}}\leq \eta(\xb) \leq\frac{e^{F_n}}{1+e^{F_n}},\\
			\frac{e^{F_n}}{1+e^{F_n}}, &\mbox{ if } \eta(\xb)>\frac{e^{F_n}}{1+e^{F_n}}.
		\end{cases}
	\end{align*}
The function $\max(\min(z,\frac{e^{F_n}}{1+e^{F_n}}),\frac{1}{1+e^{F_n}})$ can be realized by
	 $g_n(z)=\ReLU\left(-\ReLU\left(-z+\frac{e^{F_n}}{1+e^{F_n}}\right)+\frac{e^{F_n}-1}{1+e^{F_n}}\right) +\frac{1}{1+e^{F_n}}$.
	Then $g_n\circ\bar{\eta}$ is an approximation of $\eta_n$.

	Define 
	$$h_n=\max\left(\min\left(\log\left(\frac{z}{1-z}\right),F_n\right),-F_n\right)$$ 
	for $z\in[0,1]$ which is a Lipschitz function with Lipschitz constant $(1+e^{F_n})^2/e^{F_n}$.
	According to \citet[Theorem 4.1]{chen1908nonparametric}, there exists an MLP $\thh \in \cF^{\rm MLP}(L,J,\kappa)$ such that $\|\thh-h_n\cdot \frac{e^{F_n}}{(1+e^{F_n})^2}\|_{L^\infty}\leq \varepsilon_2\frac{e^{F_n}}{(1+e^{F_n})^2}$ with
	\begin{align*}
		L=O(\log(e^{F_n}/\varepsilon_2)),\
		J=O(e^{F_n}\varepsilon_2^{-1}\log(e^{F_n}/\varepsilon_2)),\
		\kappa=1.
	\end{align*}
	Let $\thh_n=\frac{(1+e^{F_n})^2}{e^{F_n}}\thh$. Then $\thh_n\in \cF(L_2,J_2,\kappa_2)$ such that  $\|\thh_n-h_n\|_{L^\infty}\leq \varepsilon_2$ with
	\begin{align*}
		&L_2=O(\log(e^{F_n}/\varepsilon_2)),\
		J_2=O(e^{F_n}\varepsilon_2^{-1}\log(e^{F_n}/\varepsilon_2)),\
		\kappa_2=O(e^{F_n}).
	\end{align*}
	Let $g_{F_n}=\ReLU\left(-\ReLU\left(-z+F_n\right)+2F_n\right) -F_n.$
	We define
	\begin{align*}
		\tf_{\phi,n}=g_{F_n}\circ\thh_n\circ g_n\circ\bar{\eta}
	\end{align*}
	as an approximation of $f_{\phi,n}^*$. Then the error of $\tf_{\phi,n}$ can be decomposed as
	\begin{align*}
		\|\tf_{\phi,n}-f_{\phi,n}^*\|_{L^\infty}&=\| (g_{F_n}\circ\thh_n)\circ (g_n\circ\bar{\eta})-h_n\circ\eta_n\|_{L^\infty}\\
		&\leq \|(g_{F_n}\circ\thh_n)\circ(g_n\circ\bar{\eta})-h_n\circ(g_n\circ\bar{\eta})\|_{L^\infty}+\|h_n\circ (g_n\circ\bar{\eta})-h_n\circ\eta_n\|_{L^\infty} \\
		&\leq \varepsilon_2+\frac{(1+e^{F_n})^2}{e^{F_n}}\|\eta_n-g_n\circ\bar{\eta}\|_{L^\infty} \leq \frac{(1+e^{F_n})^2}{e^{F_n}}\varepsilon_1+\varepsilon_2.
	\end{align*}
	Choosing $\varepsilon_2=\frac{(1+e^{F_n})^2}{e^{F_n}}\varepsilon_1$ gives rise to $\|\tf_{\phi,n}-f_{\phi,n}^*\|_{L^\infty}\leq 4e^{F_n}\varepsilon_1$. With this choice, we have
	\begin{align*}
		L_2=O(\log(1/\varepsilon_1)),\
		J_2=O(\varepsilon_1^{-1}\log(1/\varepsilon_1)),\
		\kappa_2=O(e^{F_n}).
	\end{align*}
	Setting $\varepsilon_1=\varepsilon$ proves the lemma.
\end{proof}

To show that $\widetilde{\cF}^{(F_n)}$ can be realized by a ConvResNet class and to derive its covering number, we need the following lemma to bound the covering number of ConvResNets.
\begin{lemma}\label{lem.resCover}
  Let $\cC(M,L,J,K,\kappa_1,\kappa_2)$ be the ConvResNet structure defined in Theorem \ref{thm.approximation}. Its covering number is bounded by
  $$\log \cN\left(\delta,\cC(M,L,J,K,\kappa_1,\kappa_2),\|\cdot\|_{L^\infty}\right)=O\left(D^3\varepsilon^{-d/s}\log (1/\varepsilon)(\log (1/\varepsilon)+\log D\log(1/\delta)\right).$$
\end{lemma}
Lemma \ref{lem.resCover} is proved based on the following lemma:
\begin{lemma}[Lemma 4 of \citet{oono2019approximation}]\label{lem.resCover.1}
	Let $\cC(M,L,J,K,\kappa_1,\kappa_2)$ be a class of ConvResNet architecture from $\RR^D$ to $\RR$. Let $\kappa=\kappa_1\vee \kappa_2$.  For $\delta>0$, we have
	$$\cN(\delta,\cC(M,L,J,K,\kappa_1,\kappa_2),\|\cdot\|_{L^\infty})\leq (2\kappa \Lambda_1/\delta)^{\Lambda_2},$$
	where
	\begin{align*}
		&\Lambda_1=(8M+12)D^2(1\vee \kappa_2)(1\vee \kappa_1)\trho\trho^+,\ \Lambda_2=ML(16D^2K+4D)+4D^2+1
	\end{align*}
	with $\trho= (1+\rho)^M,\trho^+=1+ML\rho^+, \rho=(4DK\kappa_1)^L$ and $\rho^+=(1\vee 4DK\kappa_1)^L$.
\end{lemma}
\begin{proof}[Proof of Lemma \ref{lem.resCover}]
According to Lemma \ref{lem.resCover.1},
\begin{align*}
    \log \cN\left(\delta,\cC(M,L,J,K,\kappa_1,\kappa_2),\|\cdot\|_{L^\infty}\right) \leq \Lambda_2\log(2\kappa \Lambda_1/\delta).
  \end{align*}
   In the ConvResNet architecture defined in Theorem \ref{thm.approximation}, $\kappa_1=(8DK)^{-1}M^{-1/L}$, $\rho=(1/2)^LM^{-1}<M^{-1}$. We have $\trho= (1+\rho)^M\leq (1+M^{-1})^M\leq e$. Moreover, we have $\rho^+=1, \trho^+=1+ML$. Since $\log \kappa_2=O(\log^2 (1/\varepsilon))$, substituting $M=O\left(\varepsilon^{-d/s}\right)$ and $ L=O(\log (1/\varepsilon)+D+\log D)$ gives rise to $\log \Lambda_1=O(\log^2 (1/\varepsilon)+\log D)$ and $\Lambda_2=O\left(D^3\varepsilon^{-d/s}\log (1/\varepsilon)\right)$. Therefore,
  $$\log \cN\left(\delta,\cC(M,L,J,K,\kappa_1,\kappa_2),\|\cdot\|_{L^\infty}\right)=O\left(D^3\varepsilon^{-d/s}\log (1/\varepsilon)(\log^2 (1/\varepsilon)+\log D+ \log(1/\delta)\right).$$
   The constants hidden in $O(\cdot)$ depend on $d,s,\frac{2d}{sp-d},p,q,c_0,\tau$ and the surface area of $\cM$.
\end{proof}

The following lemma shows that $\widetilde{\cF}^{(F_n)}$ can be realized by a ConvResNet class $\cC^{(F_n)}$ and estimates the covering number of the class of $\cC^{(F_n)}$.
\begin{lemma}\label{lem.fPhiNcnn.1}
Let $\cC^{(F_n)}$ be defined as in Lemma \ref{lem.sketch2.fPhiNcnn}.
The network composition class $\widetilde{\cF}^{(F_n)}$ defined in Lemma \ref{lem.fPhiN} can be realized by a ConvResNet class $ \cC^{(F_n)}$.
Moreover, the covering number of $\cC^{(F_n)}$ is bounded by
\begin{align*}
  \log \cN(\delta,\cC^{(F_n)},\|\cdot \|_{L^\infty})=O\left(\varepsilon^{-\left(\frac{d}{s} \vee 1\right))}\left(\log^3 (1/\varepsilon)+F_n+\log(1/\delta)\right)\right).
\end{align*}
\end{lemma}
\begin{proof}[Proof of Lemma \ref{lem.fPhiNcnn.1}]
In this proof, we show that each part of $\tf_{\phi,n}$ in Lemma \ref{lem.fPhiN} can be realized by ConvResNet architectures. Specifically, we show that $\bar{\eta},g_n,\thh_n$ can be realized by residual blocks $\bar{\eta},\bar{g}_n, \bar{h}$, and $g_{F_n}$ can be realized by a ConvResNet $\bar{g}_{F_n}$.
In the following, we show the existence of each ingredient.

\paragraph{Realize $\bar{\eta}$ by residual blocks.}
In Lemma \ref{lem.fPhiN}, $\bar{\eta}\in \cC(M^{(\eta)}, L^{(\eta)}, J^{(\eta)}, K^{(\eta)}, \kappa^{(\eta)}_1, \kappa^{(\eta)}_2)$ with
  \begin{align*}
  &M^{(\eta)}=O\left(\varepsilon^{-d/s}\right),\ L^{(\eta)}=O(\log (1/\varepsilon)+D+\log D),\ J^{(\eta)}=O(D),\ \kappa^{(\eta)}=O(1),\ \log \kappa^{(\eta)}_2=O(\log^2(1/\varepsilon)).
  \end{align*}
  By Lemma \ref{lem.resCover}, the covering number of this architecture is bounded by
  \begin{align*}
  &\log \cN\left(\delta,\cC(M^{(\eta)}, L^{(\eta)}, J^{(\eta)}, K^{(\eta)}, \kappa^{(\eta)}_1, \kappa^{(\eta)}_2),\|\cdot\|_{L^\infty}\right)\\
  =&O\left(\varepsilon^{-\left(\frac{d}{s} \vee 1\right))}\log (1/\varepsilon)\left(\log^2 (1/\varepsilon)+\log D+\log(1/\delta)\right)\right).
  \end{align*}
  Excluding the final fully connected layer, denote all of the residual blocks of $\bar{\eta}$ by $\bar{\eta}^{(\Conv)}$, then $\bar{\eta}\in \cC^{(\eta)}$ with $\cC^{(\eta)}=\cC^{\Conv}(M^{(\eta)}, L^{(\eta)}, J^{(\eta)}, K^{(\eta)}, \kappa^{(\eta)}_1)$ and
  \begin{align*}
  \log \cN\left(\delta,\cC^{(\eta)},\|\cdot\|_{L^\infty}\right)&\leq \log \cN\left(\delta,\cC(M^{(\eta)}, L^{(\eta)}, J^{(\eta)}, K^{(\eta)}, \kappa^{(\eta)}_1, \kappa^{(\eta)}_2),\|\cdot\|_{L^\infty}\right)\\
  &=O\left(D^3\varepsilon^{-d/s}\log (1/\varepsilon)(\log^2 (1/\varepsilon)+\log D+\log(1/\delta))\right).
  \end{align*}

  We denote the $i$-th row (the $i$-th element of all channels) of the output of the residual-blocks $\bar{\eta}^{(\Conv)}$   by $(\bar{\eta}^{(\Conv)})_{i,:}$.
  In the proof of Theorem \ref{thm.approximation}, the input in $\RR^D$ is padded into $\RR^{D\times 3}$ by 0's. The output of $\bar{\eta}$ has the form $(\bar{\eta}^{(\Conv)})_{1,:}=\frac{\kappa_1^{(\eta)}}{\kappa_2^{(\eta)}}\begin{bmatrix}
  	\star & \bar{\eta}_+& \bar{\eta}_-
  \end{bmatrix}$. Here $\star$ denotes some number that does not affect the result. In this proof, instead of padding the input into size $D\times 3$, we pad it into size $D\times 8$. The weights in the first $M$ blocks of $h$ is the same as that of $\bar{\eta}$ except we need to pad the filters and biases by 0 to be compatible with the additional channels. Then the output of $\bar{\eta}^{(\Conv)}$ is $\frac{\kappa_1^{(\eta)}}{\kappa_2^{(\eta)}}\begin{bmatrix}
  \star & \bar{\eta}_+& \bar{\eta}_- &0 &0 &0 &0 &0
\end{bmatrix}$.

\paragraph{Realize $g_n$ by residual blocks.}
  To realize $g_n$, we add another block with 4 layers with filters and biases $\left\{\cW_{g_n}^{(l)},B_{g_n}^{(l)}\right\}_{l=1}^5$ where $\cW_{g_n}^{(l)}\in \RR^{8\times1\times8},B_{g_n}^{(l)}\in\RR^{8\times 8} $. We set the parameters in the first layer as
  \begin{align*}
  	\left(\cW_{g_n}^{(1)}\right)_{2,1,:}&= \begin{bmatrix}
  		 0 &\frac{\kappa_2^{(\eta)}}{\kappa_1^{(\eta)}} & 0  & 0 &0&0&0&0
  	\end{bmatrix},\\
  \left(\cW_{g_n}^{(1)}\right)_{3,1,:}&= \begin{bmatrix}
  	 0 & 0 & \frac{\kappa_2^{(\eta)}}{\kappa_1^{(\eta)}}  &0 &0&0&0&0
  \end{bmatrix},\\
\left(\cW_{g_n}^{(1)}\right)_{i,1,:}&=\mathbf{0} \quad \mbox{ for } i=1,4,5,...,8,
  \end{align*}
and $B_{g_n}^{(1)}=\mathbf{0}$.
  This layer scales the output of $\bar{\eta}$ back to $\begin{bmatrix}
  	\star & \bar{\eta}_+& \bar{\eta}_- & 0& 0 &0 &0&0
  \end{bmatrix}$.  Then we use the other 4 layers to realize $g_n$. The second layer is set as
  \begin{align*}
  		\left(\cW_{g_n}^{(2)}\right)_{4,1,:}&= \begin{bmatrix}
  		0 &-1 & 1  & 0 &0&0&0&0
  	\end{bmatrix},\\
  	\left(\cW_{g_n}^{(2)}\right)_{i,1,:}&=\mathbf{0} \quad \mbox{ for } i=1,2,3,5,6,7,8,\\
  	\left(B_{g_n}^{(2)}\right)_{1,:}&=\begin{bmatrix}0& 0& 0& \frac{e^{F_n}}{1+e^{F_n}}& 0 &0 & 0& 0&\end{bmatrix},\\
  	\left(B_{g_n}^{(2)}\right)_{i,:}&=\mathbf{0} \quad \mbox{ for } i=2,...,8.
  \end{align*}
The third layer is set as
 \begin{align*}
	\left(\cW_{g_n}^{(3)}\right)_{4,1,:}&= \begin{bmatrix}
		0 &0 & 0  & -1 &0&0&0&0
	\end{bmatrix},\\
	\left(\cW_{g_n}^{(3)}\right)_{i,1,:}&=\mathbf{0} \quad \mbox{ for } i=1,2,3,5,6,7,8,\\
	\left(B_{g_n}^{(3)}\right)_{1,:}&=\begin{bmatrix}0& 0& 0& \frac{e^{F_n}-1}{1+e^{F_n}}& 0 &0 & 0& 0&\end{bmatrix},\\
	\left(B_{g_n}^{(3)}\right)_{i,:}&=\mathbf{0} \quad \mbox{ for } i=2,...,8.
\end{align*}
The forth layer is set as
\begin{align*}
	\left(\cW_{g_n}^{(4)}\right)_{4,1,:}&= \begin{bmatrix}
		0 &0 & 0  & 1 &0&0&0&0
	\end{bmatrix},\\
	\left(\cW_{g_n}^{(4)}\right)_{i,1,:}&=\mathbf{0} \quad \mbox{ for } i=1,2,3,5,6,7,8,\\
	\left(B_{g_n}^{(4)}\right)_{1,:}&=\begin{bmatrix}0& 0& 0& \frac{1}{1+e^{F_n}}& 0 &0 & 0& 0&\end{bmatrix},\\
	\left(B_{g_n}^{(4)}\right)_{i,:}&=\mathbf{0} \quad \mbox{ for } i=2,...,8.
\end{align*}
  The output of $g_n\circ \bar{\eta}$ is stored as the first element in the forth channel of the output of $\bar{g}_n\circ \bar{\eta}^{\rm (Conv)}$:
  $$(\bar{g}_n\circ \bar{\eta})_{1,:}=\begin{bmatrix} \star & \star &\star & g_n\circ \bar{\eta} &0&0 & 0 &0\end{bmatrix}.$$
  
  We have $\bar{g}_n\in \cC^{(g_n)} $ where $\cC^{(g_n)}=\cC^{(g_n)}=\cC^{\Conv}(M^{(g_n)}, L^{(g_n)},J^{(g_n)},K^{(g_n)},\kappa^{(g_n)})$ with
  \begin{align*}
  	M^{(g_n)}=1,\ L^{(g_n)}=4, \ J^{(g_n)}=8,\ K^{(g_n)}=1,\ \kappa^{(g_n)}=O\left(\frac{\kappa^{(\eta)}_2}{\kappa^{(\eta)}_1} \vee \frac{e^{F_n}}{1+e^{F_n}}\right).
  \end{align*}
  According to Lemma \ref{lem.resCover.1}, the covering number of $\cC^{(g_n)}$ is bounded as
  $$
  \log \cN(\delta,\cC^{(g_n)},\|\cdot\|_{L^\infty})=O\left(\log \left(\frac{\kappa_2^{(\eta)}}{\kappa_1^{(\eta)}} \vee \frac{e^{F_n}}{1+e^{F_n}}\right) + \log (1/\delta)\right).
  $$
   Substituting the expressions of $\kappa^{(\eta)}_1,\kappa^{(\eta)}_2$ into the expression above gives rise to $\log \kappa^{(g_n)}=O(\left(\log^2 (1/\varepsilon)\right)$ and $\log \cN(\delta,\cC^{(g_n)},\|\cdot\|_{L^\infty})=O(\log^2 (1/\varepsilon)+\log (1/\delta))$.

\paragraph{Realize $\thh_n$ by residual blocks.}
  To realize $\thh_n$, from the construction of $\thh_n$ and using \citet[Corollary 4]{oono2019approximation}, we can realize $\thh_n$ by $\bar{h}_n\in \cC^{(h_n)}$ with
  \begin{align*}
  	\cC^{(h_n)}=\cC(M^{(h_n)},L^{(h_n)},J^{(h_n)},K^{(h_n)},\kappa_1^{(h_n)}\kappa_2^{(h_n)},F_n),
  \end{align*}
and 
  \begin{align*}
  &M^{(h_n)}=O\left(\varepsilon_2^{-1}\right),\ L^{(h_n)}=O\left(\log (1/\varepsilon_2)\right),\  J^{(h_n)}=O(1),\ K^{(h_n)}=1, \\ &\kappa_1^{(h_n)}=O(1),\ \log \kappa_2^{(h_n)}=O(\log (L_h/\varepsilon_2)),
  \end{align*}
  where $\varepsilon_2=\frac{(1+e^{F_n})^2}{e^{F_n}}\varepsilon$, and $L_h$ is the Lipschitz constant of $h_n$.
  According to Lemma \ref{lem.resCover.1}, the covering number of this class is bounded by
  $$\log \cN\left(\delta,\cC^{(h_n)},\|\cdot\|_{L^\infty}\right)=O\left(D^2\varepsilon_2^{-1}\log (1/\varepsilon_2)(\log (1/\varepsilon_2)+\log (L_h/\varepsilon_2)+\log(1/\delta)\right).$$

  Note that such $\bar{h}$ is from $\RR$ to $\RR$. Since the information we need from the output of $\bar{g}_n\circ \bar{\eta}$ is only the first element in the forth channel, we can follow the proof of \citet[Theorem 6]{oono2019approximation} to construct $\bar{h}$ by padding the elements in the filters and biases by 0 so that all operations work on the forth channel and store results on the fifth and sixth channel. Substituting $\varepsilon_2=\frac{(1+e^{F_n})^2}{e^{F_n}}\varepsilon$ and $L_{h_n}=(1+e^{F_n})^2/e^{F_n}$ yields
  \begin{align*}
  &M^{(h_n)}=O\left(e^{-F_n}\varepsilon^{-1}\right), L^{(h_n)}=O\left(\log (1/\varepsilon)\right), J^{(h_n)}=O(1), \\ &\kappa_1^{(h_n)}=O(1),\log \kappa_2^{(h_n)}=O\left(\log \left(e^{F_n}/\varepsilon\right)\right)
  \end{align*}
  and
  $$
  \log \cN\left(\delta,\cC^{(h_n)},\|\cdot\|_{L^\infty}\right)= O\left(D^2e^{-F_n}\varepsilon^{-1}\log (1/\varepsilon)(\log (1/\varepsilon)+F_n+\log(1/\delta)\right).
  $$

  Similar to $\bar{\eta}^{(\Conv)}$, denote all residual blocks of $\bar{h}$ by $\bar{h}^{(\Conv)}$. We have $$(\bar{h}^{(\Conv)})_{1,:}=\frac{\kappa_1^{(h_n)}}{\kappa_2^{(h_n)}}\begin{bmatrix} \star & \star & \star & \star & (\thh_n)_+ & (\thh_n)_- &0 &0\end{bmatrix}.
  $$ 
  
  \paragraph{Realize $g_{F_n}$ by a ConvResNet.}
  We then add another residual block of 3 layers followed by a fully connected layer to realize $g_{F_n}$. Denote the parameters in this block and the fully connected layer by $\{\cW^{(l)}_{g_{F_n}},B^{(l)}_{g_{F_n}}\}_{l=1}^3$ and $\{W,b\}$, respectively. Here $\cW^{(l)}_{g_{F_n}}\in \RR^{8\times1\times8},B^{(l)}_{g_{F_n}}\in \RR^{8\times8}, W\in \RR^{8\times8}$ and $b\in\RR$.
  
  The first layer is set as
  \begin{align*}
  	\left(\cW_{g_{F_n}}^{(1)}\right)_{5,1,:}&= \begin{bmatrix}
  		0 &0 & 0  & 0 &\frac{\kappa_2^{(h_n)}}{\kappa_1^{(h_n)}}&0&0&0
  	\end{bmatrix},\\
  	\left(\cW_{g_{F_n}}^{(1)}\right)_{6,1,:}&= \begin{bmatrix}
  	0 &0 & 0  & 0 &0&\frac{\kappa_2^{(h_n)}}{\kappa_1^{(h_n)}}&0&0
  \end{bmatrix},\\
  	\left(\cW_{g_{F_n}}^{(1)}\right)_{i,1,:}&=\mathbf{0} \quad \mbox{ for } i=1,2,3,4,7,8,\\
  	\left(B_{g_{F_n}}^{(1)}\right)_{i,:}&=\mathbf{0} \quad \mbox{ for }i=1,2,...,8.
  \end{align*}
This layer scales the output of $\bar{h}^{(\Conv)}$ back to $\begin{bmatrix} \star & \star & \star & \star & (\thh_n)_+ & (\thh_n)- &0 &0 \end{bmatrix}$. The rest layers are used to realize $g_{F_n}$.

The second layer is set as
\begin{align*}
	\left(\cW_{g_{F_n}}^{(2)}\right)_{7,1,:}&= \begin{bmatrix}
		0 &0 & 0  & 0 &-1&0&0&0
	\end{bmatrix},\\
	\left(\cW_{g_{F_n}}^{(2)}\right)_{8,1,:}&= \begin{bmatrix}
		0 &0 & 0  & 0 &0&-1&0&0
	\end{bmatrix},\\
	\left(\cW_{g_{F_n}}^{(2)}\right)_{i,1,:}&=\mathbf{0} \quad \mbox{ for } i=1,...,6\\
	\left(B_{g_{F_n}}^{(2)}\right)_{1,:}&=\begin{bmatrix}0& 0& 0& 0 & 0 &0& F_n& F_n\end{bmatrix},\\
	\left(B_{g_{F_n}}^{(2)}\right)_{i,:}&=\mathbf{0} \quad \mbox{ for }i=2,...,8.
\end{align*}

The third layer is set as
\begin{align*}
	\left(\cW_{g_{F_n}}^{(3)}\right)_{7,1,:}&= \begin{bmatrix}
		0 &0 & 0  & 0 &0&0&-1&0
	\end{bmatrix},\\
	\left(\cW_{g_{F_n}}^{(3)}\right)_{8,1,:}&= \begin{bmatrix}
		0 &0 & 0  & 0 &0&0&0&-1
	\end{bmatrix},\\
	\left(\cW_{g_{F_n}}^{(3)}\right)_{i,1,:}&=\mathbf{0} \quad \mbox{ for } i=1,...,6\\
	\left(B_{g_{F_n}}^{(3)}\right)_{1,:}&=\begin{bmatrix}0& 0& 0& 0 & 0 &0& F_n& F_n\end{bmatrix},\\
	\left(B_{g_{F_n}}^{(3)}\right)_{i,:}&=\mathbf{0} \quad \mbox{ for }i=2,...,8.
\end{align*}
The first row of the output of the third layer is 
$$
\begin{bmatrix} \star & \star & \star & \star &\star &\star&\min((\thh_n)_+,F_n) &  \min((\thh_n)_-,F_n) \end{bmatrix}.
$$
Then the fully connected layer is set as
\begin{align*}
	W_{1,:}&=\begin{bmatrix}
		0 &0 & 0  & 0 &0&0&1&-1
	\end{bmatrix},\\
	W_{i,:}&=\mathbf{0} \quad \mbox{ for } i=2,...,8
\end{align*}
and $b=0$.

Thus $\bar{g}_{F_n}\in \cC^{(g_{F_n})}$ with $\cC^{(g_{F_n})}=\cC(M^{(g_{F_n})}, L^{(g_{F_n})},J^{(g_{F_n})},K^{(g_{F_n})},\kappa_1^{(g_{F_n})},\kappa_2^{(g_{F_n})})$  and
\begin{align*}
	M^{(g_{F_n})}=1,\ L^{(g_{F_n})}=4, \ J^{(g_{F_n})}=8,\ K^{(g_{F_n})}=1,\ \kappa_1^{(g_{F_n})}=O\left(\frac{\kappa^{(h_n)}_2}{\kappa^{(h_n)}_1} \vee F_n\right),\ \kappa_2^{(g_{F_n})}=1.
\end{align*}

According to Lemma \ref{lem.resCover.1}, substituting the expressions of $\kappa^{(h_n)}_1,\kappa^{(h_n)}_2$ gives rise to $\log \cN(\delta,\cC^{(g_{F_n})},\|\cdot\|_{L^\infty})=O\left(D^2\left(\log (1/\varepsilon) + \log F_n+\log(1/\delta)\right)\right)$.

  The resulting network $\bar{f}_{\phi,n}\equiv \bar{g}_{F_n}\circ\bar{h}^{\rm (Conv)}\circ \bar{g}_n\circ \bar{\eta}^{\rm (Conv)}$ is a ConvResNet and
  $$
  \bar{f}_{\phi,n}(\xb)=\tf_{\phi,n}(\xb)
  $$
  for any $\xb\in\cM$.
  Denote the class of the architecture of $\bar{f}_{\phi,n}$ by $\cC^{(F_n)}$. Its covering number is bounded by
  \begin{align*}
    &\log \cN(\delta,\cC^{(F_n)},\|\cdot\|_{L^\infty})\\
    \leq &\log \cN\left(\delta,\cC^{(\eta)},\|\cdot\|_{L^\infty}\right)+ \log \cN\left(\delta,\cC^{(g_n)},\|\cdot\|_{L^\infty}\right)\\
    &+ \log \cN\left(\delta,\cC^{(h_n)},\|\cdot\|_{L^\infty}\right) +\log \cN(\delta,\cC^{(g_{F_n})},\|\cdot\|_{L^\infty})\\
    =&O\left(D^3\varepsilon^{-\left(\frac{d}{s} \vee 1\right))}\log (1/\varepsilon)\left(\log^2 (1/\varepsilon)+\log D+ F_n+\log(1/\delta)\right)\right).
  \end{align*}
The constants hidden in $O(\cdot)$ depend on $d,s,\frac{2d}{sp-d},p,q,c_0,\tau$ and the surface area of $\cM$.
\end{proof}

\begin{proof}[Proof of Lemma \ref{lem.sketch2.fPhiNcnn}]
  Lemma \ref{lem.sketch2.fPhiNcnn} is a direct result of Lemma \ref{lem.fPhiN} and Lemma \ref{lem.fPhiNcnn.1}.
\end{proof}

\subsection{Proof of Lemma \ref{lem.sketch2.II}}\label{proof.sketch2.II}
\begin{proof}[Proof of Lemma \ref{lem.sketch2.II}]
We divide $A_n^{\complement}$ into two regions: $\{\xb\in \cM: f^*_{\phi}>F_n\}$ and $\{\xb\in \cM: f^*_{\phi}<-F_n\}$. A bound of ${\rm T_2}$ is derived by bounding the integral on both regions.

Let us first consider the region $\{\xb\in \cM: f^*_{\phi}>F_n\}$. Since $\eta=e^{f_{\phi}^*}/(1+e^{f_{\phi}^*})$, we have
\begin{align}
  \eta\phi(f_{\phi}^*)+(1-\eta)\phi(-f_{\phi}^*)&=\frac{e^{f_{\phi}^*}}{1+e^{f_{\phi}^*}} \phi(f_{\phi}^*) + \frac{1}{1+e^{f_{\phi}^*}}\phi(-f_{\phi}^*)\nonumber\\
  &\leq \phi(F_n)+\sup_{z\geq F_n} \frac{\log(1+e^z)}{1+e^z}\nonumber\\
  &\leq \log(1+e^{-F_n})+ \frac{\log(1+e^{F_n})}{1+e^{F_n}}\nonumber\\
  &\leq 2F_ne^{-F_n}.
  \label{eq.proof.I.1}
\end{align}
On this region, $F_n-1\leq \bar{f}_{\phi,n}\leq F_n$. Thus
\begin{align}
  \eta\phi(\bar{f}_{\phi,n})+(1-\eta)\phi(-\bar{f}_{\phi,n})&=\frac{e^{f_{\phi}^*}}{1+e^{f_{\phi}^*}} \phi(\bar{f}_{\phi,n}) + \frac{1}{1+e^{f_{\phi}^*}}\phi(-\bar{f}_{\phi,n}) \nonumber\\
  & \leq \phi(F_n-1)+\frac{\log(1+e^{\bar{f}_{\phi,n}})}{1+e^{f_{\phi}^*}}\nonumber\\
  &\leq \log(1+e^{-(F_n-1)})+ \frac{\log(1+e^{F_n})}{1+e^{F_n}}\nonumber\\
  &\leq 2F_ne^{-F_n}.
   \label{eq.proof.I.2}
\end{align}
Combining (\ref{eq.proof.I.1}) and (\ref{eq.proof.I.2}) gives
\begin{align}
  \left|\left[\eta\phi(f_{\phi}^*)+(1-\eta)\phi(-f_{\phi}^*)\right] - \left[\eta\phi(\bar{f}_{\phi,n})+(1-\eta)\phi(-\bar{f}_{\phi,n}) \right]\right| \leq 4F_ne^{-F_n}.
  \label{eq.proof.I}
\end{align}

Now consider the region $\{\xb\in \cM: f^*_{\phi}<-F_n\}$, we have
\begin{align}
  \eta\phi(f_{\phi}^*)+(1-\eta)\phi(-f_{\phi}^*)&=\frac{e^{f_{\phi}^*}}{1+e^{f_{\phi}^*}} \phi(f_{\phi}^*) + \frac{1}{1+e^{f_{\phi}^*}}\phi(-f_{\phi}^*) \nonumber\\
  &= \frac{1}{1+e^{-f_{\phi}^*}} \phi(f_{\phi}^*) + \frac{e^{-f_{\phi}^*}}{1+e^{-f_{\phi}^*}}\phi(-f_{\phi}^*)\nonumber\\
  &\leq \phi(F_n)+\sup_{z\leq -F_n} \frac{\log(1+e^{-z})}{1+e^{-z}} \nonumber\\
   &\leq \log(1+e^{-F_n})+ \frac{\log(1+e^{F_n})}{1+e^{F_n}}\nonumber\\
   &\leq 2F_ne^{-F_n}.
   \label{eq.proof.II.1}
\end{align}
On this region, $-F_n\leq \bar{f}_{\phi,n}\leq -F_n+1$. Thus
\begin{align}
  \eta\phi(\bar{f}_{\phi,n})+(1-\eta)\phi(-\bar{f}_{\phi,n})&= \frac{1}{1+e^{-f_{\phi}^*}} \phi(\bar{f}_{\phi,n}) + \frac{e^{-f_{\phi}^*}}{1+e^{-f_{\phi}^*}}\phi(-\bar{f}_{\phi,n}) \nonumber \\
  &\leq \phi(F_n-1)+ \frac{\log(1+e^{-\bar{f}_{\phi,n}})}{1+e^{-f_{\phi}^*}} \nonumber\\
  &\leq \log(1+e^{-(F_n-1)})+ \frac{\log(1+e^{F_n})}{1+e^{F_n}}\nonumber\\
  &\leq 2F_ne^{-F_n}.
   \label{eq.proof.II.2}
\end{align}
Combining (\ref{eq.proof.II.1}) and (\ref{eq.proof.II.2}) gives
\begin{align}
  \left|\left[\eta\phi(f_{\phi}^*)+(1-\eta)\phi(-f_{\phi}^*)\right] - \left[\eta\phi(\bar{f}_{\phi,n})+(1-\eta)\phi(-\bar{f}_{\phi,n}) \right]\right| \leq 4F_ne^{-F_n}.
  \label{eq.proof.II}
\end{align}

Putting (\ref{eq.proof.I}) and (\ref{eq.proof.II}) together, we have
\begin{align*}
  {\rm T_2}&\leq\int_{A_n^{\complement}}\left| \left[\eta\phi(f_{\phi}^*)+(1-\eta)\phi(-f_{\phi}^*)\right]- \left[\eta\phi(\bar{f}_{\phi,n})+(1-\eta)\phi(-\bar{f}_{\phi,n}) \right]\right|\mu(d\xb)\leq 8F_ne^{-F_n}.
\end{align*}
\end{proof}


\end{document}